\documentclass[lettersize,journal]{IEEEtran}
\usepackage[colorlinks,linkcolor=black,anchorcolor=black,citecolor=black,]{hyperref}
\usepackage{times}
\usepackage{soul}
\usepackage{url}%
\usepackage[utf8]{inputenc}
\usepackage[small]{caption}
\usepackage{graphicx}%
\usepackage{amsmath} %
\usepackage{amssymb}
\usepackage{amsthm}
\usepackage{booktabs}
\usepackage{subfig}
\usepackage{amsthm}
\usepackage{mathrsfs}
\usepackage{verbatim}%
\usepackage{multirow}
\usepackage{colortbl}
\usepackage{epstopdf}
\usepackage[ruled]{algorithm2e} %
\newtheorem{theorem}{Theorem}
\SetKw{KwInit}{Initialize:}
\SetKw{KwRetern}{Retern}
\usepackage{float}
\usepackage{stackengine}
\usepackage[table,xcdraw]{xcolor}
\usepackage{graphicx,subfig,color,multirow,amsmath}
\usepackage{cite}
\usepackage{amsfonts}
\usepackage{algorithmic}
\usepackage{array}
\begin{document}

\title{Local Sample-weighted Multiple Kernel Clustering with Consensus Discriminative Graph}

\author{Liang~Li$^{\dagger}$,~Siwei~Wang$^{\dagger}$,~Xinwang~Liu$^{\ast}$,~\IEEEmembership{Senior~Member,~IEEE,}~En~Zhu,~Li Shen,~Kenli~Li,~\IEEEmembership{Senior~Member,~IEEE,}\\
~and~Keqin~Li,~\IEEEmembership{Fellow,~IEEE}
        % <-this % stops a space
\thanks{L. Li, S. Wang, X. Liu, E. Zhu, and L. Shen are with School of Computer, National University of Defense Technology, Changsha 410073, China (E-mail: \{liangli,\,wangsiwei13,\,xinwangliu,\,enzhu,\,lishen\}@nudt.edu.cn).}
\thanks{K. Li is with College of Computer Science and Electronic Engineering, Hunan University, and also with Supercomputing and Cloud Computing Institute, Changsha 410073, China (E-mail: lkl@hnu.edu.cn).}
\thanks{K. Li is with the Department of Computer Science, State University of New York, New Paltz, New York 12561, USA (E-mail: lik@newpaltz.edu).}
\thanks{$^{\dagger}$ Equal contribution.}
\thanks{$^{\ast}$ Corresponding author.}}


\markboth{Submitted to IEEE Transactions on Neural Networks and Learning Systems, December~2021}
{Shell \MakeLowercase{\textit{et al.}}: A Sample Article Using IEEEtran.cls for IEEE Journals}

\IEEEpubid{0000--0000/00\$00.00~\copyright~2021 IEEE}

\maketitle

\begin{abstract}
Multiple kernel clustering (MKC) is committed to achieving optimal information fusion from a set of base kernels. Constructing precise and local kernel matrices is proved to be of vital significance in applications since the unreliable distant-distance similarity estimation would degrade clustering performance. Although existing localized MKC algorithms exhibit improved performance compared to globally-designed competitors, most of them widely adopt KNN mechanism to localize kernel matrix by accounting for $\tau$-nearest neighbors. However, such a coarse manner follows an unreasonable strategy that the ranking importance of different neighbors is equal, which is impractical in applications. To alleviate such problems, this paper proposes a novel local sample-weighted multiple kernel clustering (LSWMKC) model. We first construct a consensus discriminative affinity graph in kernel space, revealing the latent local structures. Further, an optimal neighborhood kernel for the learned affinity graph is output with naturally sparse property and clear block diagonal structure. Moreover, LSWMKC implicitly optimizes adaptive weights on different neighbors with corresponding samples. Experimental results demonstrate that our LSWMKC possesses better local manifold representation and outperforms existing kernel or graph-based clustering algorithms. The source code of LSWMKC can be publicly accessed from \url{https://github.com/liliangnudt/LSWMKC}.
\end{abstract}


%==========================================================
\section{Introduction}
\IEEEPARstart{C}{lustering} is one of the representative unsupervised learning techniques widely employed in data mining and machine learning \cite{jain1999data,xu2005survey,liao2017automatic,liao2018multiple,yang2018multi, xiao2019novel}. As a popular algorithm, $k$-means has been well investigated \cite{hartigan1979algorithm, wagstaff2001constrained, peng2022xai}. Although achieving extensive applications, $k$-means assumes that data can be linearly separated into different clusters \cite{scholkopf1998nonlinear}. By employing kernel tricks, the non-linearly separable data are embedded into a higher dimensional feature space and  become linearly separable. As a consequence, kernel $k$-means (KKM) is naturally developed for handling non-linearity issues \cite{scholkopf1998nonlinear,dhillon2004kernel}. Moreover, to encode the emerging data generated from heterogeneous sources or views, multiple kernel clustering (MKC) provides a flexible and expansive framework for combining a set of kernel matrices since different kernels naturally correspond to different views \cite{zhao2009multiple,yu2012optimized,gonen2014localized,liu2016multiple,liu2017optimal,wang2019multi,wang2021late}. Multiple kernel $k$-means (MKKM) \cite{huang2012multiple} and various variants are further developed and widely employed in many applications  \cite{lu2014multiple,du2015robust,liu2016multiple,liu2017optimal,liu2019multiple,liu2021incomplete}. 
\begin{figure}[t]
\vspace{-10pt}
\begin{center}{
		\centering
		\hspace{-5mm}\subfloat[Average kernel]{{\includegraphics[width=0.19\textwidth]{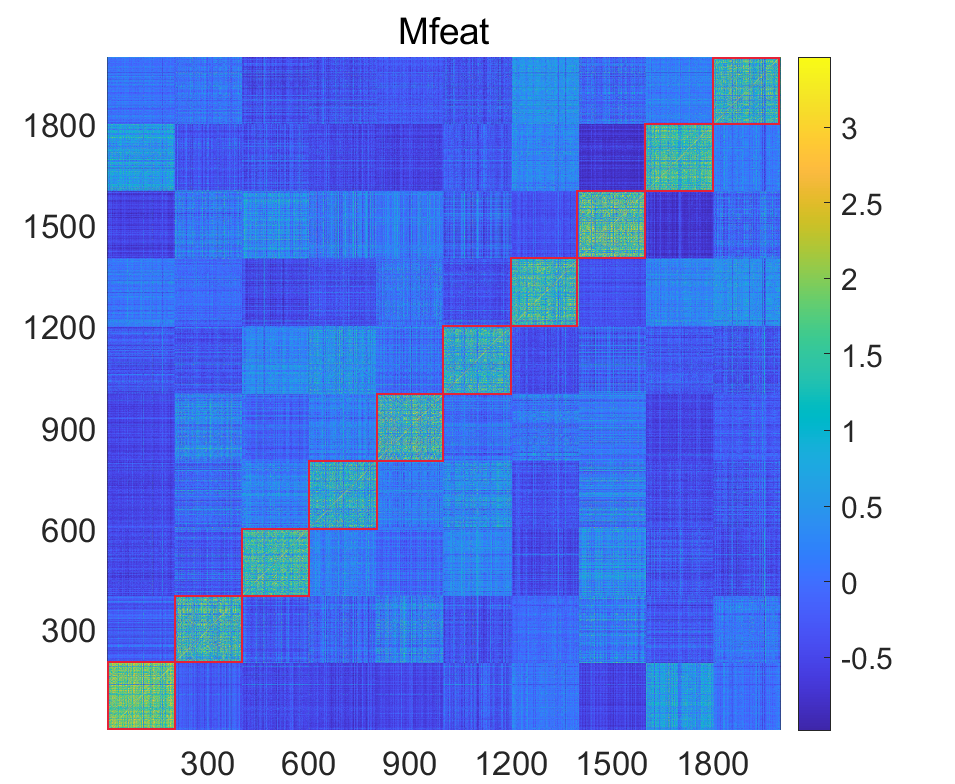}}\hspace{-5.3mm} \label{Avg-mfeat}}
		\subfloat[KNN mechanism]{{\includegraphics[width=0.19\textwidth]{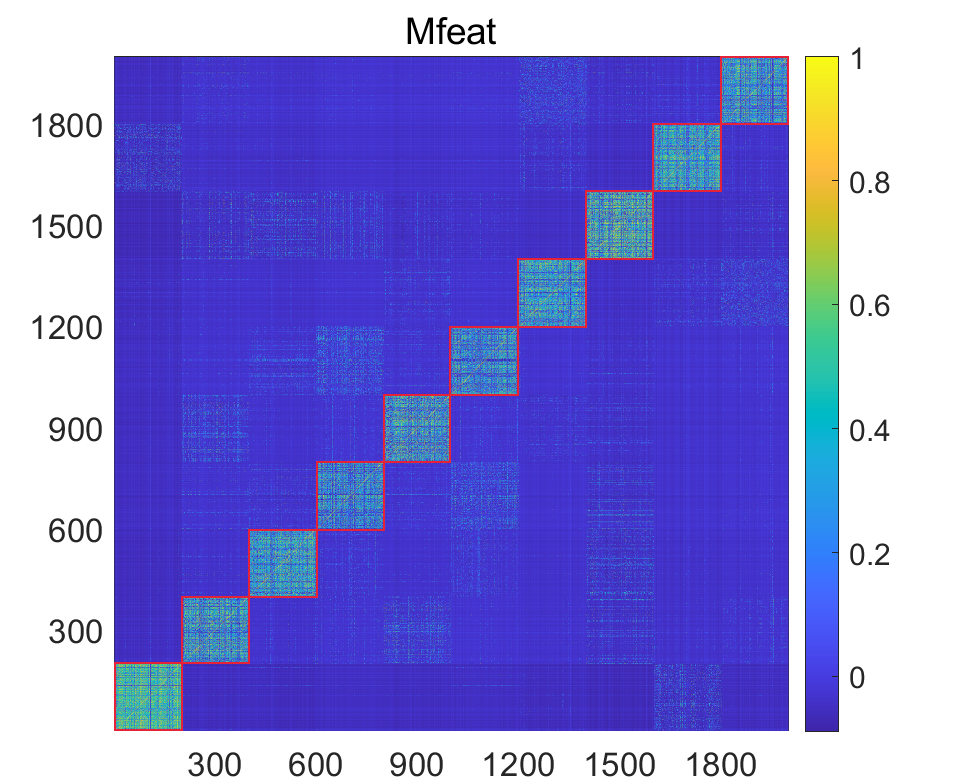}}\hspace{-5mm} \label{KNN-mfeat}}
		\subfloat[Proposed]{{\includegraphics[width=0.19\textwidth]{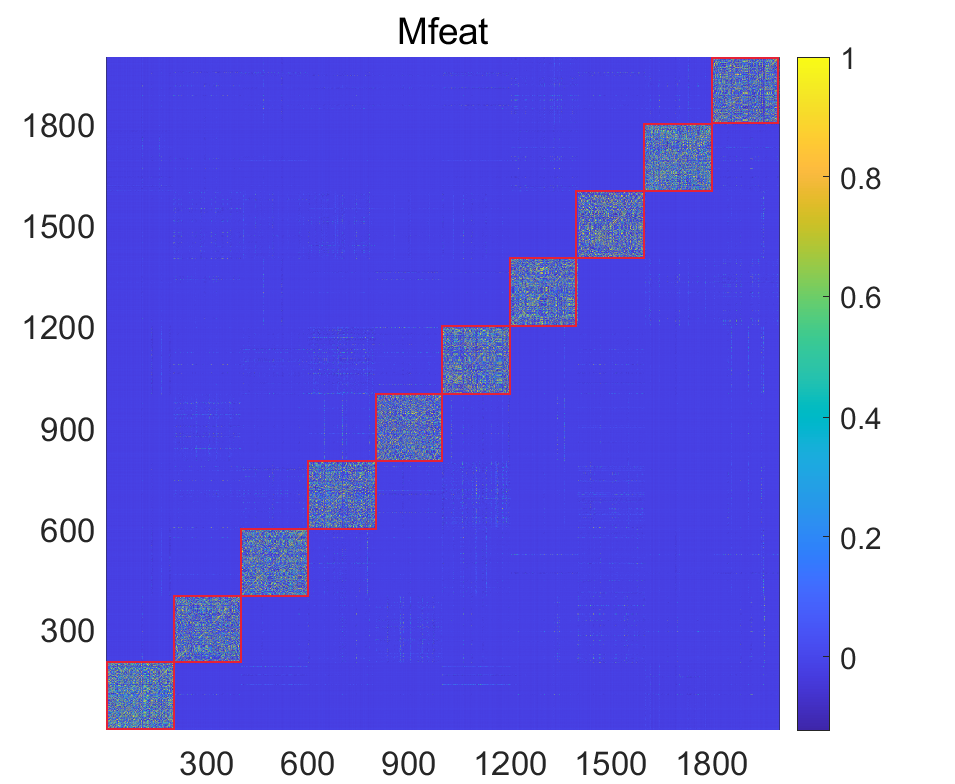}}\hspace{-5mm} \label{proposed-mfeat}}
		\caption{Illustration of (a) original average kernel, (b) localized average kernel in KNN mechanism by carefully tuning $\tau$ within $[0.1,0.2,\cdots,0.9]$ and present the optimal results ($\tau = 0.1$), and (c) localized kernel learned by proposed model on Mfeat dataset.}
		\label{comparison-KNN}
			}
\end{center}
\vspace{-15pt}
\end{figure}
Most of the kernel-based algorithms follow a common assumption that all the samples are reliable to exploit the intrinsic structures of data, thus such a globally designed manner equally calculates the pairwise similarities of all samples\cite{lu2014multiple,du2015robust,liu2016multiple,liu2017optimal,wang2019multi,liu2020simplemkkm,liu2021one}. Nevertheless, in a high-dimensional space, this assumption is incompatible with the well-acknowledged theory that the similarity estimation for distant samples is less reliable on account of the intrinsic manifold structures are highly complex with curved, folded, or twisted characteristics \cite{tenenbaum2000global,li2018dynamic,wang2018spectral,yao2018local}. Furthermore, researchers have found that preserving reliable local manifold structures of data could achieve better effectiveness than globally preserving all the pairwise similarities in unsupervised tasks, and can achieve better clustering performance, such as dimension reduction \cite{roweis2000nonlinear,de2018nonlinear,liu2013global,sun2021projective} and clustering \cite{zhou2007spectral,zhao2018large}.

Therefore, many approaches are proposed to localize kernels to enhance discrimination \cite{li2016multiple,zhu2018localized,zhou2019multiple,zhang2021late,liu2021localized}. The work in \cite{li2016multiple} develops a localized kernel maximizing alignment method that merely aligns the original kernel with $\tau$-nearest neighbors of each sample to the learned optimal kernel. Along this way, KNN mechanism is introduced to kernel-based subspace segmentation \cite{zhou2019multiple}. Moreover, a recently proposed simple MKKM method \cite{liu2020simplemkkm} with min-max optimization is also localized in the same way to consider local structures \cite{liu2021localized}. Besides, such a localized manner also has been extended to handle incomplete data \cite{zhu2018localized}. Although showing improved performance, most traditional localized kernel methods adopt the simple KNN mechanism to select neighbors.\\

As can be seen in Figure \ref{comparison-KNN} (a)-(b), previous localized MKC methods with KNN mechanism encounter two issues: 
\romannumeral1) These methods follow the common assumption that all the neighbors are reliable without considering their variation and ranking relationship. However, it is incompatible with common knowledge that the neighbors of a sample are adaptively varied, and some may have been corrupted by noise or outliers. For instance, in social networking, the closer relationship means more essential and vice versa. 
\romannumeral2) The KNN mechanism introduces a hyper-parameter neighbor ratio, which is fixed for each sample and commonly pre-determined empirically. Apart from this unreasonable fixed neighbor ratio, it incurs dataset-related parameter-tuning in a wide range to obtain satisfying clustering results. From experimental results, we can observe that KNN mechanism still preserves apparent noise compared to the original average kernel.  

To alleviate these problems, we start our work with a natural thought that adaptively assign a reasonable weight to each neighbor according to its ranking importance. However, there is no sufficient prior knowledge in kernel space to identify the ranking relationship among neighbors. Owing to the remarkable performance in exploring the complex nonlinear structures of various data, developing graph-based methods is greatly popular with scholars \cite{nie2016constrained,nie2016unsupervised,nie2016parameter,li2018dynamic,peng2018structured,zhou2019person,wang2019parameter,nie2020self,ren2020consensus,li2020multi,ren2020multiple,ren2020simultaneous,nie2021learning,nie2021implicit,shi2021multi,ren2021simultaneous,Liu2022DCRN}. Considering kernel matrix can be regarded as affinity graph with additional positive semi-definite (PSD) constraint, it is practicable and more flexible to learn a discriminative affinity graph with naturally sparsity and clear block diagonal structures \cite{nie2016constrained,nie2016parameter,nie2017self,nie2020self}. 

Based on the above motivation and our inspiration from graph learning \cite{nie2016constrained,nie2017self,ren2020simultaneous,ren2020consensus,nie2020self,DBLP:journals/tnn/LuoNCYHZ18}, we develop a novel local sample-weighted MKC with consensus discriminative graph (LSWMKC) method. Instead of using KNN mechanism to localize kernel matrix without considering the ranking importance of neighbors, We firstly learn a consensus discriminative affinity graph across multiple views in kernel space to reveal the latent manifold structures, and further heuristically learn an optimal neighborhood kernel. As Figure \ref{comparison-KNN} (c) shows, the learned neighborhood kernel is naturally sparse with clear block diagonal structures. We develop an efficient iterative algorithm to simultaneously learn weights of base kernels, discriminative affinity graph, and localized consensus neighborhood kernel. Instead of empirically tuning or selecting pre-defined neighbor ratio, our model can implicitly optimize adaptive weights on different neighbors with corresponding samples. Extensive experiments demonstrate that the learned neighborhood kernel can achieve clear local manifold structures, and it outperforms localized MKC methods in KNN mechanism and other existing models.
We briefly summarize the main contributions as follows,
\begin{itemize}
    \item A novel local sample-weighted MKC algorithm is proposed based on kernelized graph learning, which can implicitly optimize adaptive weights on different neighbors with corresponding samples according to their ranking importance. 
    \item We learn an optimal neighborhood kernel with more discriminative capacity by further denoising the graph, revealing the latent local manifold representation in kernel space. 
    \item We conduct extensive experimental evaluations on twelve MKC benchmark datasets compared with existing thirteen methods. Our proposed LSWMKC shows apparent effectiveness over localized MKC methods in KNN mechanism and other existing methods.
\end{itemize}
%==========================================================
\section{Background}\label{Background}
This section introduces multiple kernel clustering and traditional KNN-based localized MKC methods.
%----------------------------------------------------------
\subsection{Multiple Kernel $k$-means}\label{MKKMC}
For a data matrix ${\mathbf{X} \in \mathbb{R}^{d \times n}}$ including $n$ samples with ${d}$-dimensional features from ${k}$ clusters, nonlinear feature mapping ${\psi(\cdot) : \mathbb{R}^{d} \mapsto \mathcal{H}}$ achieves the transformation from sample space $\mathbb{R}^{d}$ to a Reproducing Kernel Hilbert Space (RKHS) $\mathcal{H}$ \cite{tzortzis2008global}. Kernel matrix $\mathbf{K}$ is computed by 
\begin{equation}\label{Kij}
\begin{split}
\mathbf{K}_{ij}={\kappa}\left(\mathbf{x}_{i}, \mathbf{x}_{j}\right)=\psi\left(\mathbf{x}_{i}\right)^{\top} \psi\left(\mathbf{x}_{j}\right),
\end{split}
\end{equation}
where ${\kappa} \left(\cdot,\cdot\right) : \mathbb{R}^{d} \times \mathbb{R}^{d} \mapsto \mathbb{R}$ denotes a PSD kernel function.

$k$-means is to minimize the clustering loss, i.e.,
\begin{equation}\label{KKM}
\begin{split}
\min \limits_{\mathbf{S}} \; \sum_{i=1}^{n}\sum_{q=1}^{k}\mathbf{S}_{iq}\|\mathbf{x}_{i}-\mathbf{c}_{q}\|_{2}^{2}, \;\; \textrm{s.t.} \sum_{q=1}^{k} \mathbf{S}_{iq}=1,
\end{split}
\end{equation}
where $\mathbf{S}\in\{0,1\}^{n \times k}$ denotes the indicator matrix, $\mathbf{c}_{q}$ denotes the centroid of $q$-th cluster and $n_{q}=\sum_{i=1}^{n} \mathbf{S}_{i q}$ denotes the corresponding amount of samples. To deal with nonlinear features, the samples are mapped into RKHS $\mathcal{H}$. Kernel $k$-means is formulated as
\begin{equation}\label{KKM-matrix2}
\begin{split}
\min \limits_{\mathbf{H}} \; \mathrm{Tr}\left(\mathbf{K}\left(\mathbf{I}_{n}-\mathbf{H} \mathbf{H}^{\top}\right)\right), \; \textrm {s.t.} \; \mathbf{H}^{\top} \mathbf{H}=\mathbf{I}_{k},
\end{split}
\end{equation}
where partition matrix $\mathbf{H} \in \mathbb{R}^{n \times k}$ is computed by taking rank-$k$ eigenvectors of $\mathbf{K}$ and then exported to $k$-means to compute the final results \cite{scholkopf1998nonlinear,dhillon2004kernel}.

For multiple kernel learning scenarios, $\mathbf{x}$ can be represented as $\psi_{\boldsymbol{\omega}}\left(\mathbf{x}\right) = [\omega_{1} \psi_{1}\left(\mathbf{x}\right)^{\top},\omega_{2} \psi_{2}\left(\mathbf{x}\right)^{\top},\ldots,\omega_{m} \psi_{m}\left(\mathbf{x}\right)^{\top}]^{\top}$, where $\boldsymbol{\omega}=\left[\omega_{1}, \cdots, \omega_{m}\right]^{\top}$ denotes the coefficients of $m$ base kernel functions $\{\kappa_{p}\left(\cdot, \cdot\right)\}_{p=1}^{m}$. ${\kappa}_{\boldsymbol{\omega}} \left(\cdot,\cdot\right)$ is expressed as
\begin{equation}\label{kernel-function}
\begin{split}
 \kappa_{\boldsymbol{\omega}}\left(\mathbf{x}_{i}, \mathbf{x}_{j}\right)=\psi_{\boldsymbol{\omega}}\left(\mathbf{x}_{i}\right)^{\top} \psi_{\boldsymbol{\omega}}\left(\mathbf{x}_{j}\right)=\sum_{p=1}^{m} \omega_{p}^{2} \kappa_{p}\left(\mathbf{x}_{i}, \mathbf{x}_{j}\right). 
 \end{split}
\end{equation}

The objective of MKKM is formulated as
\begin{equation}\label{MKKM}
\begin{split}
\min \limits_{\mathbf{H}, \boldsymbol{\omega}} &\; \operatorname{Tr}\left(\mathbf{K}_{\boldsymbol{\omega}}\left(\mathbf{I}_{n}-\mathbf{H H}^{\top}\right)\right), \\ 
\textrm {s.t.} &\;\; \mathbf{H} \in \mathbb{R}^{n \times k},\; \mathbf{H}^{\top} \mathbf{H}=\mathbf{I}_{k},\; \omega_{p} \geq 0,\; \forall p,
\end{split}
\end{equation}
where the consensus kernel $\mathbf{K}_{\boldsymbol{\omega}} = \sum_{p=1}^{m} \omega_{p}^{2} \mathbf{K}_{p}$ is commonly assumed as a combination of base kernels $\mathbf{K}_{p}$. To control the contribution of different kernels, there are some strategies on ${\boldsymbol{\omega}}$, such as ``kernel affine weight strategy" \cite{ren2020simultaneous}, ``auto-weighted strategy" \cite{nie2016parameter,ren2020consensus}, and ``sum-to-one strategy" \cite{liu2021localized}. According to \cite{huang2012multiple}, Eq. (\ref{MKKM}) can be solved by alternatively optimizing $\boldsymbol{\omega}$ and $\mathbf{H}$. 
%----------------------------------------------------------
\subsection{Construction of Localized Kernel in KNN Mechanism}\label{MKKM-LKAM}
Most kernel-based methods assume that all the samples are reliable and calculate fully connected pairwise similarity. However, as pointed in \cite {tenenbaum2000global,han2018local,li2018dynamic,wang2018spectral,yao2018local}, the similarity estimation of distant-distance samples in high-dimensional space is unreliable. Many localized kernel-based works have been developed to alleviate this problem \cite{li2016multiple,wang2018local,liu2021localized}. Commonly, the localized kernel is constructed in KNN mechanism.       

The construction of localized kernel mainly includes two steps, i.e., neighbor searching and localized kernel construction. Firstly, in average kernel space, the neighbors of each sample are identified by labelling its $\tau$-nearest samples. Denoting the neighbor mask matrix as $\mathbf{N} \in \{0,1\}^{n \times n}$. The neighbor searching is defined as follows,
\begin{equation}\label{neighbor_kernel_Nij}
\begin{split}
\mathbf{N}_{ij}=\left\{\begin{array}{ll} 1, & \mathbf{x}_{j} \in \mathrm{KNN}(\mathbf{x}_{i}), \\ 0, & \text {otherwise},\end{array}\right.
\end{split}
\end{equation}
where $j$ denotes the neighbor index of $i$-th sample. For each row, there are $round(\tau n)$ elements are labelled as neighbors, where neighbor ratio $\tau$ is commonly pre-determined empirically and carefully tuned by grid search, such as $\tau$ varies within $[0.1,0.2,\ldots,0.9]$, and finally obtain the optimal clustering results. If we set neighbor ratio $\tau=1$, the KNN structure will be full-connected. For the pre-computed base kernels $\mathbf{K}_{p}$, the corresponding localized kernel $\mathbf{K}_{p(l)}$ is formulated as
\begin{equation}\label{neighbor_kernel}
\begin{split}
\mathbf{K}_{p(l)}=\mathbf{N} \odot \mathbf{K}_{p},
\end{split}
\end{equation}
where $\odot$ is the Hadamard product. 

Although the traditional KNN mechanism to localize kernel is simple and has improved performance than globally designed methods. This manner neglects a critical issue that the variation of neighbors. Therefore, it is important and practical to assign reasonable weights to different neighbors according to their ranking relationship. Another issue is that the initial neighbor ratio $\tau$ of each sample is usually fixed and pre-determined empirically and needs to be tuned to report the best clustering result. As Figure \ref {comparison-KNN} (a)-(b) shows, the obtained localized kernels preserve much noise, which will incur degeneration of clustering performance.
%==========================================================
\section{METHODOLOGY}
This section presents our proposed LSWMKC in detail and provides an efficient three-step optimization solution. Moreover, we analyze convergence, computational complexity, limitation and extensions.
%----------------------------------------------------------
\subsection{Motivation}
From our aforementioned analysis about the traditional localized kernel method in KNN mechanism, we find that: 
\romannumeral1) This seemingly simple method neglects the ranking importance of the neighbors, which may degrade the clustering performance due to the impact of the unreliable distant-distance relationship. 
\romannumeral2) The neighbor ratio is commonly pre-determined empirically and needs to be tuned to report the best results.   

The above issues inspire us to rethink the manner of constructing localized MKC, and a natural motivation is to exploit their ranking relationship and assign a reasonable weight to each neighbor. However, there is no sufficient prior knowledge in kernel space to identify the ranking importance of neighbors. In recent years, graph-based algorithms have been greatly popular with scholars to explore the nonlinear structures of data. An ideal affinity graph exhibits two good properties: 
\romannumeral1) Clear block diagonal structures with $k$ connected blocks, each corresponding to one cluster. 
\romannumeral2) The affinity represents the similarity of pairwise samples, and the intra-cluster affinities are nonzero, while the extra-cluster affinities are zeros. Considering the kernel matrix can be regarded as the affinity graph with additional PSD constraint, a discriminative graph can reveal the latent local manifold representation in kernel space. These issues inspire us to exploit the capacity of graph learning in capturing nonlinear structures of kernel space.  
%----------------------------------------------------------
\subsection{The Proposed Formula}\label{Formula}
Here, we briefly introduce the affinity graph learning method, which will be the base of our proposed model.

For sample set $\{\mathbf{x}_{1}, \ldots, \mathbf{x}_{n}\}$, it is desirable to learn an affinity graph ${\mathbf{Z} \in \mathbb{R}^{n \times n}}$ with distinct distance $\|\mathbf{x}_i - \mathbf{x}_j \|_{2}^{2}$ corresponding to small similarity ${z}_{ij}$, which is formulated as
\begin{equation}\label{GL-reg}
\begin{split}
\min \limits_{\mathbf{Z}}&\; \sum_{i,j=1}^{n}\left \|\mathbf{x}_{i}-\mathbf{x}_{j}\right\|^{2}_{2} z_{ij}+\gamma z_{ij}^{2}, \\
\textrm{s.t.}& \;\;\mathbf{Z}_{i,:} \mathbf{1}_{n}=1, \; z_{ij} \geq 0, \;z_{ii} = 0,
\end{split}
\end{equation}
where $\gamma$ is a hyper-parameter, $\mathbf{Z}_{i,:} \mathbf{1}_{n}=1$ is for normalization, $z_{ij} \geq 0$ is to ensure the non-negative property, $z_{ii} = 0$ can avoid trivial solutions. Commonly, the second term $\ell_{2}$ norm regularization is to avoid undesired trivial solutions \cite{du2015unsupervised,nie2016unsupervised}.

However, the existing graph-based methods are developed in sample space $\mathbb{R}^{d}$ rather than RKHS $\mathcal{H}$ kernel space, significantly limiting their applications. To fill this gap and exploit their potent capacity to capture nonlinear structures in kernel space, by using kernel tricks, the first term of Eq. (\ref{GL-reg}) can be extended as
\begin{equation}\label{mapping}
\resizebox{0.5\textwidth}{!}{$
\begin{split}
\min \limits_{\mathbf{Z}}& \sum_{i,j=1}^{n} \|\psi(\mathbf{x}_{i})-\psi(\mathbf{x}_{j})\|^{2}_{2} z_{ij} \\
= \min \limits_{\mathbf{Z}}& \sum_{i,j=1}^{n} (\psi(\mathbf{x}_{i})^{\top}\psi(\mathbf{x}_{i})-2\psi(\mathbf{x}_{i})^{\top}\psi(\mathbf{x}_{j})+\psi(\mathbf{x}_{j})^{\top}\psi(\mathbf{x}_{j}))z_{ij}\\
= \min \limits_{\mathbf{Z}}& \sum_{i,j=1}^{n} (\kappa({\mathbf{x}_{i},\mathbf{x}_{i}})-2\kappa({\mathbf{x}_{i},\mathbf{x}_{j}})+\kappa({\mathbf{x}_{j},\mathbf{x}_{j}}))z_{ij}\\
= \min _{\mathbf{Z}}& \; 2n - \sum_{i,j=1}^{n} 2\kappa({\mathbf{x}_{i},\mathbf{x}_{j}})z_{ij}
\Leftrightarrow \min_{\mathbf{Z}} \sum_{i,j=1}^{n} -\kappa({\mathbf{x}_{i},\mathbf{x}_{j}})z_{ij} \\
 \textrm{s.t.}&\;\; \mathbf{Z}_{i,:} \mathbf{1}_{n}=1, \; z_{ij} \geq 0, \;z_{ii} = 0.
\end{split}$}
\end{equation}

Note that the condition for Eq. (\ref{mapping}) is that we assume $\kappa({\mathbf{x}_{i},\mathbf{x}_{i}}) = 1$. However, it is not always valid for all the kernel functions. A common choice is the Gaussian kernel which satisfies $\kappa(\mathbf{x}_i,\mathbf{x}_i) = 1$. The present work utilizes this manner or directly downloads the public kernel datasets. Moreover, all the base kernels are firstly centered and then normalized following \cite{cortes2012algorithms, shawe2004kernel}, which further guarantees $\kappa({\mathbf{x}_{i},\mathbf{x}_{i}}) = 1$.   

We have the following insights from the kernelized affinity graph learning model:
\romannumeral1) Compared to using $\|\mathbf{x}_i - \mathbf{x}_j \|_{2}^{2}$ to estimate the pairwise distance in sample space, we should adopt $-\kappa(\mathbf{x}_{i},\mathbf{x}_{j})$ in kernel space.          
\romannumeral2) Such compact form achieves affinity graph learning in kernel space to explore the complex nonlinear structures. 

In multiple kernel learning scenarios, it is commonly assumed that the ideal kernel is optimally combined by given base kernels, Eq. (\ref{mapping}) can be extended as    
\begin{equation}\label{GL-K-MV}
\begin{split}
\min \limits_{\mathbf{Z}, \boldsymbol{\omega}}& \;\sum_{p=1}^{m} \sum_{i,j=1}^{n}\; \; -\omega_{p} {\kappa}_{p} (\mathbf{x}_{i}, \mathbf{x}_{j}) z_{ij} + \gamma z_{ij}^{2}, \\
\textrm{s.t.}&\;\;\left\{ 
     \begin{array}{lr}
        \mathbf{Z}_{i,:} \mathbf{1}_{n}=1, \; z_{ij} \geq 0, \;z_{ii} = 0,\\
        \sum_{{p}=1}^{m} \omega_{p}^{2} = 1, \; \omega_{p} \geq 0,
     \end{array}
     \right.
\end{split}
\end{equation}
where $\omega_{p}$ is the weight of $p$-th base kernel. Since using $\sum_{p=1}^{m}\omega = 1$ will only activate the best kernel, and it incurs the multi-kernel scenario degraded into the undesirable single-kernel scenario. We employ the squared $\ell_{2}$ norm constraint of $\omega_{p}$ to smooth the weights and avoid the sparse trivial solution. Other weight strategies can refer to \cite{nie2016parameter,ren2020consensus,ren2020simultaneous}. The above formula achieves multiple kernel-based graph learning by jointly optimizing kernel weights and consensus affinity graph. Specifically, the learned consensus discriminative graph reveals kernel space's intrinsic local manifold structures by graph learning mechanism and fuses latent clustering information across multiple kernels by weight learning mechanism. 

Recall we aim to estimates the ranking relationship of neighbors with corresponding samples in kernel space. The above discriminative consensus graph inspires us to further learn an optimal neighborhood kernel, which obtains a consensus kernel with naturally sparse property and precise block diagonal structures. This idea can be naturally modeled by minimizing squared $\mathrm{F}$-norm loss $\|\mathbf{K}^\ast - \mathbf{Z}\|_\mathrm{F}^{2}$ with constraints $\mathbf{K}^{\ast} \succeq 0$ and $\mathbf{K}^{\ast} = \mathbf{K}^{\ast\top}$. We define the optimization goal as follows,
\begin{equation}\label{GLK-Matrix}
\begin{split}
 \min \limits_{\mathbf{Z},\mathbf{K}^\ast,\boldsymbol{\omega}}&\; - \mathrm{Tr}\left(\sum_{p=1}^m \omega_{p} \mathbf{K}_{p} \mathbf{Z}^{\top}\right) + \|\mathbf{G} \odot \mathbf{Z}\|_{\mathrm{F}}^{2} + \alpha \|\mathbf{K}^\ast - \mathbf{Z}\|_{2}^{2},  \\
  \textrm{s.t.}&\;\;\left\{
     \begin{array}{lr}
        \mathbf{Z} \mathbf{1}_{n} = \mathbf{1}_{n}, \; \mathbf{Z} \geq 0, \; \mathbf{Z}_{ii} = 0,\\
        \mathbf{K}^{\ast} \succeq 0, \; \mathbf{K}^{\ast} = \mathbf{K}^{\ast\top}, \;
        \sum_{{p}=1}^{m} \omega_{p}^{2} = 1, \; \omega_{p} \geq 0,
     \end{array}
     \right.
\end{split}
\end{equation}
where $\mathbf{G} = \mathbf{1}_{n}^{\top} \otimes \boldsymbol{\gamma}$, $\boldsymbol{\gamma} = (\sqrt{\gamma_{1}},\sqrt{\gamma_{2}},\cdots,\sqrt{\gamma_{n}})^{\top}$ denotes hyper-parameter $\gamma_{i}$ with corresponding $i$-row of $\mathbf{Z}$, $\otimes$ is outer product, $\odot$ is Hadamard product, and $\alpha$ is the balanced hyper-parameter for neighborhood kernel construction. 

Note that $n$ hyper-parameters $\gamma$ corresponding to $n$ rows of $\mathbf{Z}$ respectively, which is due to the following considerations:
\romannumeral1) As our analysis in Eq. (\ref{GL-K-MV}), reasonable hyper-parameters $\gamma$ can avoid trivial solutions, i.e., $\gamma \rightarrow {0}$ or $\gamma \rightarrow {\infty}$ will incur undesired extremely sparse or dense affinity matrix respectively. \romannumeral2) Section \ref{Sub-Z} also illustrates the sub-problem of optimizing $\mathbf{Z}$ involves $n$-row formed independent optimization. It is reasonable to assign different $\gamma_{i}$ to each problem, considering their variations. Such issues inspire us to learn reasonable $\gamma$ instead of empirical and time-consuming parameter-tuning. We derive a theoretical solution in Section \ref{Initialization} and experimentally validate the ablation study on tuning $\gamma$ by grid search in Section \ref{Ablation}.

From the above formula, our proposed LSWMKC model jointly optimizes the kernel weights, the consensus affinity graph, and the consensus neighborhood kernel into a unified framework. Although the formula is straightforward, LSWMKC has the following merits: 
\romannumeral1) It addresses localized kernel problem via a heuristic manner rather than the traditional KNN mechanism, which achieves implicitly optimizing adaptive weights on different neighbors with corresponding samples according to their ranking relationship. 
\romannumeral2) Instead of tuning hyper-parameter $\boldsymbol{\gamma}$ by grid search, we propose an elegant solution to pre-determine it.
\romannumeral3) More advanced graph learning methods in kernel space can be easily introduced to this framework.
%----------------------------------------------------------
\subsection{Optimization}
Simultaneously optimizing all the variables in Eq. (\ref{GLK-Matrix}) is difficult since the optimization objective is not convex. This section provides an effective alternate optimization strategy by optimizing each variable with others been fixed. The original problem is separated into three sub-problems such that each one is convex. 
%----------------------------------------------------------
%
\subsubsection{Optimization ${\omega}_{p}$ with fixed $\mathbf{Z}$ and ${\mathbf{K}^\ast}$}
With fixed $\mathbf{Z}$ and $\mathbf{K}^\ast$, the objective in Eq. ({\ref{GLK-Matrix}}) is formulated as
\begin{equation}\label{weight}
\begin{split}
\max \limits_{\boldsymbol{\omega}}\; \sum_{p=1}^{m} \omega_{p} \delta_{p}, \;\; \textrm{s.t.} \; \sum_{p=1}^{m} \omega_{p}^{2}=1, \omega_{p} \geq 0,
\end{split}
\end{equation}
where $\delta_{p} = \mathrm{Tr}(\mathbf{K}_{p} \mathbf{Z}^{\top})$. This problem could be easily solved with closed-form solution as follows,
\begin{equation}\label{weight_closed_solution}
\begin{split}
\omega_{p} = \frac{\delta_{p}}{\sqrt{\sum_{p=1}^{m} \delta_{p}^{2}}}.
\end{split}
\end{equation}

The computational complexity is $\mathcal{O}(mn^{2})$.
%----------------------------------------------------------
\subsubsection{Optimization $\mathbf{Z}$ with fixed $\mathbf{K}^{\ast}$ and ${\omega}_{p}$} \label{Sub-Z}
With fixed $\mathbf{K}^{\ast}$ and $\omega_{p}$, Eq. ({\ref{GLK-Matrix}}) is transformed to $n$ sub-optimization problems, and each one can be independently solved by 
%----------------
\begin{equation}\label{Opt-Z}  
\begin{split}
\min \limits_{\mathbf{Z}_{i,:}}& \; \left(\gamma_{i} + \alpha \right){\mathbf{Z}_{i,:}} \mathbf{Z}_{i,:}^{\top} - \left(2\alpha \mathbf{K}_{i,:}^\ast +\sum\nolimits_{p=1}^m \omega_{p} \mathbf{K}_{p[i,:]} \right)\mathbf{Z}_{i,:}^{\top},\\
\textrm{s.t.}& \;\; \mathbf{Z}_{i,:} \mathbf{1}_{n}=1, \; \mathbf{Z}_{i,:} \geq 0, \; \mathbf{Z}_{ii} = 0,
\end{split}
\end{equation}
where $\mathbf{K}_{p[i,:]}$ denotes the i-th row of the $p$-th base kernel.

Further, Eq. (\ref{Opt-Z}) can be rewritten as Quadratic Programming (QP) problem, 
%----------------
\begin{equation}\label{Z-QP} 
\begin{split}
\min \limits_{\mathbf{Z}_{i,:}}& \; \frac{1}{2} \mathbf{Z}_{i,:} \mathbf{AZ}_{i,:}^{\top}+{\mathbf{e}}_{i} \mathbf{Z}_{i,:}^{\top},  \\
\textrm{s.t.}& \;\; \mathbf{Z}_{i,:} \mathbf{1}_{n}=1, \; \mathbf{Z}_{i,:} \geq 0, \; \mathbf{Z}_{ii} = 0,
\end{split}
\end{equation}
where $\mathbf{A} = 2(\gamma_{i} + \alpha) \mathbf{I}_{n}, \; {\mathbf{e}}_{i} = - \left(2 \alpha \mathbf{K}^\ast_{i,:} + \sum\nolimits_{p=1}^m \omega_{p} \mathbf{K}_{p[i,:]} \right)$. The global optimal solution of QP problem can be easily solved by the toolbox of MATLAB. Since $\mathbf{Z}_{i,:}$ is a $n$-dimensional row vector, the computational complexity of Eq. (\ref{Z-QP}) is $\mathcal{O}(n^{3} + mn)$ and the total complexity is $\mathcal{O}(n^{4} + mn^{2})$.

%----------------
Furthermore, the above Eq. (\ref{Z-QP}) can be simplified as,
\begin{equation}\label{QP-closed-solution}
\begin{split}
\min \limits_{\mathbf{Z}_{i,:}}& \; \frac{1}{2}\left\|\mathbf{Z}_{i,:}-\hat{\mathbf{Z}}_{i,:}\right\|_{2}^{2}, \\
  \textrm{s.t.}& \;\; \mathbf{Z}_{i,:} \mathbf{1}_{n}=1, \; \mathbf{Z}_{i,:} \geq 0, \; \mathbf{Z}_{ii} = 0,
\end{split}
\end{equation}
where $\hat{\mathbf{Z}}_{i,:} = - \frac{{\mathbf{e}}_{i}}{2(\alpha + \gamma_{i})}$.

%----------------
Mathematically, the following Theorem \ref{Theorem 1} illustrates that the solution of Eq. (\ref{QP-closed-solution}) can be analytically solved.
%----------------
\begin{theorem} 
\label{Theorem 1}
The analytical solution of Eq. (\ref{QP-closed-solution}) is as follows,
\begin{equation}\label{QP-closed-Theorem}
\begin{split}
\mathbf{Z}_{i,:}=\max \left(\hat{\mathbf{Z}}_{i,:}+\beta_{i} \mathbf{1}_{n}^{\top}, 0 \right), \; \mathbf{Z}_{ii} = 0, 
% \; \beta_{i}=\frac{1+\hat{\mathbf{Z}}_{i,:} \mathbf{1}_{n}}{n}.
\end{split}
\end{equation}
where $\beta_{i}$ can be solved by Newton's method efficiently. 
% is obtained by solving $\mathbf{Z}_{i,:} \mathbf{1}_{n}=1$. 
\end{theorem}
%----------------
\begin{proof} 
For $i$-th row of $\mathbf{Z}$, the Lagrangian function of Eq. (\ref{QP-closed-solution}) is as follows,
\begin{equation}\label{QP-closed-proof}
\begin{split}
\mathcal{L}\left(\mathbf{Z}_{i,:}, \beta_{i}, \boldsymbol{\eta}_{i} \right)= \frac{1}{2}\left\|\mathbf{Z}_{i,:}-\hat{\mathbf{Z}}_{i,:}\right\|_{2}^{2}-\beta_{i}\left(\mathbf{Z}_{i,:} \mathbf{1}_{n}-1\right)-\boldsymbol{\eta}_{i} \mathbf{Z}_{i,:}^{\top},
\end{split}
\end{equation}
where scalar $\beta_{i}$ and row vector $\boldsymbol{\eta}_{i}$ are Lagrangian multipliers. According to the KKT condition, 
\begin{equation}\label{QP-closed-1}
\begin{split}
 \left\{\begin{array}{l}\mathbf{Z}_{i,:}-\hat{\mathbf{Z}}_{i,:}-\beta_{i} \mathbf{1}_{n}^{\top}-\boldsymbol{\eta}_{i}=\mathbf{0}^{\top}, \\ 
 \boldsymbol{\eta}_{i} \odot \mathbf{Z}_{i,:}=\mathbf{0}^{\top}, \end{array}\right.
\end{split}
\end{equation}

We have 
\begin{equation}\label{QP-closed-2}
\begin{split}
\mathbf{Z}_{i,:}=\max \left(\hat{\mathbf{Z}}_{i,:}+\beta_{i} \mathbf{1}^{\top}_{n}, 0 \right), \; \mathbf{Z}_{ii} = 0.
\end{split}
\end{equation}

Note that $\mathbf{Z}_{i,:} \mathbf{1}_{n}=1$, and $\mathbf{Z}_{i,:} \mathbf{1}_{n}$ increases monotonically with respect to $\beta_{i}$ according to Eq. (\ref{QP-closed-2}), $\beta_{i}$ can be solved by Newton's method efficiently. This completes the proof.
\end{proof}
By computing the closed-formed solution, the computational complexity of Eq. (\ref{Z-QP}) is reduced to $\mathcal{O}(mn)$, which is mainly from computing ${\mathbf{e}}_{i}$. The total complexity is $\mathcal{O}(mn^{2})$. 
%----------------------------------------------------------
\subsubsection{Optimization $\mathbf{K}^\ast$ with fixed $\mathbf{Z}$ and ${\omega}_{p}$}
With fixed $\mathbf{Z}$ and $\omega_{p}$, the original objective Eq. ({\ref{GLK-Matrix}}) can be converted to
\begin{equation}\label{Opt-K*}
\begin{split}
\min \limits_{\mathbf{K}^\ast}& \; \left\|\mathbf{K}^\ast - \mathbf{Z}\right\|_\mathrm{F}^{2},  \\
  \textrm{s.t.}& \;\; \mathbf{K}^\ast \succeq 0, \;\mathbf{K}^\ast = \mathbf{K}^{\ast\top}.
\end{split}
\end{equation}

However, this seemingly simple sub-problem is hard to be directly solved. Theorem \ref{Theorem 2} provides an equivalent solution. 
%----------------
\begin{theorem}\label{Theorem 2}
The optimization in Eq. (\ref{Opt-K*}) has the same solution as Eq. (\ref{Opt-K*-proof}),
\begin{equation}\label{Opt-K*-proof}
\begin{split}
\min \limits_{\mathbf{K}^\ast}& \; \left \|\mathbf{K}^\ast -\frac{1}{2} (\mathbf{Z}+\mathbf{Z}^{\top}) \right\|_\mathrm{F}^{2},  \\
  \textrm{s.t.}& \;\; \mathbf{K}^\ast \succeq 0, \;\mathbf{K}^\ast = \mathbf{K}^{\ast\top}.
\end{split}
\end{equation}
\end{theorem}
%----------------
\begin{proof}
According to the PSD property of $\mathbf{K}^\ast$, we can derive that the original optimization objective $\left\|\mathbf{K}^\ast - \mathbf{Z}\right\|_\mathrm{F}^{2}$ in Eq. (\ref{Opt-K*}) is equivalent to $\left\|\mathbf{K}^\ast - \mathbf{Z}^{\top}\right\|_\mathrm{F}^{2}$. Therefore, the solution of Eq. (\ref{Opt-K*}) is the same as Eq. (\ref{Opt-K*-proof}). This completes the proof.
\end{proof}
%----------------
According to Theorem \ref{Theorem 2}, supposing the eigenvalue decomposition result of $(\mathbf{Z}+\mathbf{Z}^{\top})/2$ is $\mathbf{U}_\mathbf{Z} \mathbf{\Sigma}_\mathbf{Z} \mathbf{U}^{\top}_\mathbf{Z}$. The optimal $\mathbf{K}^\ast$ can be easily obtained by imposing $\mathbf{K}^\ast = \mathbf{U}_\mathbf{Z} \mathbf{\Sigma} \mathbf{U}^{\top}_\mathbf{Z}$, where $\mathbf{\Sigma} = \max(\mathbf{\Sigma}_\mathbf{Z}, 0)$. Note that the learned $\mathbf{K}^\ast$ can further denoise the $\mathbf{Z}$ from the above optimization. Once we obtain $\mathbf{K}^\ast$, it is exported to KKM to calculate the final results.

%----------------------------------------------------------
\subsection{Initialize the Affinity Graph $\mathbf{Z}$ and Hyper-parameter $\gamma_{i}$}\label{Initialization}
For graph-based clustering methods, the performance is sensitive to the initial affinity graph. A bad graph construction will degrade the overall performance. For the proposed algorithm, we aim to learn a neighborhood kernel $\mathbf{K^{\ast}}$ of the consensus affinity graph $\mathbf{Z}$. This section proposes a strategy to initialize the affinity matrix $\mathbf{Z}$ and the hyper-parameter $\gamma_{i}$.

Recall our objective in Eq. (\ref{GLK-Matrix}), a sparse discriminative affinity graph is preferred. Theoretically, by constraining $\gamma_{i}$ within reasonable bounds, $\mathbf{Z}$ will be naturally sparse. The $c$ nonzero values of $\mathbf{Z}_{i,:}$ denotes the affinity of each instance corresponding to its initialized neighbours. Therefore, with all the other parameters fixed, we learn an initialized $\mathbf{Z}$ with the maximal $\gamma_{i}$. Based on our objective in Eq. (\ref{GLK-Matrix}), by constraining the $\ell_{0}$-norm of $\mathbf{Z}_{i,:}$ to be $c$, we solve the following problem:
\begin{equation}\label{max-gamma}
\begin{split}
\max \limits_{\gamma_{i}} \;\gamma_{i}, \;\; \textrm{s.t.} \; \|\mathbf{Z}_{i,:}\|_{0}=c.
\end{split} 
\end{equation}

Recall the sub-problem of optimizing $\mathbf{Z}$ in Eq. (\ref{QP-closed-solution}), its equivalent form can be written as follows,
\begin{equation}\label{GL-reg-Simplified}
\begin{split}
\min \limits_{\mathbf{Z}_{i,:} \mathbf{1}_{n}=1, \;\mathbf{Z}_{i,:} \geq 0, \;\mathbf{Z}_{ii}=0}\; \frac{1}{2}\left\|\mathbf{Z}_{i,:}+\frac{\mathbf{e}_{i}}{2(\alpha + \gamma_{i})}\right\|_{2}^{2},
\end{split}
\end{equation}
where $\mathbf{e}_{i} = - \left(2 \alpha \mathbf{K}^\ast_{i,:} + \sum\nolimits_{p=1}^m \omega_{p} \mathbf{K}_{p[i,:]} \right)$.
The Lagrangian function of Eq. (\ref{GL-reg-Simplified}) is
\begin{equation}\label{Lag-GL-reg-Simplified}
\begin{split}
\small{\mathcal{L}\left(\mathbf{Z}_{i,:}, \zeta, \boldsymbol{\lambda}_{i}\right)=\frac{1}{2}\left\|\mathbf{Z}_{i,:}+\frac{\mathbf{e}_{i}}{2(\alpha + \gamma_{i})}\right\|_{2}^{2}-\zeta\left(\mathbf{Z}_{i,:} \mathbf{1}_{n}-1\right)-\boldsymbol{\lambda}_{i} \mathbf{Z}_{i,:}^{\top},}
\end{split}
\end{equation}
where scalar $\zeta$ and row vector $\boldsymbol{\lambda}_{i} \geq \mathbf{0}^{\top}$ denote the Lagrange multipliers.
The optimal solution $\mathbf{Z}_{i,:}^{\ast}$ satisfy that the derivative of Eq. (\ref{Lag-GL-reg-Simplified}) equal to zero, that is
\begin{equation}\label{Lag-derivative}
\begin{split}
\mathbf{Z}_{i,:}^{\ast}+\frac{\mathbf{e}_{i}}{2(\alpha + \gamma_{i})}-\zeta \mathbf{1}^{\top}_{n}-\boldsymbol{\lambda}_{i}=\mathbf{0}^{\top}.
\end{split}
\end{equation}

For the ${j}$-th element of $\mathbf{Z}_{i,:}^{\ast}$, we have
\begin{equation}\label{Lag-der-element}
\begin{split}
{z}_{ij}^{\ast}+\frac{e_{ij}}{2(\alpha + \gamma_{i})}-\zeta -\lambda_{ij}={0}.
\end{split}
\end{equation}

According to KKT condition that ${z}_{ij} \lambda_{ij} ={0}$, we have
\begin{equation}\label{z*}
\begin{split}
{z}_{ij}^{\ast}=\max\left(-\frac{e_{ij}}{2(\alpha + \gamma_{i})}+\zeta,0 \right).
\end{split}
\end{equation}

To construct a sparse affinity graph with $c$ valid neighbors, we suppose each row ${e}_{i1}, {e}_{i2}, \ldots, {e}_{in}$ are ordered in ascending order. Naturally, ${e}_{ii}$ ranks first. Considering $\mathbf{Z}_{ii} = 0$, the invalid ${e}_{ii}$ should be neglected since the similarity with itself is useless. That is $\mathbf{Z}_{i,2},\mathbf{Z}_{i,3},\cdots, \mathbf{Z}_{i,c+1} \textgreater 0$ and $\mathbf{Z}_{i,c+2}, \mathbf{Z}_{i,c+3}, \cdots, \mathbf{Z}_{i,n} = 0$, we further derive
\begin{equation}\label{neibour}
\begin{split}
 -\frac{e_{i,c+1}}{2(\alpha + \gamma_{i})}+\zeta \textgreater 0, \;\; -\frac{e_{i,c+2}}{2(\alpha + \gamma_{i})}+\zeta \leq 0.
\end{split}
\end{equation}

According to Eq. (\ref{z*}) and constraint $\mathbf{Z}_{i,:} \mathbf{1}_{n}=1$, we obtain
\begin{equation}\label{zeta}
\begin{split}
\sum_{j=2}^{c+1}\left(-\frac{e_{ij}}{2(\alpha + \gamma_{i})}+\zeta\right)=1. 
\end{split}
\end{equation}

$\zeta$ is formulated as
\begin{equation}\label{eta_fomula}
\begin{split}
\zeta=\frac{1}{c}+\frac{1}{2c(\alpha + \gamma_{i})} \sum_{j=2}^{c+1} e_{ij}.
\end{split}
\end{equation}

Therefore, we have
\begin{equation}\label{gamma_interval}
\begin{split}
\frac{c}{2} e_{i,c+1}-\frac{1}{2} \sum_{j=2}^{c+1} e_{ij} -\alpha < \gamma_{i} \leq \frac{c}{2} e_{i,c+2}-\frac{1}{2} \sum_{j=2}^{c+1} e_{ij} -\alpha. 
\end{split}
\end{equation}

According to the aforementioned derivation, to satisfy $\|\mathbf{Z}^{\ast}_{i,:}\|_{0}=c$, the maximal $\gamma_{i}$ is as follows,
\begin{equation}\label{gamma_max}
\begin{split}
\gamma_{i} = \frac{c}{2} e_{i, c+2}-\frac{1}{2} \sum_{j=2}^{c+1} e_{ij} -\alpha. 
\end{split}
\end{equation}

In the meantime, the initial ${z}_{ij}^{\ast}$ is as follows,
\begin{equation}\label{initial_Z}
\begin{split}
{z}_{ij}^{\ast}=\left\{\begin{array}{cl}\frac{e_{i, c+2}-e_{i,j+1}}{c e_{i, c+2}-\sum_{h=2}^{c+1} e_{i h}}, & j \leq c, \\ 0, & j>c.\end{array}\right. 
\end{split}
\end{equation}

From the above analysis, we initialize a sparse discriminative affinity graph with each row having $c$ nonzero values and derive the maximal $\gamma_{i}$. Note that Eq. (\ref{gamma_interval}) involves an undesired hyper-parameter $\alpha$, to get rid of its impact, we directly impose $\alpha = 0$. Once the initial $\gamma_{i}$ are computed, these coefficients will remain unchanged during the iteration. According to the initialization, we have the following observations: 
\romannumeral1) The construction is simple with basic operations, but can effectively initialize a sparse discriminative affinity graph with block diagonal structures, contributing to the subsequent learning process. 
\romannumeral2) The hyper-parameter $\gamma_{i}$ can be pre-determined to avoid the undesired tuning by grid search. 
\romannumeral3) Initializing the affinity graph involves a parameter, i.e., the number of neighbors $c$. For most cases, $ 5 \leq {c} \leq 10$ is likely to achieve reasonable results and $c$ is fixed at 5 in this work.
%----------------------------algorithm-----------------------
\begin{algorithm}
\caption{LSWMKC}
\label{proposed algorithm}
	\SetAlgoRefName{} 
	\SetAlgoLined
	\KwIn{Base kernel matrices $\{\mathbf{K}_{p}\}_{p=1}^m$,
		clusters $k$, neighbors $c$,
		hyper-parameter $\alpha$.}
	\KwInit{$\mathbf{Z}$ \rm {by Eq. (\ref{initial_Z})}; $\mathbf{K}^{\ast} = \sum\nolimits_{p=1}^m \omega_{p} \mathbf{K}_{p}$; $\gamma_{i}$ \rm {by Eq. (\ref{gamma_max})}; $\omega_{p} = \sqrt{1/m}$.}\\
		\While{not converged}{
		Compute $\omega_{p}$ according to Eq. (\ref{weight})\;
		Compute $\mathbf{Z}$ according to Eq. (\ref{QP-closed-solution})\;
		Compute $\mathbf{K}^{\ast}$ according to Eq. (\ref{Opt-K*-proof})\;
}
	\KwOut{Perform kernel $k$-means on $\mathbf{K}^{\ast}$.}
\end{algorithm}
	\vspace{-15pt}
%----------------------------------------------------------
\subsection{Analysis and Extensions} 
\textit{Computational Complexity}: According to the aforementioned alternate optimization steps, the computational complexity of our LSWMKC model includes three parts. Updating $\omega_{p}$ in Eq. (\ref{weight}) needs $\mathcal{O}(mn^{2})$ to obtain the closed-form solution. When updating $\mathbf{Z}$, the complex QP problem in Eq. (\ref{Z-QP}) is transformed into an equivalent closed-form solution in Eq. (\ref{QP-closed-solution}) whose computational complexity is $\mathcal{O}(mn^{2})$. Updating $\mathbf{K}^{\ast}$ in Eq. (\ref{Opt-K*-proof}) needs $\mathcal{O}(n^{3})$ cost by eigenvalue decomposition. Commonly, $n \gg m$, the total computational complexity of our LSWKMC is $\mathcal{O}(n^{3})$ in each iteration. 

For the post-processing of $\mathbf{K}^{\ast}$, we perform kernel $k$-means to obtain the clustering partition and labels whose computational complexity is $\mathcal{O}(n^{3})$. Although the computational complexity of our LSWMKC algorithm is the same as the compared models \cite{huang2012multiple, gonen2014localized, li2016multiple, liu2016multiple, liu2017optimal, liu2020simplemkkm, ren2020consensus, ren2020simultaneous, liu2021localized}, its clustering performance exhibits significant improvement as reported in Table \ref{Comparasion_Kpi}.

\textit{Convergence}: Jointly optimizing all the variables in Eq. (\ref{GLK-Matrix}) is problematic since our algorithm is non-convex. Instead, as Algorithm \ref{proposed algorithm} shows, we adopt an alternate optimization manner, and each of sub-problems is strictly convex. For each sub-problem, the objective function decreases monotonically during iteration. Consequently, as pointed out in \cite{Bezdek2003}, the proposed model can theoretically obtain a local minimum solution. In experiments, Figure \ref{Convergence} plots the objective evolution during iteration, further validating the convergence.

\textit{Limitation and Extension}: The proposed model provides a heuristic insight on the localized mechanism in kernel space. Nevertheless, we should emphasize that the promising performance obtained at the expense of $\mathcal{O}(n^{3})$ computational complexity, which limits wide applications in large-scale clustering. Introducing more advanced and efficient graph learning methods to this framework deserve future investigation, especially for prototype or anchor learning \cite{li2020multi,9646486,nie2021learning}, which may reduce the complexity from $\mathcal{O}(n^{3})$ to $\mathcal{O}(n^{2})$, even $\mathcal{O}(n)$. Moreover, the present work still requires post-processing to get the final clustering results, i.e. $k$-means. Interestingly, several concise strategies, such as rank constraint \cite{nie2016constrained,ren2020consensus,nie2021learning} or one-pass manner \cite{liu2021one}, provide promising solutions of directly obtaining the clustering labels, these deserve further research.
%==========================================================
\section{Experiment}
\label{experiments}
This section conducts extensive experiments to evaluate the performance of our proposed algorithm, including clustering performance, running time, comparison with KNN mechanism, kernel weights, visualization, convergence, parameter sensitivity analysis, and ablation study. 
%----------------------------------------------------------
\subsection{Datasets}
\begin{table}[h]
\vspace{-5pt}
\small
\caption{Datasets summary}
\centering
\begin{tabular}{c|c|c|c}\hline
\hline
Datasets                &      Samples &      Views &   Clusters \\\hline
YALE                    &      165 &          5 &         15 \\
MSRA                    &      210 &          6 &          7 \\
Caltech101-7            &      441 &          6 &          7 \\
PsortPos                &      541 &         69 &          4 \\
BBC                     &      544 &          2 &          5 \\
BBCSport                &      544 &          6 &          5 \\
ProteinFold             &      694 &         12 &         27 \\
PsortNeg                &     1444 &         69 &          5 \\
Caltech101-mit          &     1530 &         25 &        102 \\
Handwritten             &     2000 &          6 &         10 \\
Mfeat                   &     2000 &         12 &         10 \\
Scene15                 &     4485 &          3 &         15 \\
\hline\hline
\end{tabular}
\label{Datasets}
\vspace{-8pt}
\end{table}
Table \ref{Datasets} lists the twelve widely employed multi-kernel benchmark datasets, including\\
(1) {\textbf{YALE}\footnote{\footnotesize{\texttt{http://vision.ucsd.edu/content/yale-face-database}}}} includes 165 face gray-scale images from 15 individuals with different facial expressions or configurations, and each subject includes 11 images.\\
(2) \textbf{MSRA} derived from MSRCV1 \cite{winn2005locus}, contains 210 images with 7 clusters, including airplane, bicycle, building, car, caw, face, and tree.\\
(3) {\textbf{Caltech101-7 and Caltech101-mit}\footnote{\footnotesize{\texttt{http://www.vision.caltech.edu/Image\_Datasets/\\Caltech101/}}}} originated from Caltech101, including 101 object categories (e.g., ``face", ``dollar bill", and ``helicopter") and a background category.\\
(4) {\textbf{PsortPos} and \textbf{PsortNeg}\footnote{\footnotesize{\texttt{https://bmi.inf.ethz.ch/supplements/protsubloc}}}} are bioinformatics MKL datasets used for protein subcellular localization research.\\
(5) {\textbf{BBC} and \textbf{BBCSport} \footnote{\footnotesize{\texttt{http://mlg.ucd.ie/datasets/bbc.html}}}} are two news corpora datasets derived from BBC News, consisting of various documents corresponding to stories or sports news in 5 areas.\\
(6) {\textbf{ProteinFold}\footnote{\footnotesize{\texttt{mkl.ucsd.edu/dataset/protein-fold-prediction}}}} is a bioinformatics dataset containing 694 protein patterns and 27 protein folds.\\
(7) {\textbf{Handwritten}\footnote{\footnotesize{\texttt{http://archive.ics.uci.edu/ml/datasets/}}}} and {\textbf{Mfeat}\footnote{\footnotesize{\texttt{https://datahub.io/machine-learning/mfeat-pixel}}}} are image datasets originated from UCI ML repository, including 2000 digits of handwritten numerals (``0"–``9").\\
(8) {\textbf{Scene-15} \footnote{\footnotesize{\texttt{https://www.kaggle.com/yiklunchow/scene15}}}}contains 4485 gray-scale images, 15 environmental categories, and 3 features (GIST, PHOG, and LBP).

All the pre-computed base kernels within the datasets are publicly available on websites, and are centered and then normalized following \cite{cortes2012algorithms, shawe2004kernel}. 
%----------------------------------------------------------
\subsection{Compared Algorithms}
Thirteen existing multiple kernel or graph-based algorithms are compared with our proposed model, including\\
(1) {\bf Avg-KKM} combines base kernels with uniform weights.\\
(2) {\bf MKKM} \cite{huang2012multiple} optimally combines multiple kernels by alternatively performing KKM and updating the kernel weights.\\
(3) {\bf LMKKM} \cite{gonen2014localized} can optimally fuse base kernels via an adaptive sample-weighted strategy.\\
(4) {\bf MKKM-MR} \cite{liu2016multiple} improve the diversity of kernels by introducing a matrix-induced regularization term.\\
(5) {\bf LKAM} \cite{li2016multiple} introduces localized kernel maximizing alignment by constraining $\tau$-nearest neighbors of each sample.\\
(6) {\bf ONKC} \cite{liu2017optimal} regards the optimal kernel as the neighborhood kernel of the combined kernel.\\
(7) {\bf SwMC} \cite{nie2017self} eliminates the undesired hyper-parameter via a self-weighted strategy. \\
(8) {\bf LF-MVC} \cite{wang2019multi} aims to achieve maximal alignment of consensus partition and base ones via a late fusion manner.\\
(9) {\bf SPMKC} \cite{ren2020simultaneous} simultaneously performs consensus kernel learning and graph learning.\\
(10) {\bf SMKKM} \cite{liu2020simplemkkm} proposes a novel min-max optimization based on kernel alignment criterion.\\
(11) {\bf CAGL} \cite{ren2020consensus} proposes a multi-kernel graph-based clustering model to directly learn a consensus affinity graph with rank constraint.\\
(12) {\bf OPLFMVC} \cite{liu2021one} can directly learn the cluster labels on the base partition level. \\
(13) {\bf LSMKKM} \cite{liu2021localized} is localized SMKKM in KNN method.
%----------------------------------------------------------
\subsection{Experimental Settings}
Regarding the benchmark datasets, it is commonly assumed that the true number of clusters ${k}$ is known. For the methods involving $k$-means, the centroid of clusters is repeatedly and randomly initialized 50 times to reduce its randomness and report the best results. Regarding all the compared algorithms, we directly download the public MATLAB code and carefully tune the hyper-parameters following the original suggestion. For our proposed LSWMKC, the balanced hyper-parameter $\alpha$ varies in $[2^{0},2^{1},\cdots,2^{10}]$ by grid search. The clustering performance is evaluated by four widely employed criteria, including clustering Accuracy (ACC), Normalized Mutual Information (NMI), Purity, and Adjusted Rand Index (ARI). The experimental results are obtained from a desktop with Intel Core i7 8700K CPU (3.7GHz), 64GB RAM, and MATLAB 2020b (64bit).
\begin{table*}[!t]
\caption{{ACC, NMI, Purity, and ARI comparisons of fourteen clustering algorithms on twelve benchmark datasets.}}\label{Comparasion_Kpi}
\begin{center}
\vspace{-5pt}
\tiny
{
  \centering
  \resizebox{\textwidth}{!}{
  \tiny
    \begin{tabular}{|c|c|c|c|c|c|c|c|c|c|c|c|c|c|c|}
    \toprule
    \multirow{2}{*}{Datasets}  &\multirow{2}{*}{Avg-KKM} &{MKKM} &{LMKKM} &{MKKM-MR} &{LKAM} &{ONKC} &{SwMC}  &{LF-MVC} &{SPMKC} &{SMKKM} &{CAGL} &{OPLFMVC} &{LSMKKM}  &\multirow{2}{*}{Proposed}\\
                      &   &\tiny{(2011)}  &\tiny{(2014)}  &\tiny{(2016)}  &\tiny{(2016)}  &\tiny {(2017)}  &\tiny{(2017)}  &\tiny{(2019)}  &\tiny{(2020)}  &\tiny{(2020)} &\tiny{(2020)}  &\tiny{(2021)} &\tiny{(2021)}  & \\
    
    \midrule
    &\multicolumn{14}{c|} {ACC $(\%)$} \\
    \hline
    YALE & 54.73 & 52.00 & 52.27 & 56.24 & 58.88 & 56.36 & 46.67 & 55.00 & {\color[HTML]{0000FF} \textbf{65.45}} & 56.03 & 53.33 & 55.76 & 59.24 & {\color[HTML]{FF0000} \textbf{66.67}} \\
        \hline
    MSRA & 83.33 & 81.29 & 81.93 & 88.07 & 89.14 & 85.36 & 23.33 & 87.76 & 79.05 & 86.50 & {\color[HTML]{FF0000} \textbf{99.05}} & 87.14 & {\color[HTML]{0000FF}\textbf{91.19}} &  {90.95} \\
        \hline
    Caltech101-7 & 59.17 & 52.15 & 53.89 & 68.44 & 70.39 & 69.42 & 54.65 & 71.39 & 62.59 & 68.15 & {\color[HTML]{FF0000} \textbf{78.91}} & 73.47 & 76.21 & {\color[HTML]{0000FF} \textbf{76.64}} \\
        \hline
    PsortPos & 56.94 & 60.70 & {\color[HTML]{0000FF} \textbf{61.84}} & 49.21 & 53.08 & 50.41 & 37.71 & 53.21 & 36.04 & 43.70 & 48.80 & 56.38 & 49.50 & {\color[HTML]{FF0000} \textbf{65.06}} \\
        \hline
    BBC & 63.17 & 63.03 & 63.90 & 63.17 & 73.85 & 63.35 & 36.03 & 76.42 & 88.79 & 64.20 & 76.10 & {\color[HTML]{0000FF} \textbf{90.26}} & 73.58 & {\color[HTML]{FF0000} \textbf{96.51}} \\
        \hline
    BBCSport & 66.25 & 66.24 & 66.58 & 66.17 & 76.58 & 66.43 & 36.03 & 76.46 & 40.81 & 66.76 & {\color[HTML]{0000FF} \textbf{89.15}} & 81.25 & 77.11 & {\color[HTML]{FF0000} \textbf{97.24}} \\
        \hline
    ProteinFold & 28.97 & 26.99 & 22.41 & 34.72 & {\color[HTML]{FF0000} \textbf{37.73}} & 36.27 & 14.99 & 33.00 & 21.61 & 34.68 & 32.28 & 35.88 & 35.91 & {\color[HTML]{0000FF} \textbf{36.60}} \\
        \hline
    PsortNeg & 41.01 & {\color[HTML]{0000FF} \textbf{51.88}} & - & 39.71 & 40.53 & 40.15 & 26.59 & 45.52 & 25.14 & 41.54 & 27.77 & 48.13 & 45.69 & {\color[HTML]{FF0000} \textbf{52.77}} \\
        \hline
    Caltech101-mit & 34.16 & 32.81 & 27.94 & 34.75 & 32.28 & 34.02 & 22.42 & 34.41 & 36.99 & 35.85 & {\color[HTML]{FF0000} \textbf{44.18}} & 24.84 & 36.96 & {\color[HTML]{0000FF} \textbf{39.35}} \\
        \hline
    Handwritten & 95.99 & 64.94 & 65.03 & 88.66 & 95.40 & 89.51 & 58.50 & 95.80 & 28.15 & 93.57 & 88.25 & 92.25 & {\color[HTML]{0000FF} \textbf{96.48}} & {\color[HTML]{FF0000} \textbf{97.45}} \\
        \hline
    Mfeat & 93.83 & 64.31 & - & 88.53 & 82.28 & 88.85 & 78.65 & 92.85 & 16.95 & 94.19 & 87.50 & 93.80 & {\color[HTML]{0000FF} \textbf{96.95}} & {\color[HTML]{FF0000} \textbf{97.50}} \\
        \hline
    Scene15 & 43.17 & 41.18 & 40.85 & 38.41 & 41.42 & 39.93 & 11.33 & {\color[HTML]{0000FF} \textbf{45.82}} & 11.82 & 43.60 & 22.30 & 43.26 & 43.80 & {\color[HTML]{FF0000} \textbf{48.58}}\\
    \midrule
    &\multicolumn{14}{c|}{NMI $(\%)$} \\
    \hline
    YALE & 57.32 & 54.35 & 54.56 & 58.63 & 60.23 & 59.54 & 48.86 & 57.54 & {\color[HTML]{0000FF} \textbf{64.11}} & 58.91 & 59.93 & 56.90 & 60.31 & {\color[HTML]{FF0000} \textbf{66.15}} \\
        \hline
    MSRA & 73.99 & 73.22 & 75.01 & 77.59 & 79.83 & 74.89 & 22.86 & 79.39 & 69.35 & 75.17 & {\color[HTML]{FF0000} \textbf{97.85}} & 78.96 & 82.63 & {\color[HTML]{0000FF}\textbf{85.15}}\\
        \hline
    Caltech101-7 & 59.07 & 51.60 & 52.13 & 64.12 & 65.35 & 63.52 & 58.20 & 70.08 & 56.06 & 63.73 & {\color[HTML]{FF0000} \textbf{83.85}} & 69.37 & {\color[HTML]{0000FF} \textbf{74.15}} & 71.72 \\
        \hline
    PsortPos & 28.73 & 35.50 & {\color[HTML]{0000FF} \textbf{37.16}} & 21.13 & 24.54 & 25.43 & 2.28 & 24.95 & 5.48 & 23.76 & 24.19 & 28.33 & 24.01 & {\color[HTML]{FF0000} \textbf{39.65}} \\
        \hline
    BBC & 43.50 & 43.58 & 44.01 & 43.46 & 65.42 & 43.53 & 2.00 & 58.86 & 74.60 & 44.45 & {\color[HTML]{0000FF} \textbf{80.81}} & 79.69 & 65.09 & {\color[HTML]{FF0000} \textbf{90.05}} \\
        \hline
    BBCSport & 54.18 & 54.09 & 54.37 & 53.85 & 54.50 & 53.51 & 3.71 & 57.59 & 6.75 & 49.34 & {\color[HTML]{0000FF} \textbf{79.82}} & 65.25 & 54.81 & {\color[HTML]{FF0000} \textbf{91.03}} \\
        \hline
    ProteinFold & 40.32 & 38.03 & 34.68 & 43.70 & {\color[HTML]{FF0000} \textbf{46.25}} & 44.38 & 7.91 & 41.72 & 33.03 & 44.44 & 41.56 & 41.90 & 45.15 & {\color[HTML]{0000FF} \textbf{46.03}} \\
        \hline
    PsortNeg & 17.39 & {\color[HTML]{FF0000} \textbf{32.16}} & - & 21.65 & 21.76 & 21.03 & 0.66 & 18.75 & 0.31 & 19.05 & 12.21 & 23.25 & 17.01 & {\color[HTML]{0000FF} \textbf{30.20}} \\
        \hline
    Caltech101-mit & 59.30 & 58.57 & 55.26 & 59.72 & 58.48 & 59.30 & 30.91 & 59.55 & 60.11 & 60.35 & {\color[HTML]{FF0000} \textbf{66.12}} & 52.86 & 61.37 & {\color[HTML]{0000FF}\textbf{62.91}} \\
        \hline
    Handwritten & 91.09 & 64.79 & 64.74 & 79.44 & 91.83 & 80.66 & 61.38 & 90.91 & 15.98 & 87.42 & 92.30 & 84.80 & {\color[HTML]{0000FF} \textbf{93.56}} & {\color[HTML]{FF0000} \textbf{94.17}} \\
        \hline
    Mfeat & 89.09 & 59.82 & - & 80.41 & 84.89 & 80 & 84.56 & 88.60 & 3.82 & 88.64 & 91.34 & 87.09 & {\color[HTML]{0000FF} \textbf{93.18}} & {\color[HTML]{FF0000} \textbf{94.31}} \\
        \hline
    Scene15 & 41.31 & 38.62 & 38.79 & 37.25 & 42.14 & 37.73 & 2.61 & {\color[HTML]{0000FF} \textbf{42.71}} & 2.89 & 40.60 & 29.36 & 41.88 & 40.97 & {\color[HTML]{FF0000} \textbf{46.70}}\\
    \midrule
    &\multicolumn{14}{c|}{Purity $(\%)$} \\
    \hline
    YALE & 55.42 & 52.94 & 53.06 & 56.58 & 59.42 & 57.18 & 50.91 & 56.03 & {\color[HTML]{0000FF} \textbf{66.06}} & 56.42 & 55.15 & 56.97 & 59.88 & {\color[HTML]{FF0000} \textbf{67.27}} \\
        \hline
    MSRA & 83.33 & 81.45 & 81.93 & 88.07 & 89.14 & 85.36 & 30.48 & 87.76 & 79.05 & 86.50 & {\color[HTML]{FF0000} \textbf{99.05}} & 87.14 & {\color[HTML]{0000FF}\textbf{91.19}} & {90.95} \\
        \hline
    Caltech101-7 & 68.05 & 63.84 & 66.39 & 72.93 & 76.55 & 73.97 & 64.63 & 79.59 & 68.93 & 72.34 & {\color[HTML]{FF0000} \textbf{83.22}} & 80.27 & {\color[HTML]{0000FF} \textbf{81.42}} & 81.41 \\
        \hline
    PsortPos & 60.74 & 66.66 & {\color[HTML]{0000FF} \textbf{68.03}} & 56.14 & 61.03 & 60.79 & 37.71 & 57.07 & 46.03 & 57.63 & 48.80 & 60.63 & 53.72 & {\color[HTML]{FF0000} \textbf{68.76}} \\
        \hline
    BBC & 68.06 & 68.15 & 68.40 & 68.03 & 79.39 & 68.10 & 36.76 & 76.75 & 88.79 & 68.68 & 76.29 & {\color[HTML]{0000FF} \textbf{90.26}} & 79.17 & {\color[HTML]{FF0000} \textbf{96.51}} \\
        \hline
    BBCSport & 77.33 & 77.27 & 77.50 & 77.10 & 76.58 & 76.99 & 37.87 & 78.30 & 40.81 & 73.52 & {\color[HTML]{0000FF} \textbf{89.15}} & 81.25 & 77.11 & {\color[HTML]{FF0000} \textbf{97.24}} \\
        \hline
    ProteinFold & 37.39 & 33.70 & 31.16 & 41.89 & {\color[HTML]{FF0000} \textbf{43.70}} & 42.67 & 18.30 & 39.33 & 28.24 & 41.79 & 35.88 & 38.33 & 42.52 & {\color[HTML]{0000FF} \textbf{42.80}} \\
        \hline
    PsortNeg & 43.33 & {\color[HTML]{0000FF} \textbf{56.61}} & - & 44.66 & 45.29 & 44.67 & 27.22 & 48.22 & 27.08 & 42.17 & 30.96 & 51.80 & 47.17 & {\color[HTML]{FF0000} \textbf{57.06}} \\
        \hline
    Caltech101-mit & 36.22 & 34.88 & 29.56 & 36.77 & 34.30 & 36.16 & 26.08 & 36.65 & 39.22 & 37.96 & {\color[HTML]{FF0000} \textbf{46.80}} & 25.75 & 39.25 & {\color[HTML]{0000FF}\textbf{41.31}} \\
        \hline
    Handwritten & 95.99 & 65.84 & 65.52 & 88.66 & 95.44 & 89.51 & 58.70 & 95.80 & 30.50 & 93.57 & 88.25 & 92.25 & {\color[HTML]{0000FF} \textbf{96.52}} & {\color[HTML]{FF0000} \textbf{97.45}} \\
        \hline
    Mfeat & 94.13 & 64.95 & - & 88.53 & 86.02 & 88.85 & 78.80 & 93.27 & 17.60 & 94.19 & 87.85 & 93.80 & {\color[HTML]{0000FF} \textbf{96.95}} & {\color[HTML]{FF0000} \textbf{97.70}} \\
        \hline
    Scene15 & 47.85 & 44.29 & 44.30 & 42.40 & 46.01 & 43.60 & 11.62 & {\color[HTML]{0000FF} \textbf{49.36}} & 13.00 & 48.38 & 22.52 & 47.65 & 48.62 & {\color[HTML]{FF0000} \textbf{50.81}}\\
    \midrule
    &\multicolumn{14}{c|}{ARI $(\%)$} \\
    \hline
    YALE & 33.93 & 30.42 & 30.50 & 35.49 & 37.31 & 36.56 & 13.17 & 34.29 & {\color[HTML]{0000FF} \textbf{43.70}} & 35.86 & 32.56 & 34.21 & 37.89 & {\color[HTML]{FF0000}\textbf{45.06}} \\
        \hline
    MSRA & 68.14 & 66.22 & 68.00 & 74.46 & 76.66 & 69.76 & 6.90 & 74.52 & 59.60 & 71.17 & {\color[HTML]{FF0000} \textbf{97.77}} & 74.11 & 80.63 & {\color[HTML]{0000FF}\textbf{81.38}} \\
        \hline
    Caltech101-7 & 46.02 & 38.30 & 41.23 & 55.62 & 59.44 & 56.75 & 40.59 & 65.19 & 45.01 & 55.64 & {\color[HTML]{FF0000} \textbf{74.40}} & 65.14 & 68.81 & {\color[HTML]{0000FF} \textbf{74.34}} \\
        \hline
    PsortPos & 24.36 & {\color[HTML]{0000FF}\textbf{32.19}} & {\color[HTML]{FF0000} \textbf{33.98}} & 18.93 & 26.68 & 21.44 & 0.80 & 19.60 & 4.42 & 19.50 & 11.24 & 23.94 & 18.45 & {31.80} \\
        \hline
    BBC & 39.28 & 39.24 & 40.33 & 39.27 & 62.27 & 39.45 & -0.03 & 56.97 & 74.28 & 40.80 & 61.50 & {\color[HTML]{0000FF} \textbf{82.40}} & 61.79 & {\color[HTML]{FF0000} \textbf{89.66}} \\
        \hline
    BBCSport & 48.10 & 47.97 & 48.11 & 47.77 & 54.46 & 47.12 & 0.34 & 54.76 & 3.47 & 42.64 & {\color[HTML]{0000FF} \textbf{75.59}} & 63.69 & 48.10 & {\color[HTML]{FF0000} \textbf{92.01}} \\
        \hline
    ProteinFold & 14.36 & 12.11 & 7.76 & 17.15 & {\color[HTML]{0000FF} \textbf{20.08}} & 18.01 & -0.04 & 16.08 & 7.65 & 17.61 & 7.44 & 19.71 & 19.83 & {\color[HTML]{FF0000} \textbf{20.36}} \\
        \hline
    PsortNeg & 13.14 & {\color[HTML]{0000FF} \textbf{26.75}} & - & 16.85 & 16.04 & 16.93 & -0.17 & 16.09 & -0.08 & 13.13 & 1.88 & 19.76 & 13.84 & {\color[HTML]{FF0000}\textbf{27.44}} \\
        \hline
    Caltech101-mit & 18.42 & 17.34 & 13.37 & 18.78 & 16.82 & 18.32 & 0.90 & 18.79 & 18.54 & 19.83 & 14.82 & 12.30 & {\color[HTML]{0000FF} \textbf{21.04}} & {\color[HTML]{FF0000} \textbf{23.75}} \\
        \hline
    Handwritten & 91.33 & 51.76 & 50.38 & 77.16 & 91.65 & 78.70 & 37.97 & 90.98 & 8.30 & 86.45 & 85.72 & 83.82 & {\color[HTML]{0000FF} \textbf{93.49}} & {\color[HTML]{FF0000} \textbf{94.45}} \\
        \hline
    Mfeat & 88.36 & 46.88 & - & 77.36 & 79.25 & 77.32 & 77.73 & 87.09 & 1.37 & 87.68 & 88.11 & 86.80 & {\color[HTML]{0000FF} \textbf{93.32}} & {\color[HTML]{FF0000} \textbf{94.54}} \\
        \hline
    Scene15 & 26.03 & 22.62 & 22.87 & 22.70 & 24.84 & 23.46 & 0.20 & 27.31 & 0.70 & 25.37 & 5.84 & {\color[HTML]{0000FF} \textbf{27.37}} & 25.77 & {\color[HTML]{FF0000} \textbf{29.99}}\\
    \bottomrule
    \end{tabular}}
    }
\end{center}
\vspace{-5pt}
\end{table*}
%----------------------------------------------------------------------------
\begin{figure*}[!t]
\vspace{-5pt}
\centering
\includegraphics[width= 0.99\textwidth]{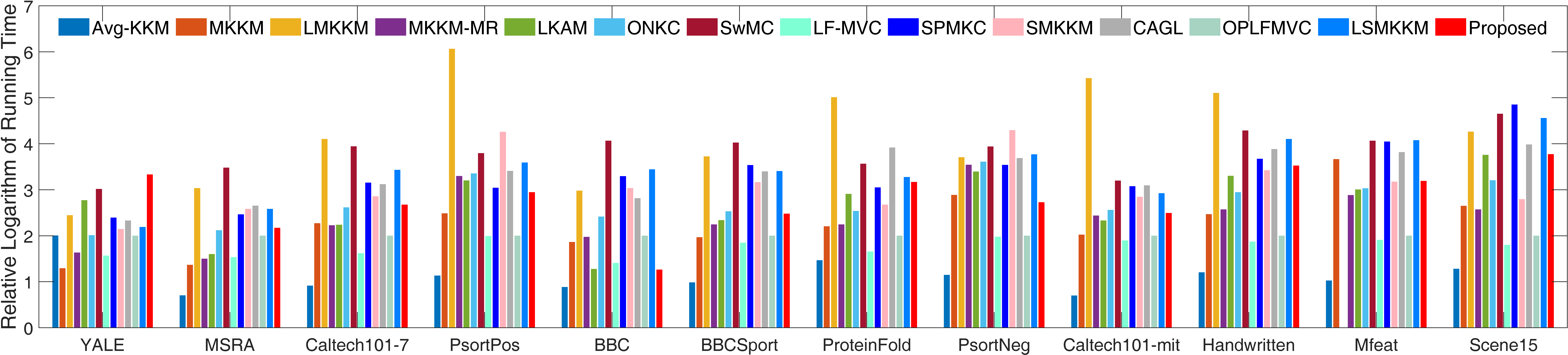}
\caption{Relative logarithm time-consuming comparison of fourteen models on twelve datasets.} 
\label{Compare_time}
\vspace{-10pt}
\end{figure*}

%----------------------------------------------------------------------------
\begin{table*}[!t]
\caption{{ACC, NMI, Purity, and ARI comparisons of our proposed algorithm and KNN mechanism on twelve benchmark datasets. }}\label{comparison_KNN}
\begin{center}
\vspace{-5pt}
\scriptsize
    \begin{tabular}{|c|c|c|c|c|c|c|c|c|c|c|c|c|}
    \toprule
    Datasets & YALE & MSRA & Caltech101-7 & PsortPos & BBC & BBCSport & ProteinFold & PsortNeg & Caltech101-mit & Handwritten & Mfeat & Scene15 \\
    \midrule
    &\multicolumn{12}{c|}{ACC $(\%)$} \\
    \hline
    KNN & 63.03 & 90.48 & 74.15 & 64.14 & 71.69 & 72.06 & 36.31 & 51.73 & 37.32 & 96.75 & 96.75 & 46.82 \\
    \hline
    Proposed & {\color[HTML]{FF0000} \textbf{66.67}} & {\color[HTML]{FF0000} \textbf{90.95}} & {\color[HTML]{FF0000} \textbf{76.64}} & {\color[HTML]{FF0000} \textbf{65.06}} & {\color[HTML]{FF0000} \textbf{96.51}} & {\color[HTML]{FF0000} \textbf{97.24}} & {\color[HTML]{FF0000} \textbf{36.60}} & {\color[HTML]{FF0000} \textbf{52.77}} & {\color[HTML]{FF0000} \textbf{39.35}} & {\color[HTML]{FF0000} \textbf{97.45}} & {\color[HTML]{FF0000} \textbf{97.50}} & {\color[HTML]{FF0000} \textbf{48.58}}
\\
    \midrule
    &\multicolumn{12}{c|}{NMI $(\%)$} \\
    \hline
    KNN & 62.00 & 83.90 & 68.78 & 35.48 & 55.66 & 48.53 & 44.22 & 28.08 & 61.74 & 92.87 & 92.88 & 42.33 \\
    \hline
    Proposed & {\color[HTML]{FF0000} \textbf{66.15}} & {\color[HTML]{FF0000} \textbf{85.15}} & {\color[HTML]{FF0000} \textbf{72.12}} & {\color[HTML]{FF0000} \textbf{39.65}} & {\color[HTML]{FF0000} \textbf{90.05}} & {\color[HTML]{FF0000} \textbf{91.03}} & {\color[HTML]{FF0000} \textbf{46.03}} & {\color[HTML]{FF0000} \textbf{30.20}} & {\color[HTML]{FF0000} \textbf{62.91}} & {\color[HTML]{FF0000} \textbf{94.17}} & {\color[HTML]{FF0000} \textbf{94.31}} & {\color[HTML]{FF0000} \textbf{46.70}}\\
    \midrule
    &\multicolumn{12}{c|}{Purity $(\%)$} \\
    \hline
    KNN & 63.64 & 90.48 & 78.91 & 68.39 & 73.16 & 73.16 & 42.36 & 53.88 & 39.22 & 96.75 & 96.75 & 49.63 \\
    \hline
    Proposed & {\color[HTML]{FF0000} \textbf{67.27}} & {\color[HTML]{FF0000} \textbf{90.95}} & {\color[HTML]{FF0000} \textbf{81.41}} & {\color[HTML]{FF0000} \textbf{68.76}} & {\color[HTML]{FF0000} \textbf{96.51}} & {\color[HTML]{FF0000} \textbf{97.24}} & {\color[HTML]{FF0000} \textbf{42.80}} & {\color[HTML]{FF0000} \textbf{57.06}} & {\color[HTML]{FF0000} \textbf{41.31}} & {\color[HTML]{FF0000} \textbf{97.45}} & {\color[HTML]{FF0000} \textbf{97.50}} & {\color[HTML]{FF0000} \textbf{50.81}}\\
    \midrule
    &\multicolumn{ 12}{c|}{ARI $(\%)$} \\
    \hline
    KNN & 40.19 & 79.95 & 67.50 & {\color[HTML]{FF0000} \textbf{34.73}} & 45.11 & 42.93 & 19.44 & 24.02 & 21.35 & 92.95 & 92.94 & 28.31 \\
    \hline
    Proposed & {\color[HTML]{FF0000} \textbf{45.06}} & {\color[HTML]{FF0000} \textbf{81.38}} & {\color[HTML]{FF0000} \textbf{74.34}} & 31.80 & {\color[HTML]{FF0000} \textbf{86.66}} & {\color[HTML]{FF0000} \textbf{92.01}} & {\color[HTML]{FF0000} \textbf{20.36}} & {\color[HTML]{FF0000} \textbf{27.44}} & {\color[HTML]{FF0000} \textbf{23.75}} & {\color[HTML]{FF0000} \textbf{94.45}} & {\color[HTML]{FF0000} \textbf{94.54}} & {\color[HTML]{FF0000} \textbf{29.99}}\\
    \bottomrule
    \end{tabular}
    \end{center}
\vspace{-10pt}
\end{table*}
%----------------------------------------------------------------------------
\begin{figure*}[!t]
\vspace{-10pt}
\begin{center}{
		\centering
			\subfloat[KNN (neighbor index)]{{\includegraphics[width=0.22\textwidth]{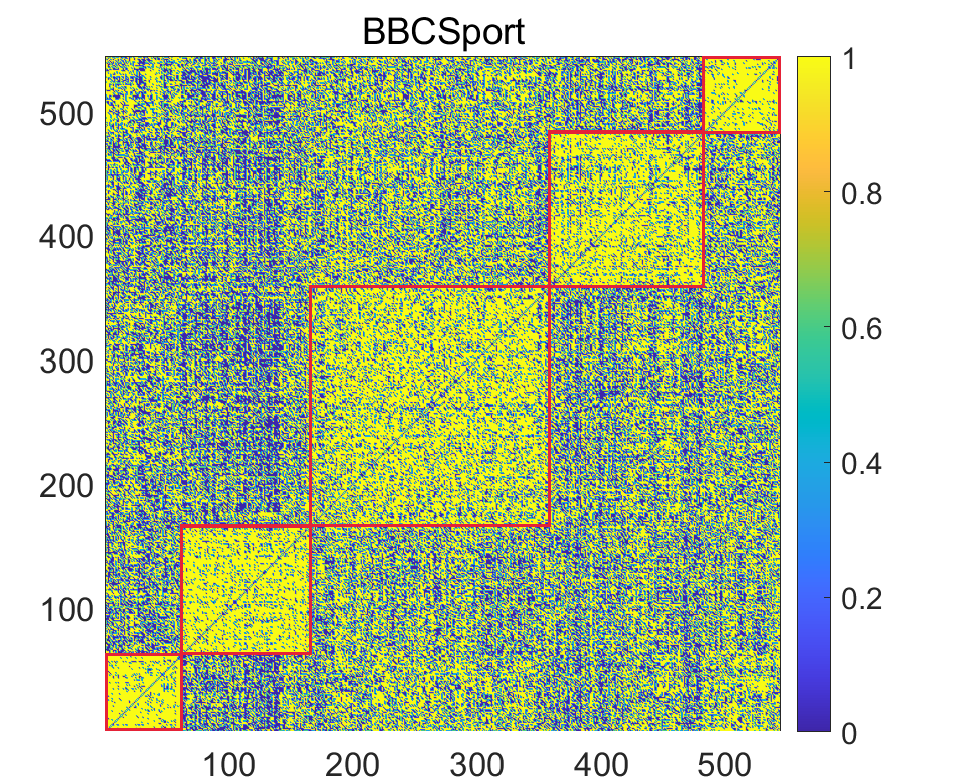}} \label{d-1}}\hspace{3mm}
			\subfloat[KNN (${\mathbf{K}}_{(l)}$)]{{\includegraphics[width=0.22\textwidth]{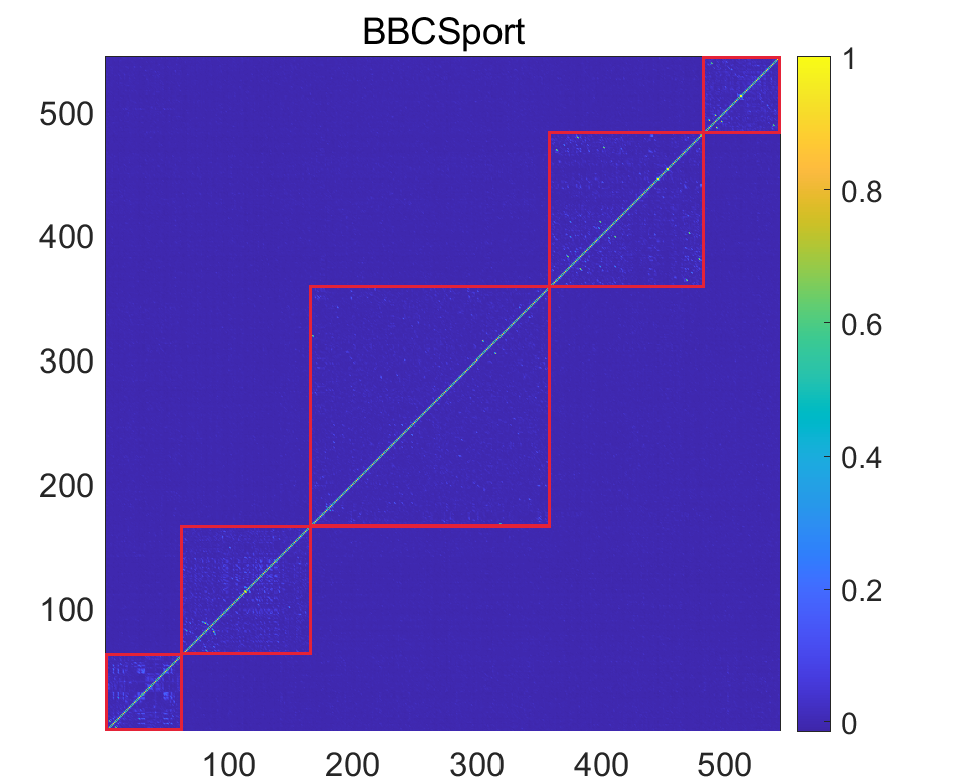}} \label{d-2}}\hspace{3mm}
            \subfloat[Proposed ($\mathbf{Z}$)]{{\includegraphics[width=0.22\textwidth]{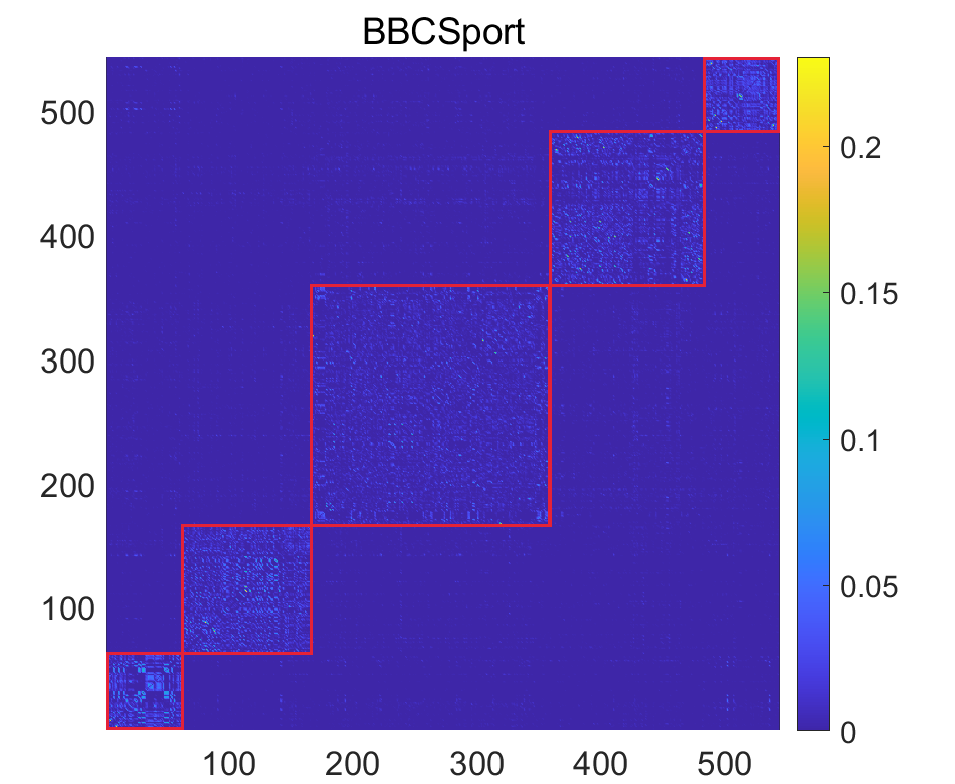}} \label{d-3}}\hspace{3mm}
            \subfloat[Proposed (${\mathbf{K}}^{\ast}$)]{{\includegraphics[width=0.22\textwidth]{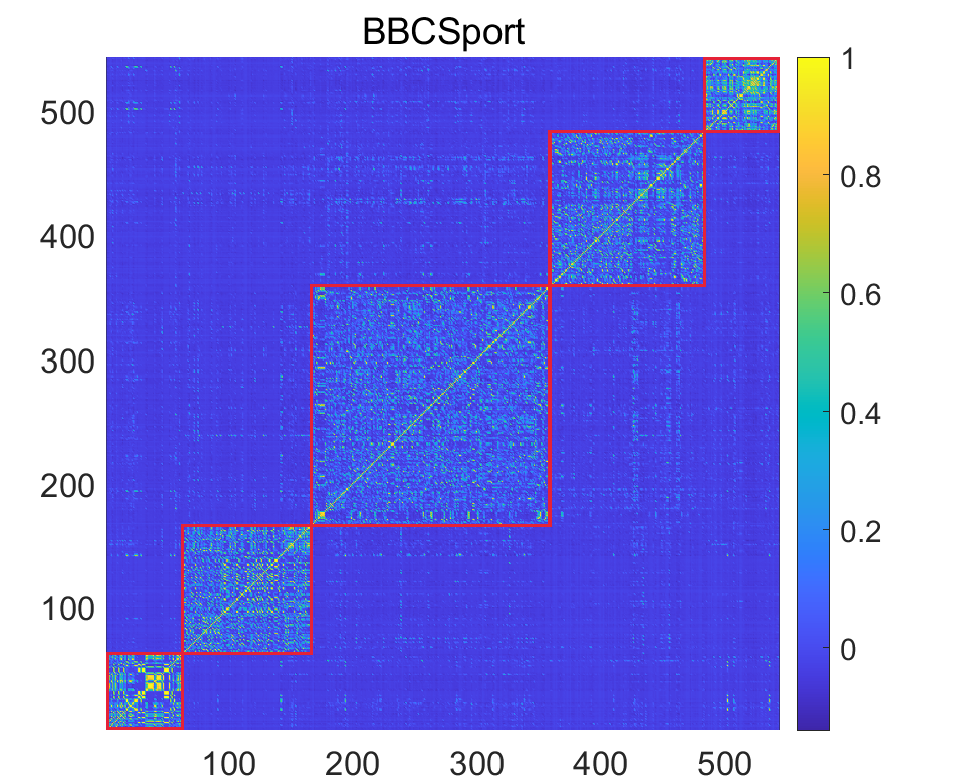}} \label{d-4}}\\
            \vspace{-8pt}
			\subfloat[KNN (neighbor index)]{{\includegraphics[width=0.22\textwidth]{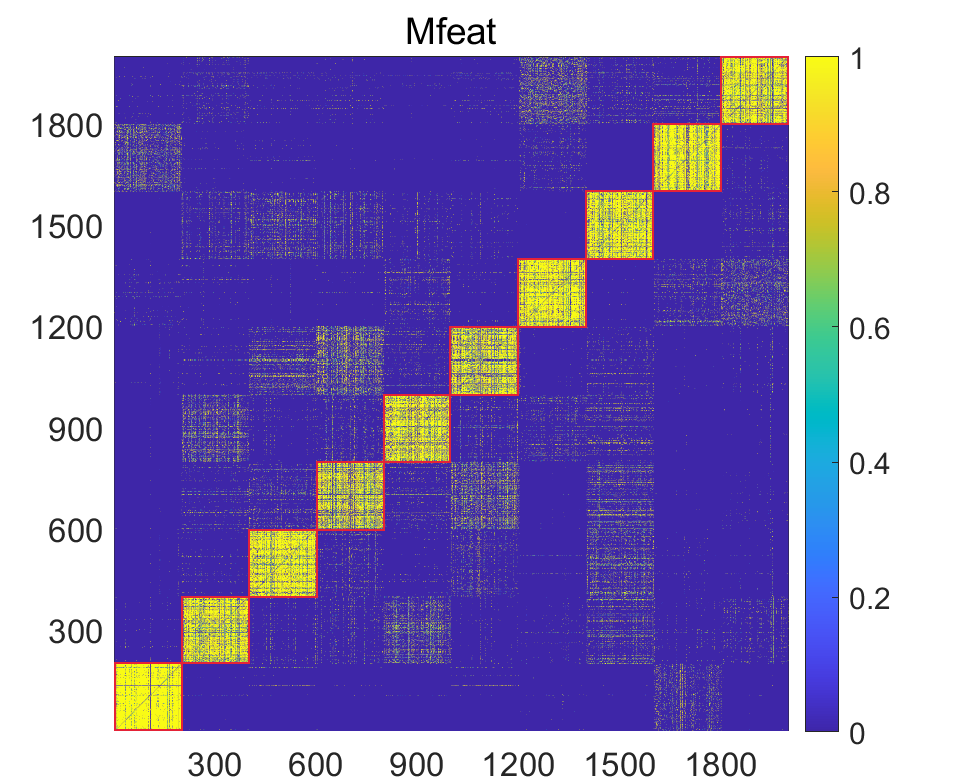}} \label{i-1}}\hspace{3mm}
			\subfloat[KNN (${\mathbf{K}}_{(l)}$)]{{\includegraphics[width=0.22\textwidth]{mfeat_KNN_K.png}} \label{i-2}}\hspace{3mm}
            \subfloat[Proposed ($\mathbf{Z}$)]{{\includegraphics[width=0.22\textwidth]{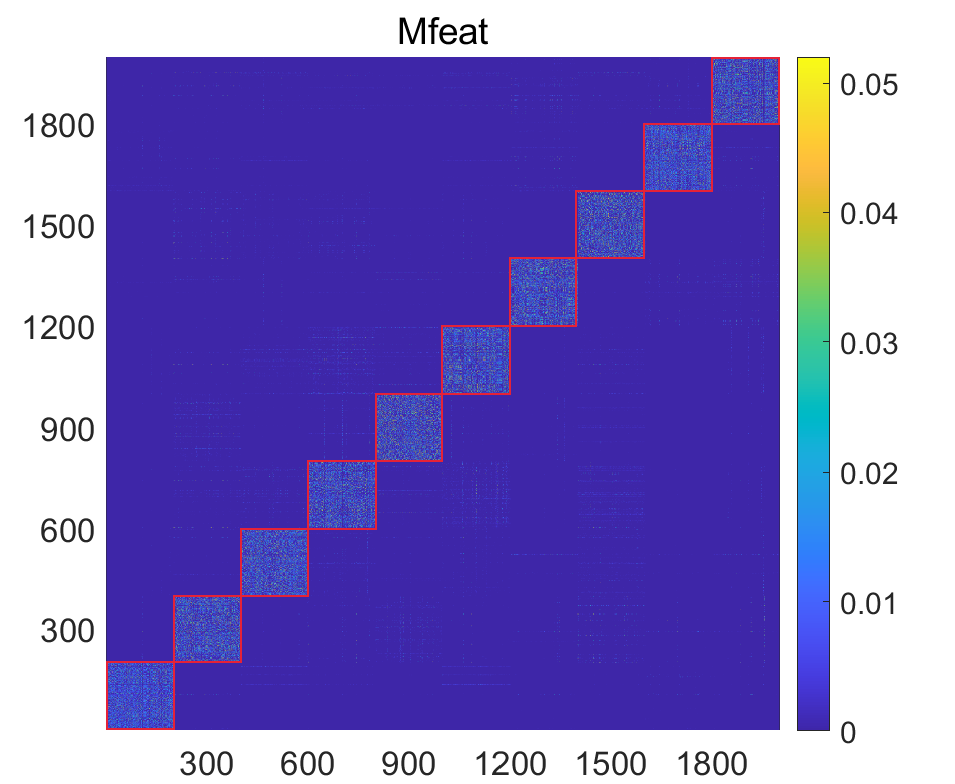}} \label{i-5}}\hspace{3mm}
            \subfloat[Proposed (${\mathbf{K}}^{\ast}$)]{{\includegraphics[width=0.22\textwidth]{mfeat_Proposed_Z_iter=22.png}} \label{i-6}}\\
			\caption{{The visualization of neighbor index and localized ${\mathbf{K}}_{(l)}$ in KNN mechanism, the affinity graph $\mathbf{Z}$ and localized ${\mathbf{K}}^{\ast}$ of the proposed algorithm on BBCSport and Mfeat datasets.}}
			\label{compare-KNN}
			}
\end{center}
\vspace{-10pt}
\end{figure*}
%----------------------------------------------------------------------------
\begin{figure*}[!t]
\vspace{-10pt}
\begin{center}{
		\centering
            \subfloat[YALE]{{\includegraphics[width=0.155\textwidth]{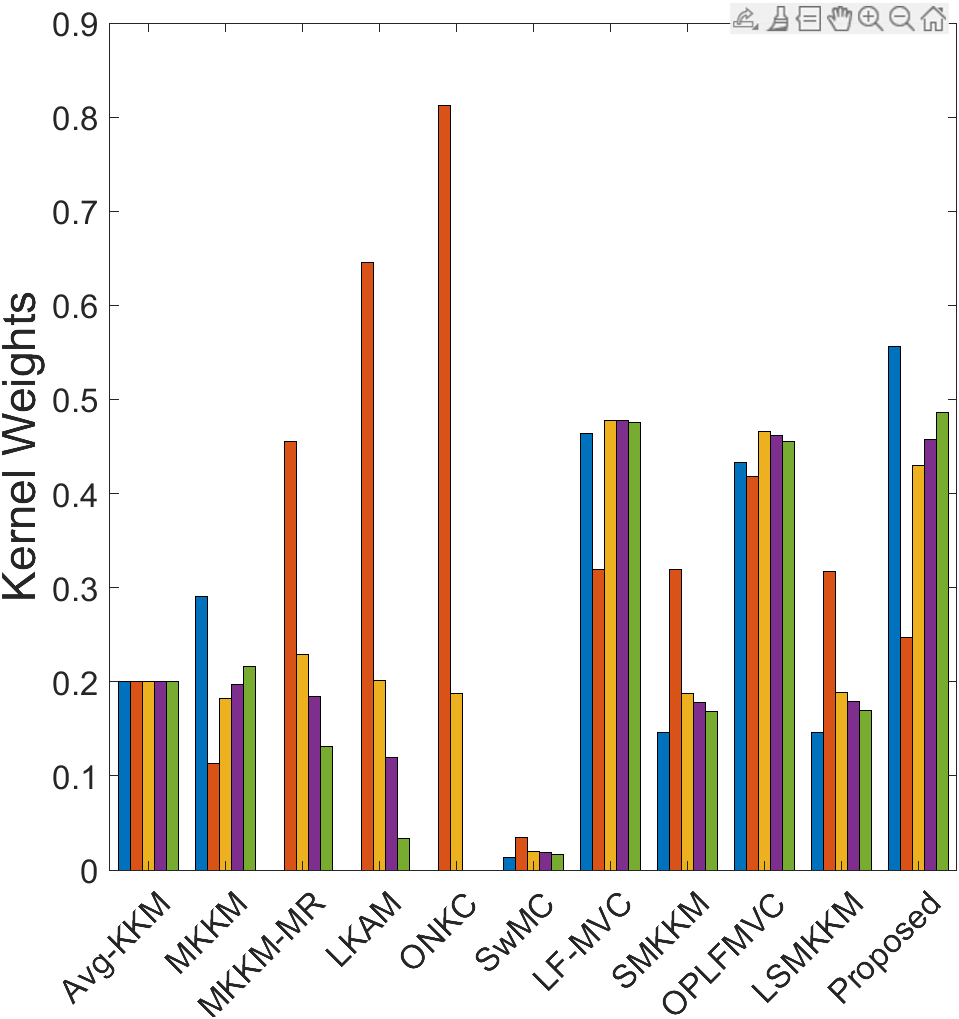}}}
            \subfloat[BBC]{{\includegraphics[width=0.155\textwidth]{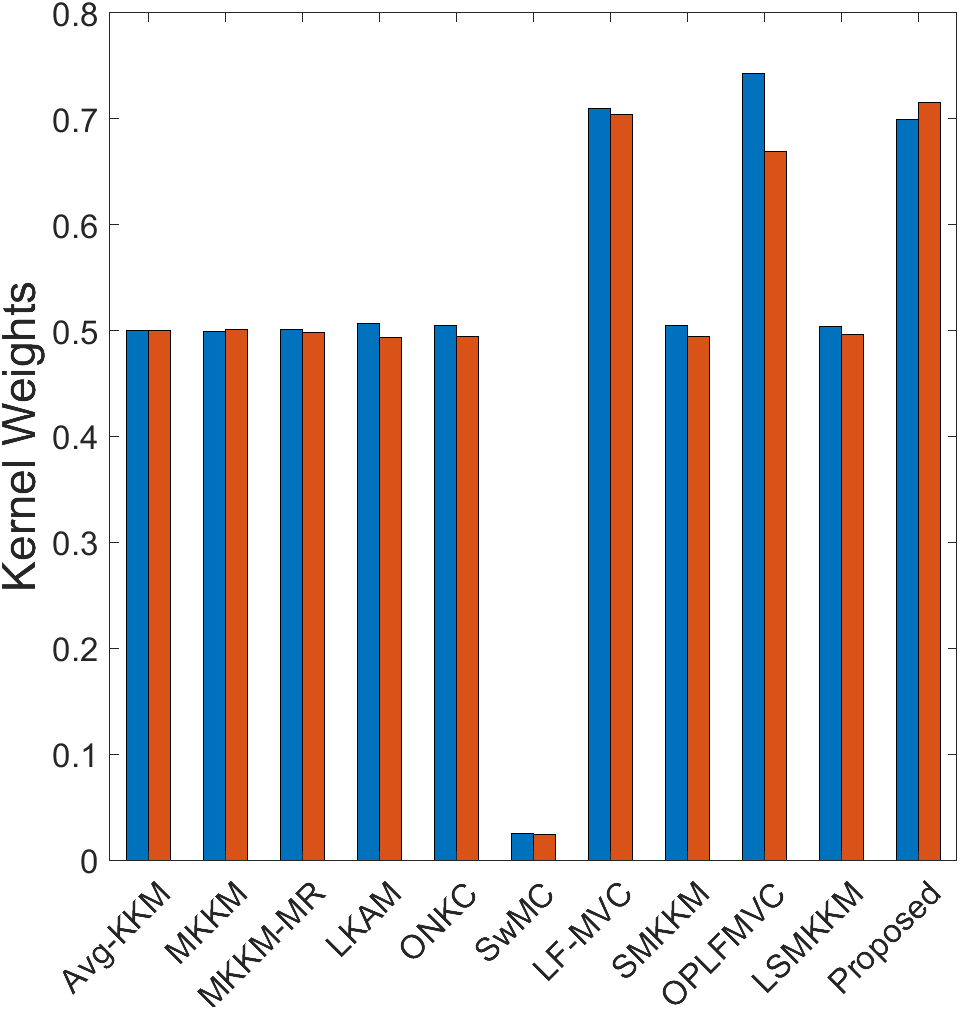}}}
            \subfloat[BBCSport]{{\includegraphics[width=0.155\textwidth]{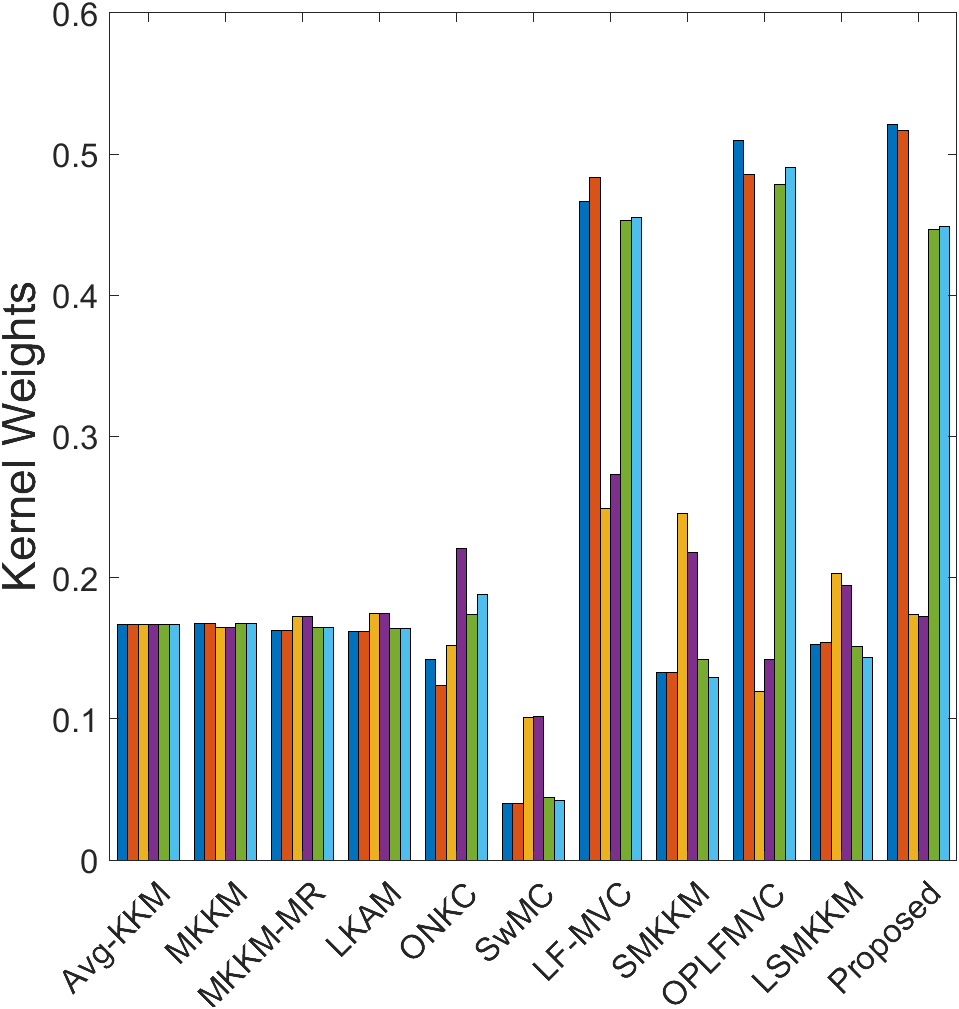}}} 
            \vspace{-5pt}
            \subfloat[Handwritten]{{\includegraphics[width=0.155\textwidth]{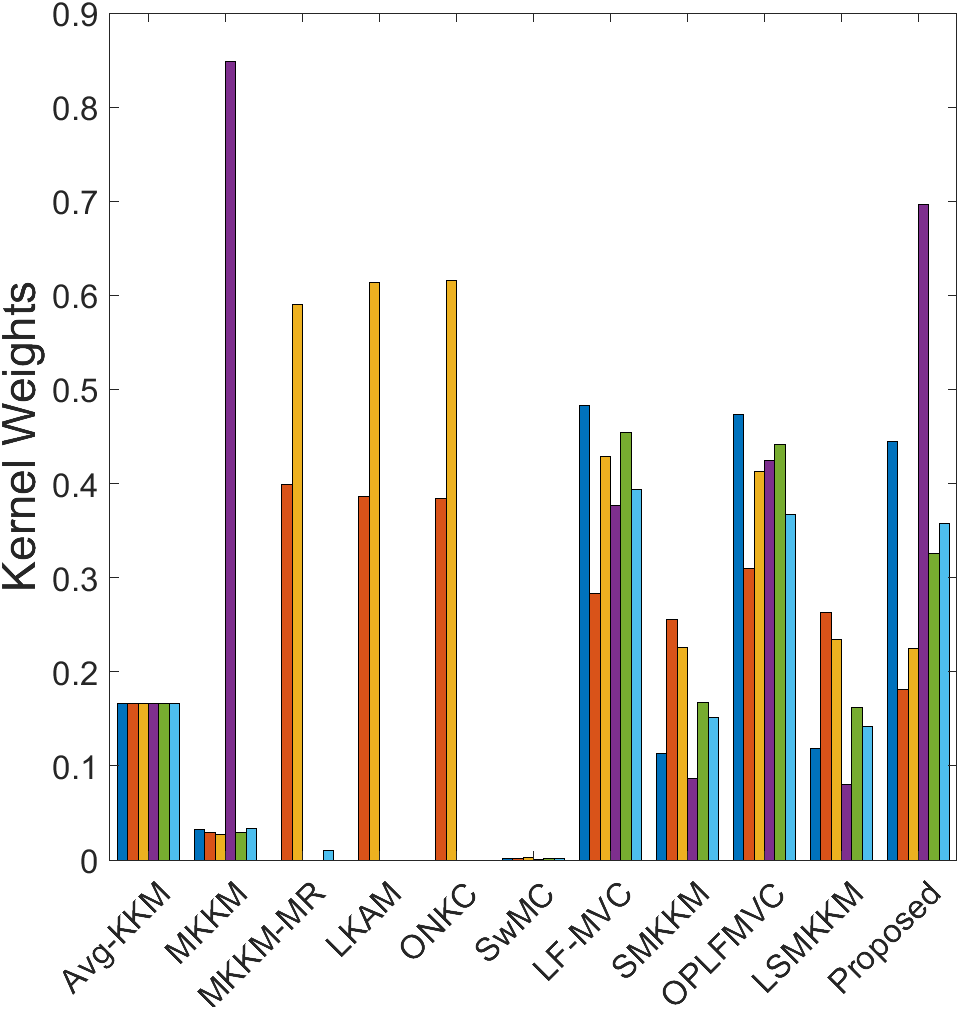}}}
            \subfloat[Mfeat]{{\includegraphics[width=0.155\textwidth]{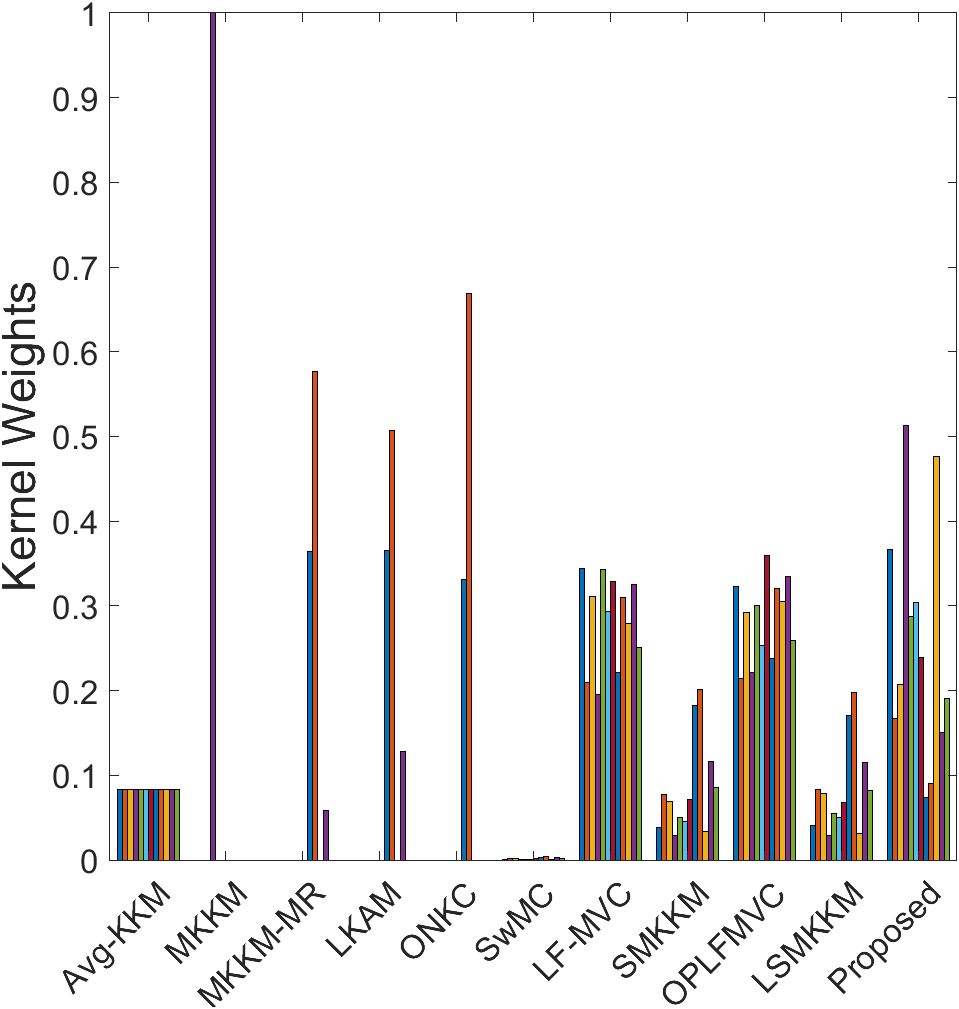}}}
            \subfloat[Scene15]{{\includegraphics[width=0.155\textwidth]{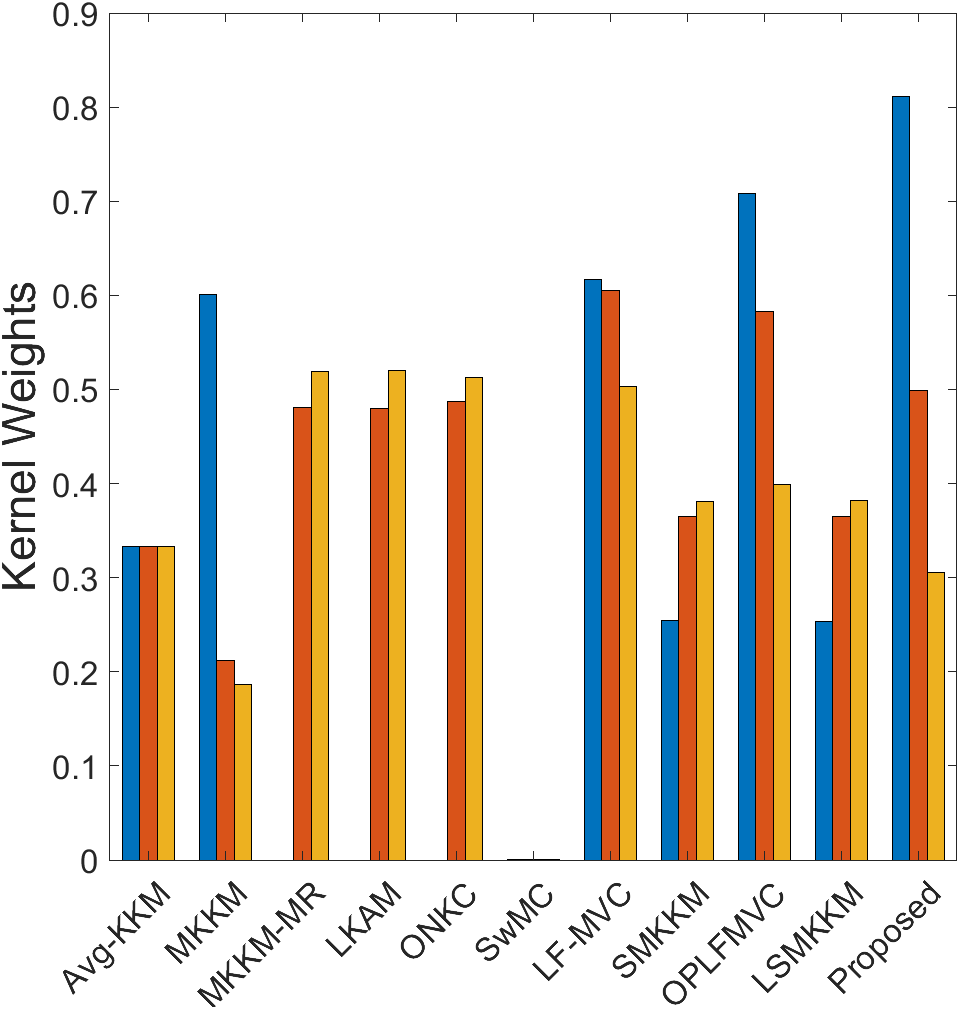}}}
            \vspace{5pt}
			\caption{{Comparison of the learned kernel weights of different algorithms on six datasets. Other datasets' results are provided in the appendix.}}
			\label{Weights}
			}
\end{center}
\vspace{-10pt}
\end{figure*}
%----------------------------------------------------------
\subsection{Experimental Results}
Table \ref{Comparasion_Kpi} reports ACC, NMI, Purity, and ARI comparisons of fourteen algorithms on twelve datasets. Red bold denotes the optimal results. Blue bold denotes the sub-optimal results while '-' denotes unavailable results due to overmuch execution time. According to the experimental results, it can be seen that:
\begin{itemize}
\item Our proposed LSWMKC algorithm achieves optimal or sub-optimal performance on most datasets. Particularly, CAGL can be regarded as the strongest competitor in affinity graph multi-kernel clustering, our LSWMKC still exceeds CAGL with a large margins improvement of 13.34\%, 16.26\%, 20.41\%, 8.09\%, 25.00\%, 9.20\%, 10.00\%, and 26.28\% on YALE, PsortPos, BBC, BBCSport, PsortNeg, Handwritten, Mfeat, and Scene15 datasets respectively in terms of ACC, which well demonstrates the superiority of our model over existing methods.
\item Compared with LKAM and LSMKKM that utilize KNN mechanism to localize base kernel, our LSWMKC still exhibits promising performance. Especially, LSMKKM can be regarded as the most competitive method in multi-kernel clustering, the ACC of our LSWMKC exceeds that of them 7.42\%, 0.43\%, 11.99\%, 22.66\%, 20.13\%, 7.08\%, 2.39\%, 0.97\%, 0.55\%, and 4.78\% on ten datasets respectively, which sufficiently illustrates the reasonableness of our model. Similarly, NMI, Purity, and ARI of our algorithm also outperforms other methods on most datasets. 
\end{itemize}

In summary, the quantitative comparison results can adequately substantiate the promising capability of our LSWMKC algorithm. The superiority of our algorithm can be attributed to the following two aspects:
\romannumeral1) Our MKC model firstly learns a discriminative graph to explore the intrinsic local manifold structures in kernel space, which can reveal the ranking relationship of samples. The noise or outliers are sufficiently removed, which directly serves for clustering.       
\romannumeral2) An optimal neighborhood kernel is obtained with naturally sparse property and clear block diagonal structures, which can further denoise the affinity graph. Our model achieves implicitly optimizing adaptive weights on different neighbors with corresponding samples in kernel space. Compared with the existing KNN mechanism, the unreliable distant-distance neighbors in our model can be removed or assigned small weights. The obtained localized kernel is more reasonable in comparison to the one from KNN mechanism. Such two aspects conduce to obvious improvement in applications.
%----------------------------------------------------------
\subsection{Running Time Comparison}
Figure \ref{Compare_time} plots the time-consuming comparison of fourteen algorithms. For simplify, the elapsed time of OPLFMVC is set as the baseline and we take the logarithm of all results. As our analysis that our LSWMKC shares the same computational complexity with MKKM, LMKKM, LKAM, ONKC, SMKKM, SPMKC, CAGL, and LSMKKM, the empirical time evaluation also demonstrates that our LSWMKC costs comparative and even shorter running time. More importantly, our LSWMKC exhibits promising performance.  
%----------------------------------------------------------
\subsection{Comparing with KNN Mechanism}
Recall our motivation to learn localized kernel by considering the ranking importance of neighbors in contrast to the traditional KNN mechanism. Here, we conduct comparison experiments with KNN mechanism (labelled as KNN). Specifically, we tune the neighbor ratio $\tau$ varying in $[0.1,0.2,\cdots,0.9]$ by grid search in average kernel space and report the best results. As Table \ref{comparison_KNN} shows, our algorithm consistently outperforms KNN mechanism. Moreover, as Figure \ref{compare-KNN} shows, for KNN mechanism, we plot the visualization of neighbor index and $\mathbf{K}_{(l)}$, for our model, we visualize the learned affinity graph $\mathbf{Z}$ and neighborhood kernel $\mathbf{K}^{\ast}$ on BBCSport and Mfeat datasets. Regarding KNN mechanism, the neighbor index involves noticeable noise, especially on BBCSport dataset, caused by the unreasonable neighbor building strategy. Such coarse localized manner directly incurs the corrupted $\mathbf{K}_{(l)}$ with much noise. In contrast, the affinity graphs learned by our neighbor learning mechanism achieve more precise block structures, which directly serve for learning localized $\mathbf{K}^{\ast}$. All the above results sufficiently illustrate the effectiveness of our neighbor building strategy.
%----------------------------------------------------------------------------
\begin{figure*}[!t]
\vspace{-10pt}
\begin{center}{
		\centering
			\subfloat[Initialized]{{\includegraphics[width=0.16\textwidth]{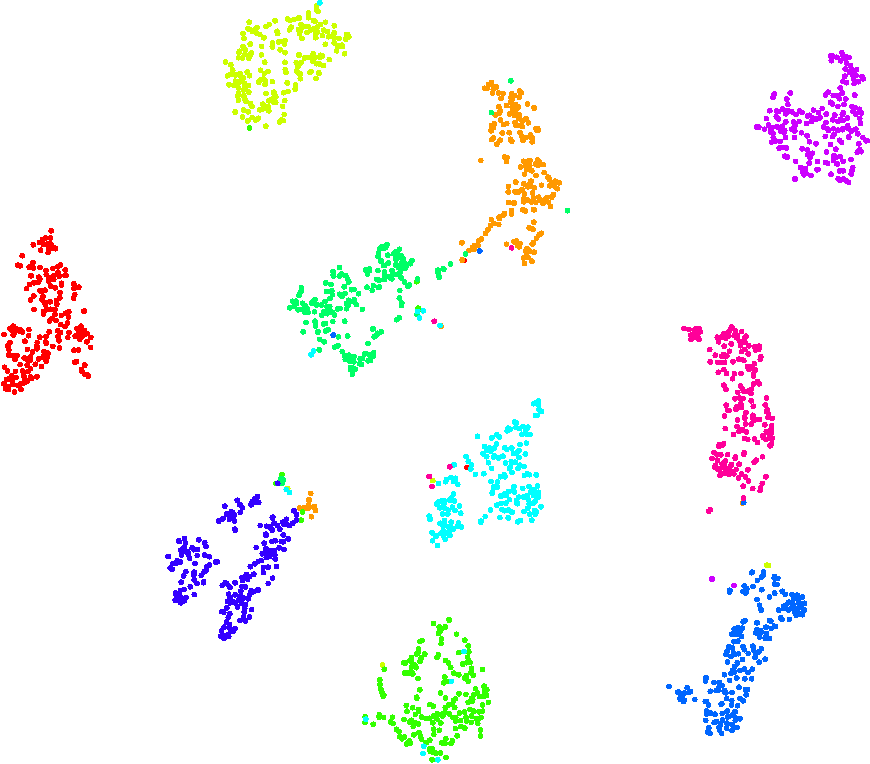}} \label{0-st iteration}}\hspace{5mm}
			\subfloat[1-st iteration]{{\includegraphics[width=0.16\textwidth]{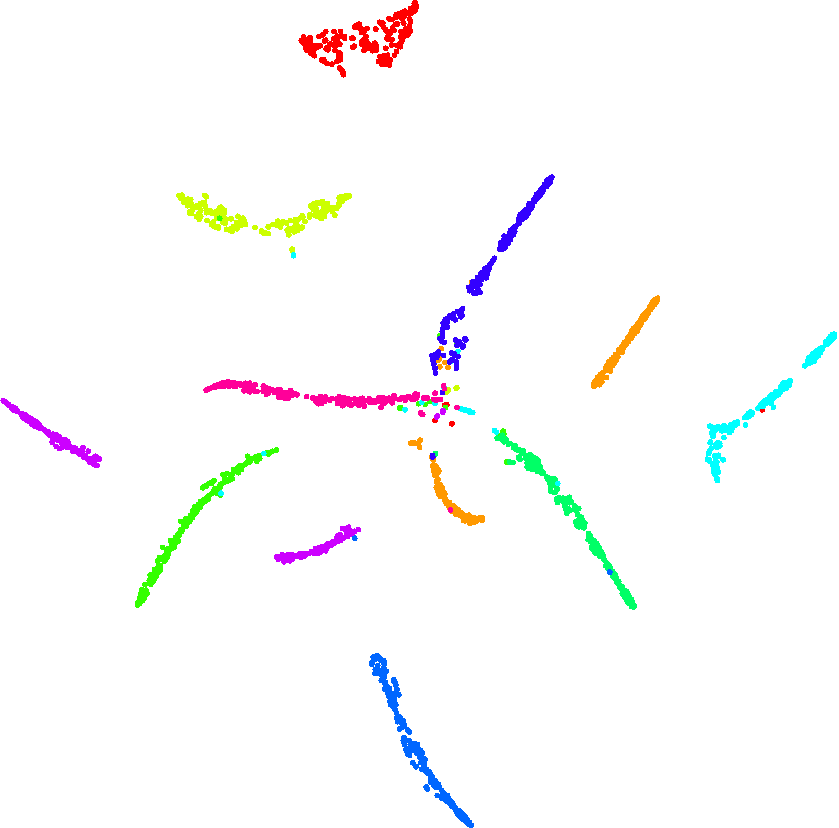}} \label{1-rd iteration}}\hspace{5mm}
 			\subfloat[5-th iteration]{{\includegraphics[width=0.16\textwidth]{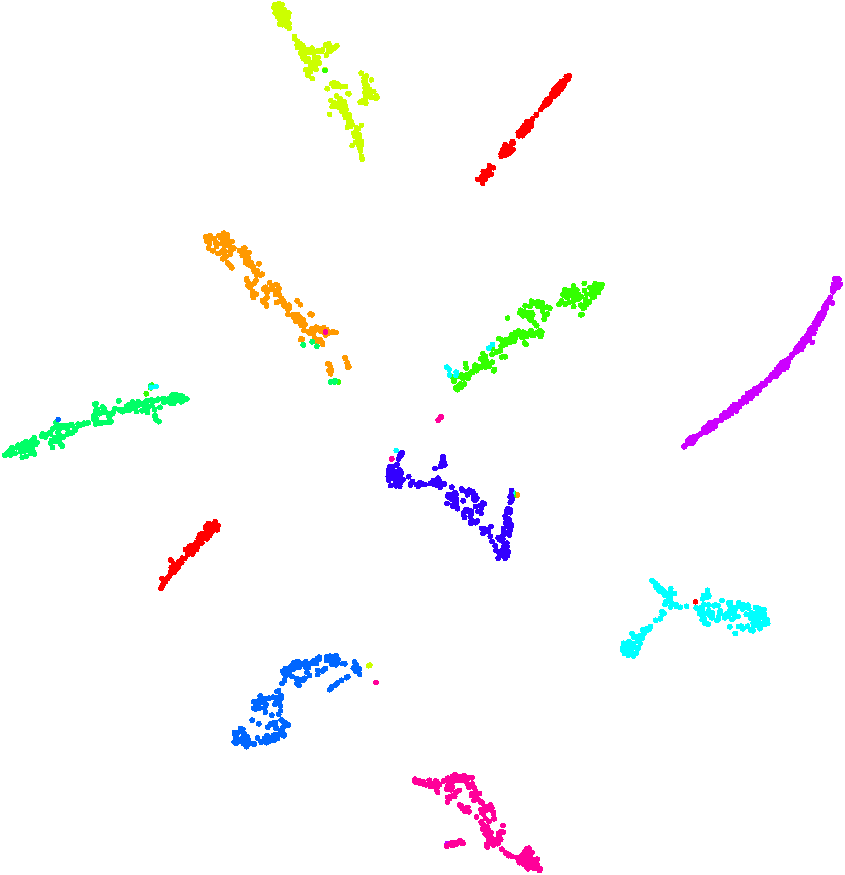}} \label{5-th iteration}}\hspace{5mm}
			\subfloat[10-th iteration]{{\includegraphics[width=0.16\textwidth]{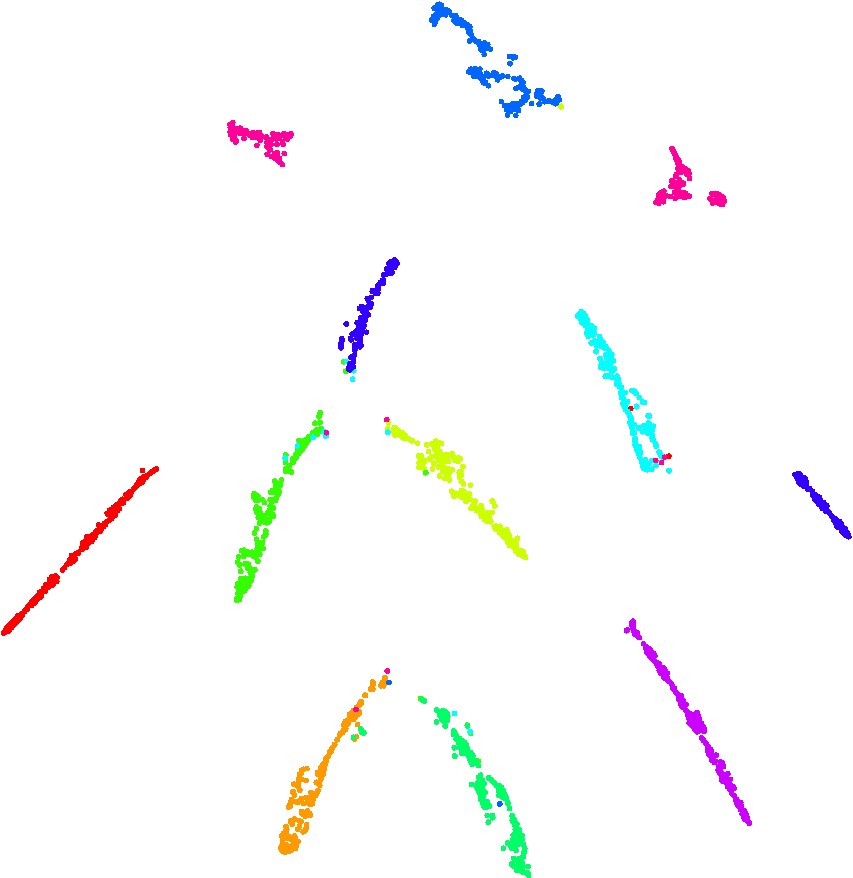}} \label{10-th iteration}}\hspace{5mm}
			\subfloat[20-th iteration]{{\includegraphics[width=0.16\textwidth]{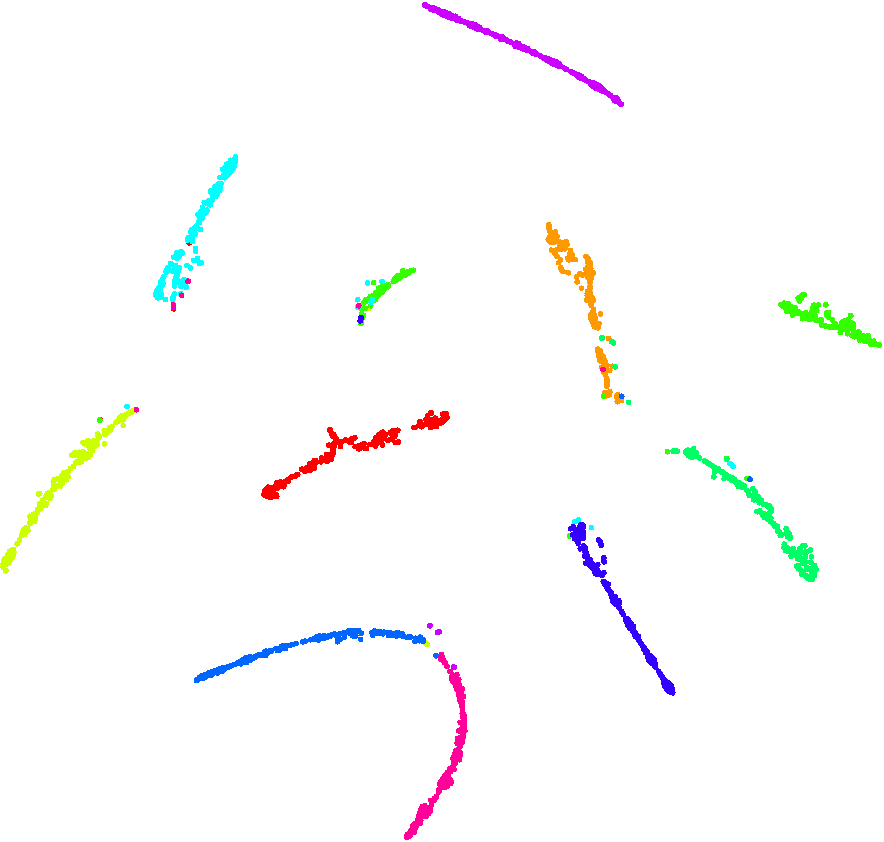}} \label{20-th iteration}}
			\caption{Evolution of data distribution by t-SNE on Handwritten dataset.}
			\label{tsne}
			}
\end{center}
\vspace{-13pt}
\end{figure*}
%----------------------------------------------------------------------------
\begin{figure*}[!t]
\vspace{-10pt}
\begin{center}{
		\centering
			\subfloat[Initialized ($\mathbf{Z}$)]{{\includegraphics[width=0.2\textwidth]{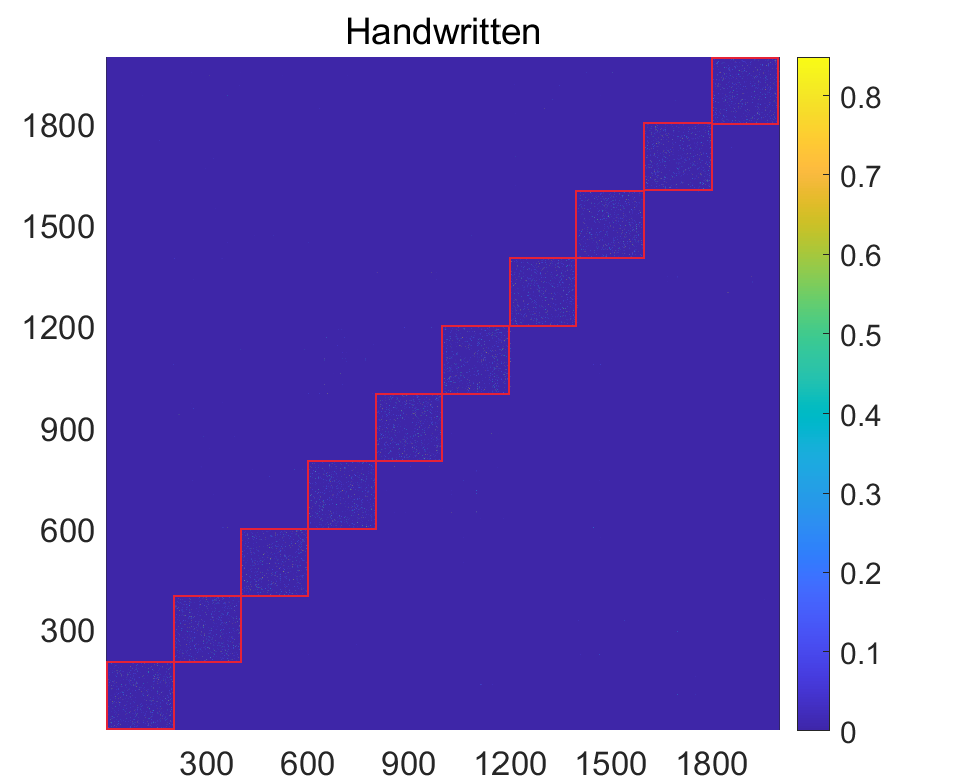}} \label{0-st iteration Z}}
			\subfloat[1-st iteration ($\mathbf{Z}$)]{{\includegraphics[width=0.2\textwidth]{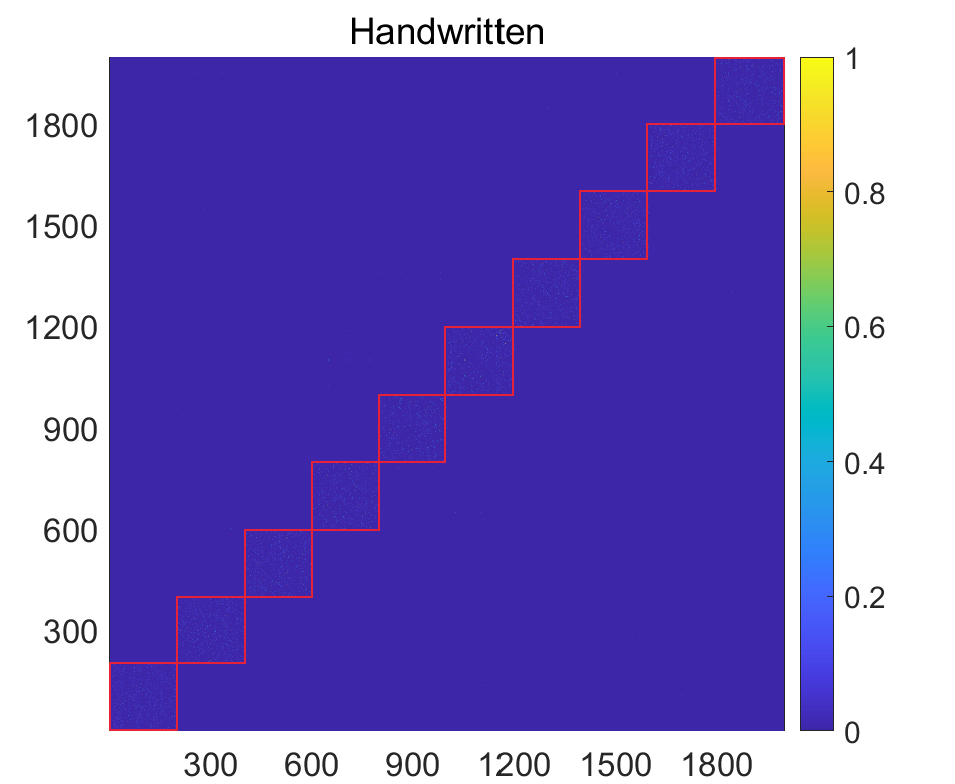}} \label{1-st iteration Z}}
			\subfloat[3-rd iteration ($\mathbf{Z}$)]{{\includegraphics[width=0.2\textwidth]{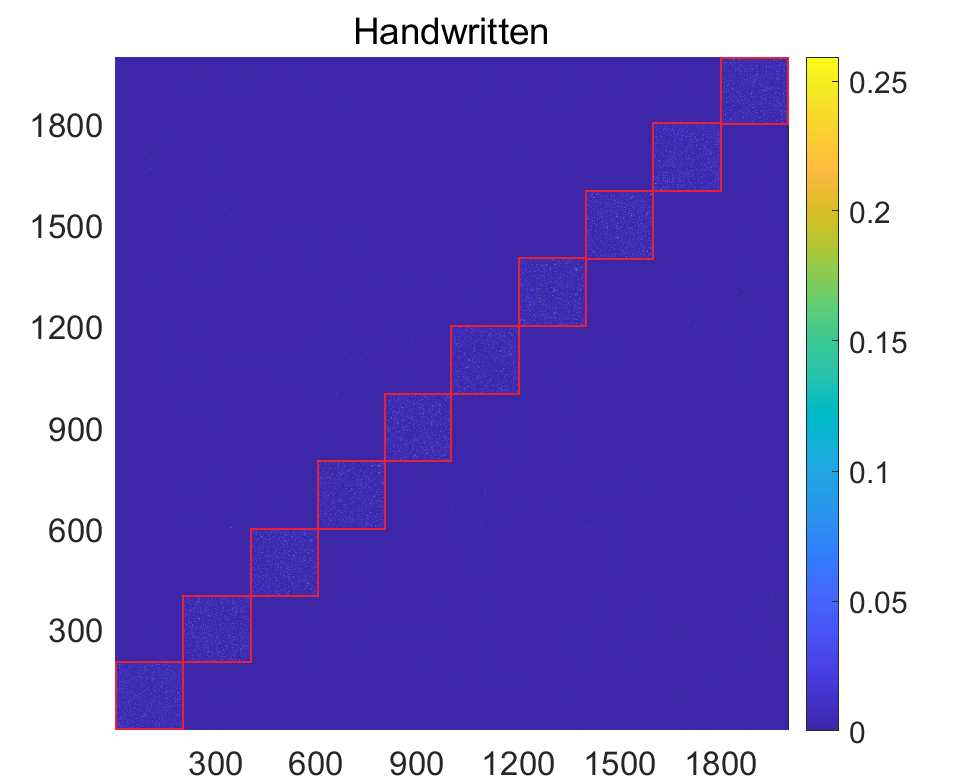}} \label{3-st iteration Z}}
			\subfloat[5-th iteration ($\mathbf{Z}$)]{{\includegraphics[width=0.2\textwidth]{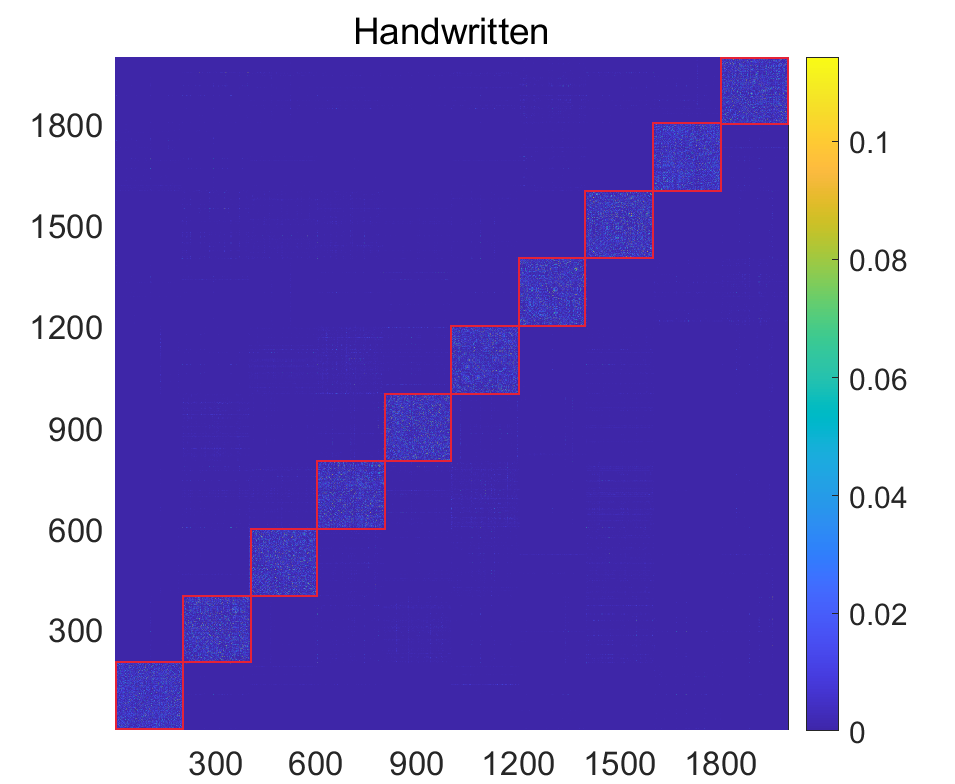}} \label{5-st iteration Z}}
			\subfloat[10-th iteration ($\mathbf{Z}$)]{{\includegraphics[width=0.2\textwidth]{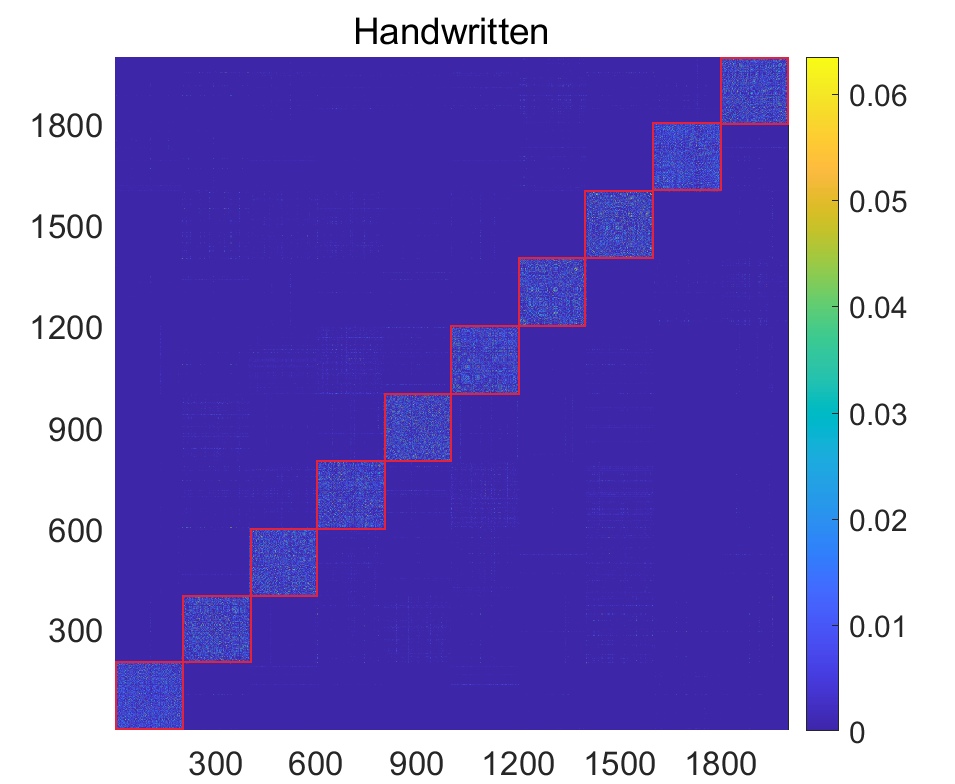}} \label{10-st iteration Z}}\\
			\vspace{-10pt}
			\subfloat[Initialized ($\mathbf{K}^{\ast}$)]{{\includegraphics[width=0.2\textwidth]{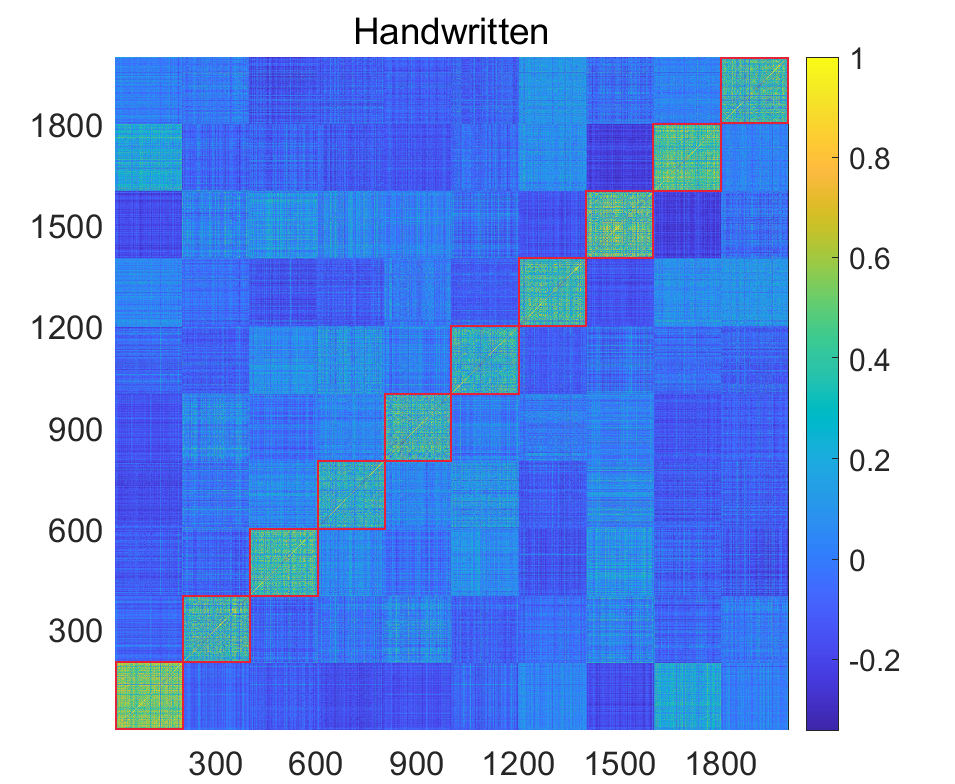}} \label{0-st iteration K*}}
			\subfloat[1-st iteration ($\mathbf{K}^{\ast}$)]{{\includegraphics[width=0.2\textwidth]{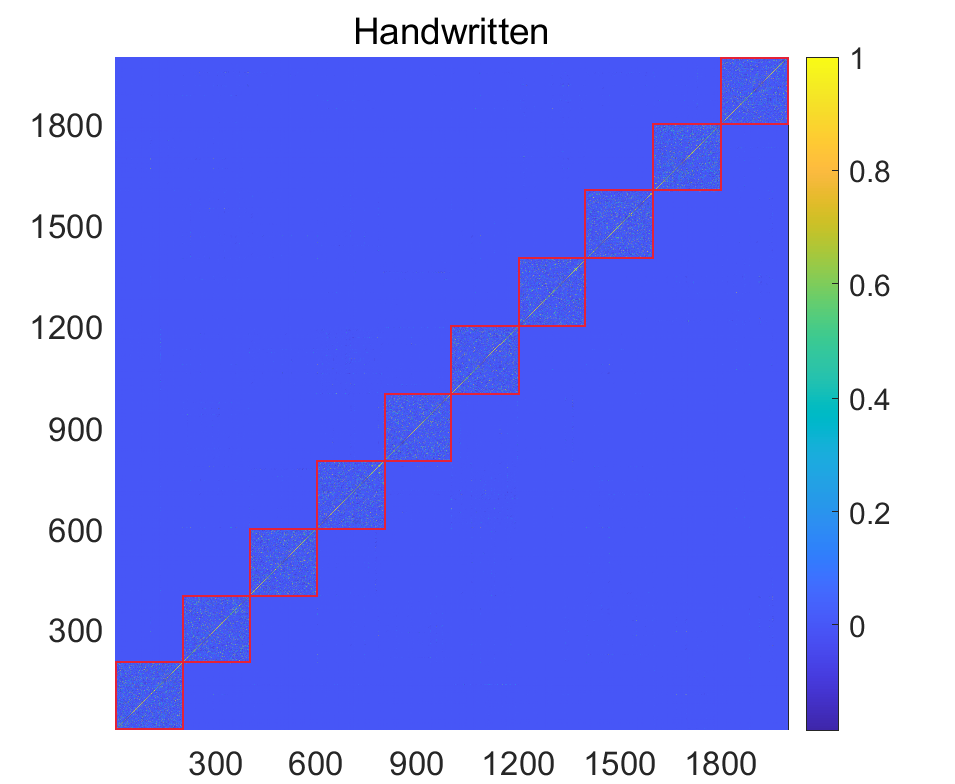}} \label{1-st iteration K*}}
			\subfloat[3-rd iteration ($\mathbf{K}^{\ast}$)]{{\includegraphics[width=0.2\textwidth]{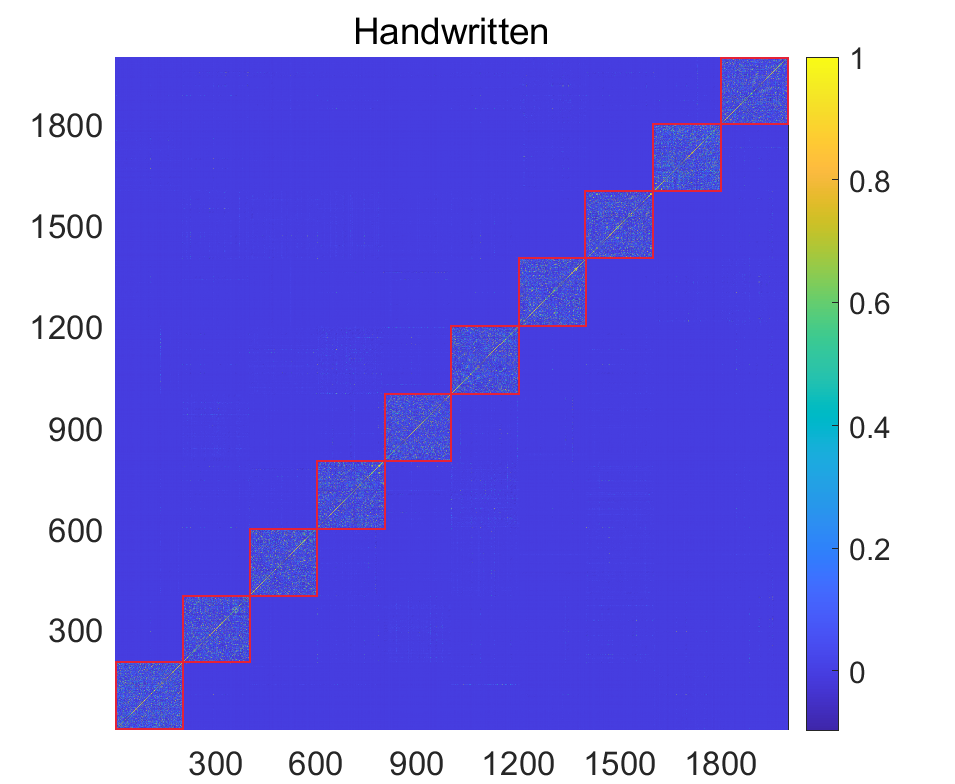}} \label{3-st iteration K*}}
			\subfloat[5-th iteration ($\mathbf{K}^{\ast}$)]{{\includegraphics[width=0.2\textwidth]{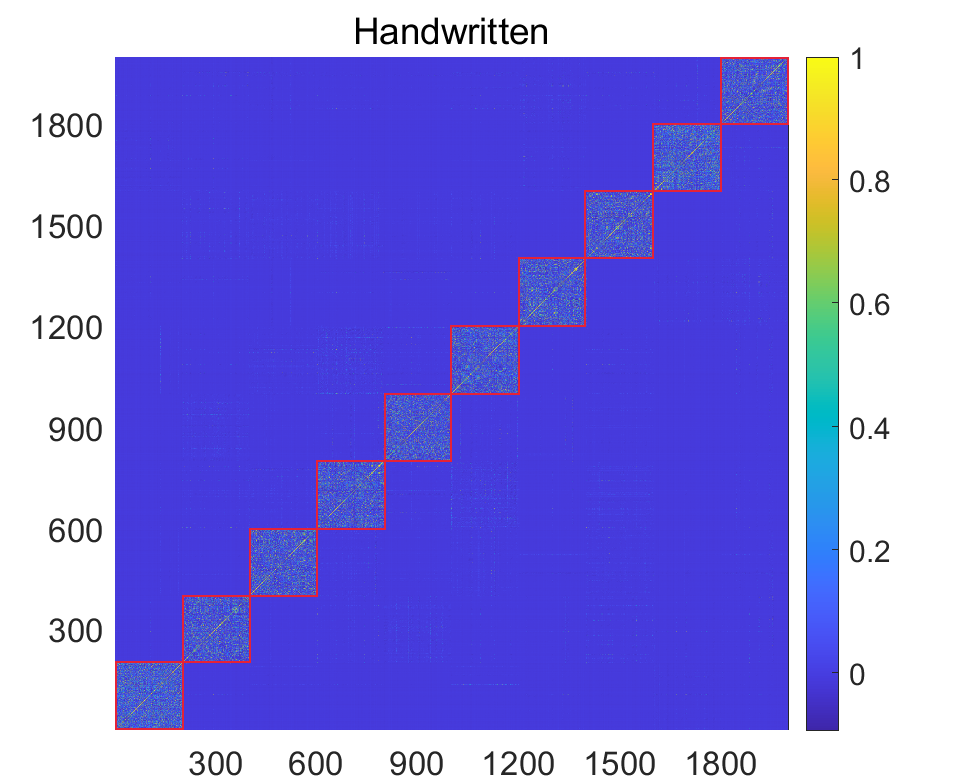}} \label{5-st iteration K*}}
			\subfloat[10-th iteration ($\mathbf{K}^{\ast}$)]{{\includegraphics[width=0.2\textwidth]{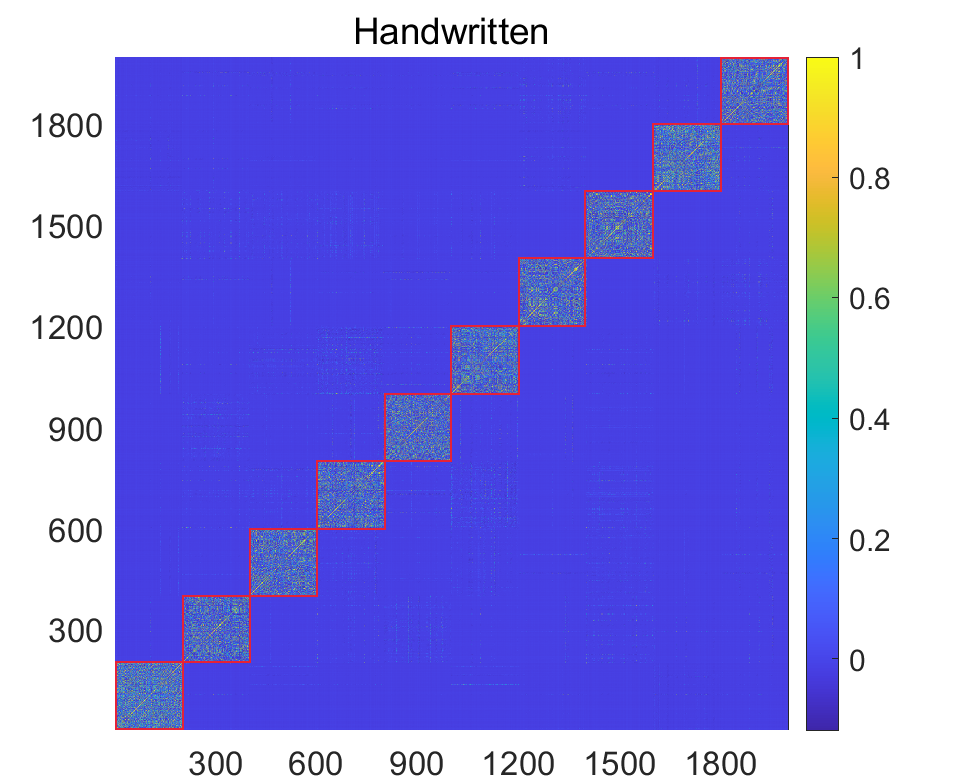}} \label{10-st iteration K*}}
			\caption{Evolution of affinity graph $\mathbf{Z}$ and neighborhood kernel $\mathbf{K}^{\ast}$ learned by our proposed algorithm on Handwritten dataset.}
			\label{Visualization-Z-K}
			}
\end{center}
\vspace{-13pt}
\end{figure*}
%----------------------------------------------------------
\begin{figure*}[!t]
\vspace{-10pt}
\begin{center}{
		\centering
            \subfloat[YALE]{{\includegraphics[width=0.163\textwidth]{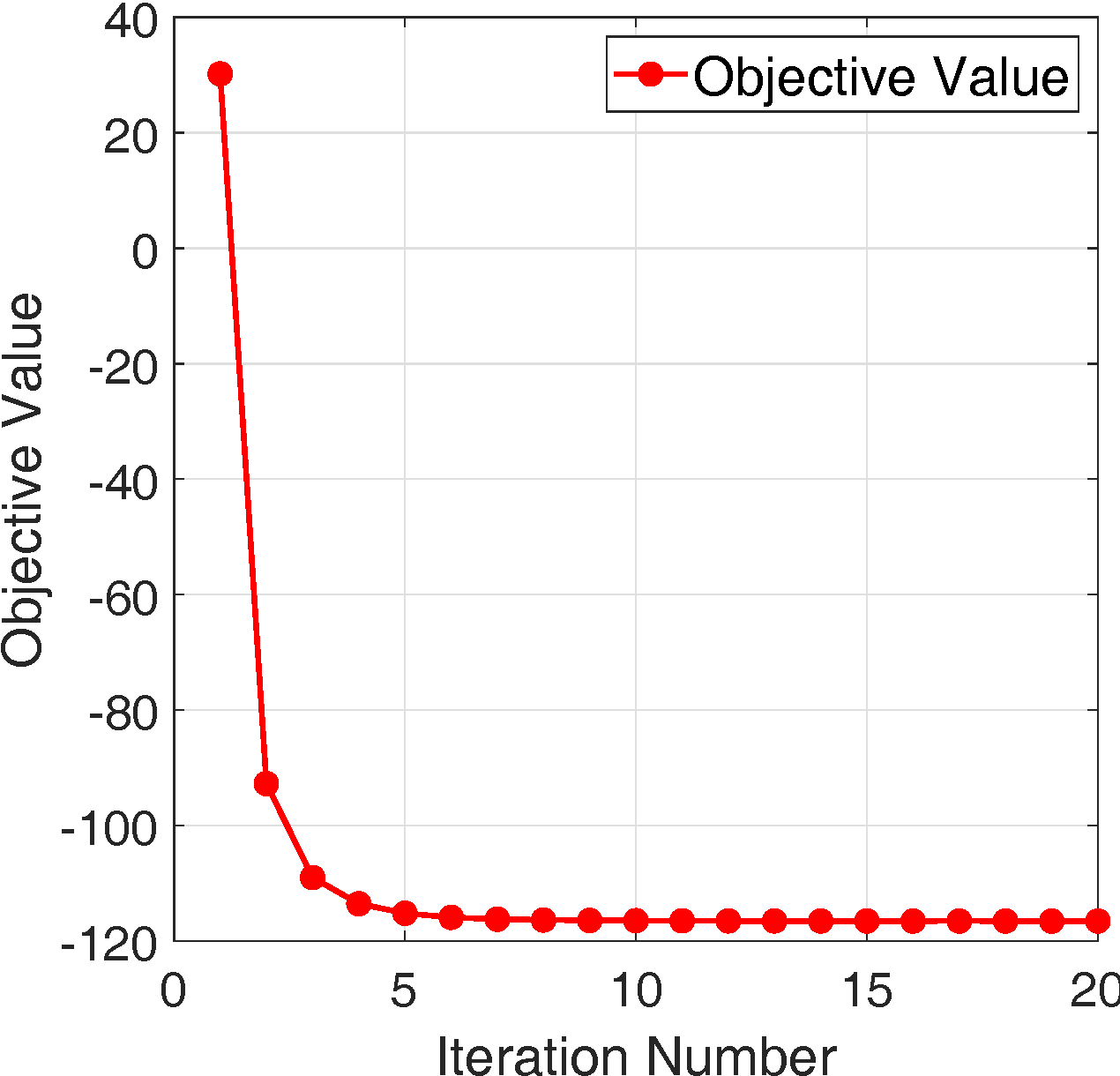}} \label{C-1}}
            \subfloat[BBC]{{\includegraphics[width=0.16\textwidth]{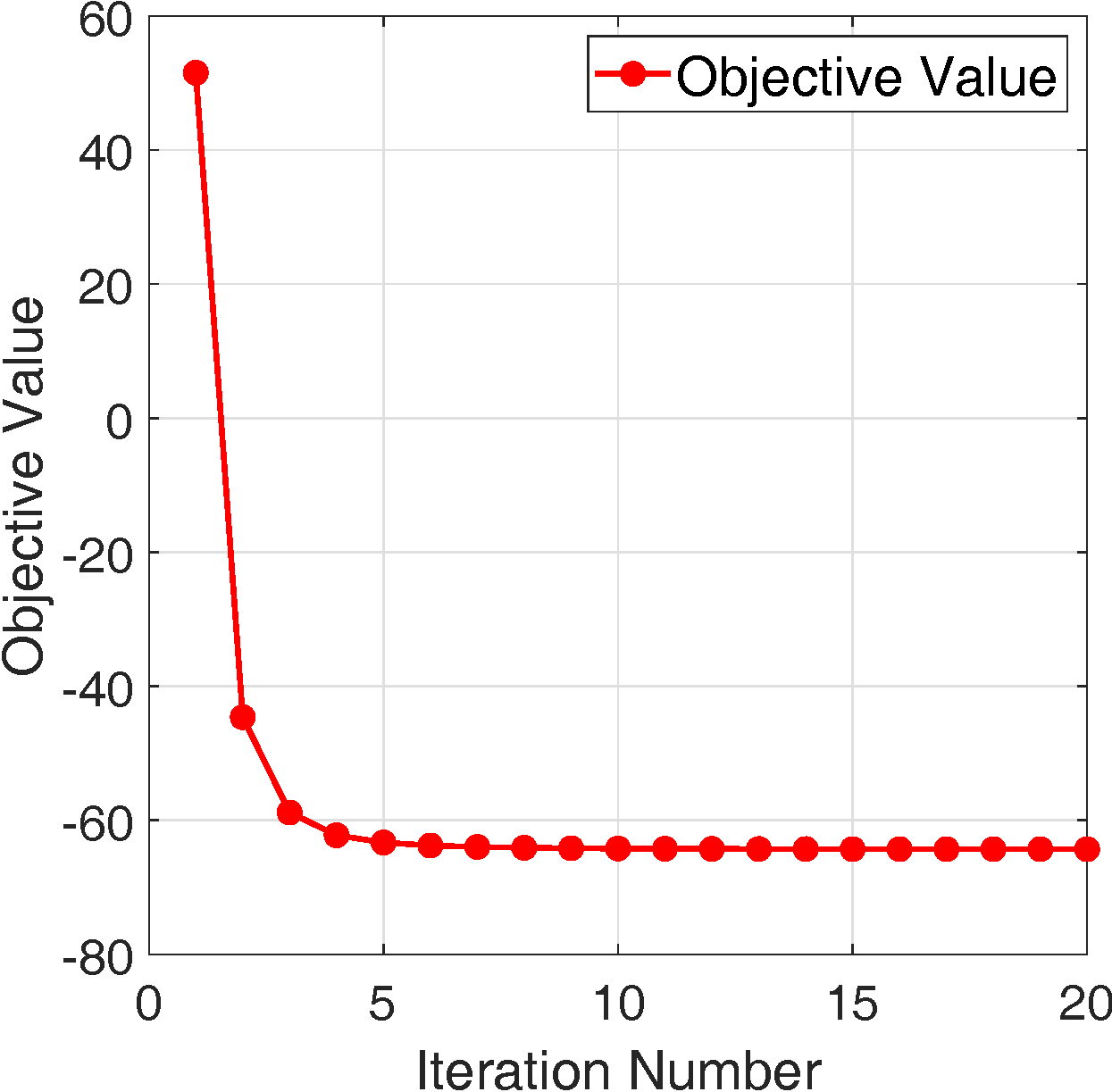}} \label{C-5}}
            \subfloat[BBCSport]{{\includegraphics[width=0.161\textwidth]{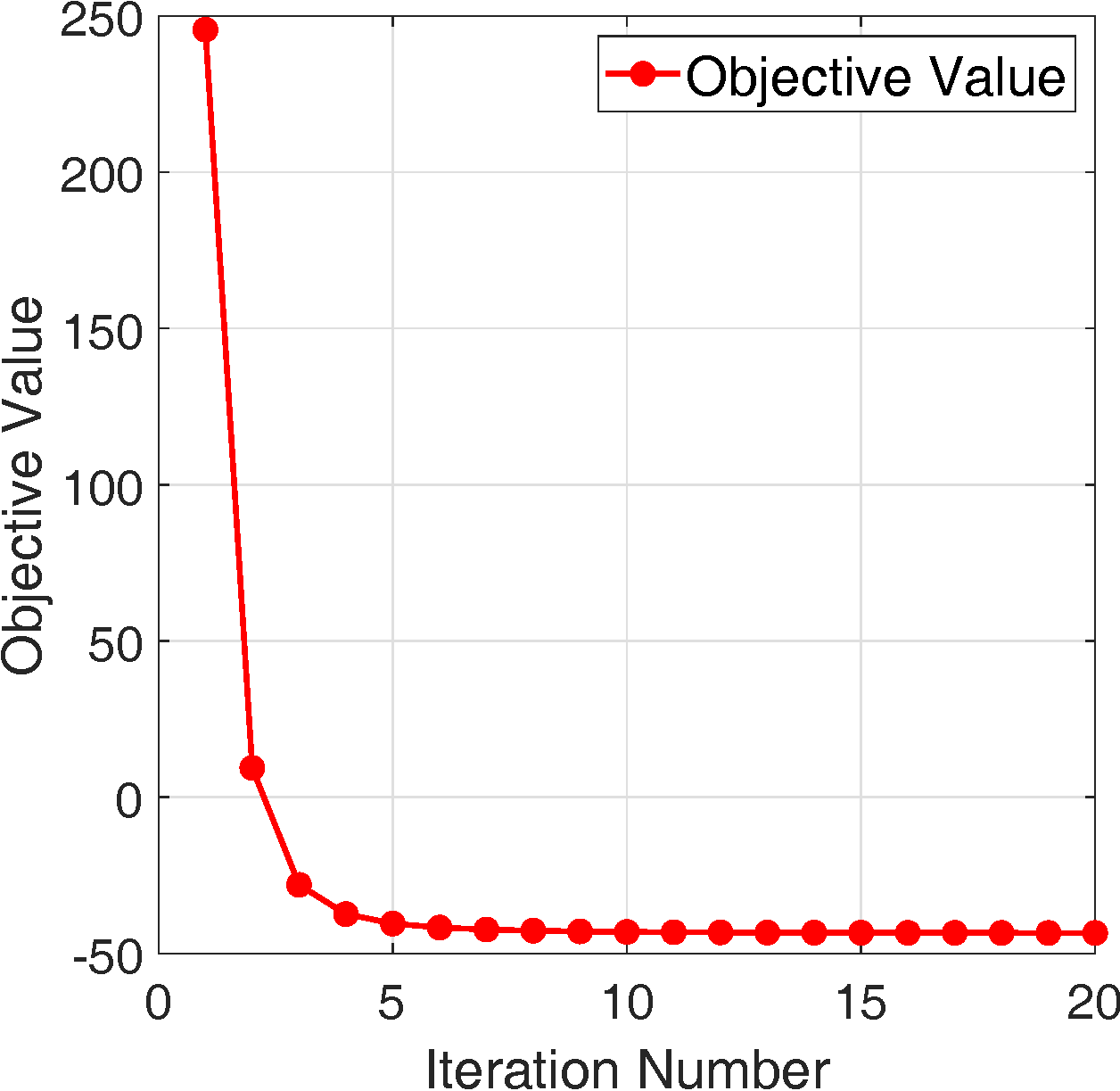}} \label{C-6}}
            \subfloat[Handwritten]{{\includegraphics[width=0.16\textwidth]{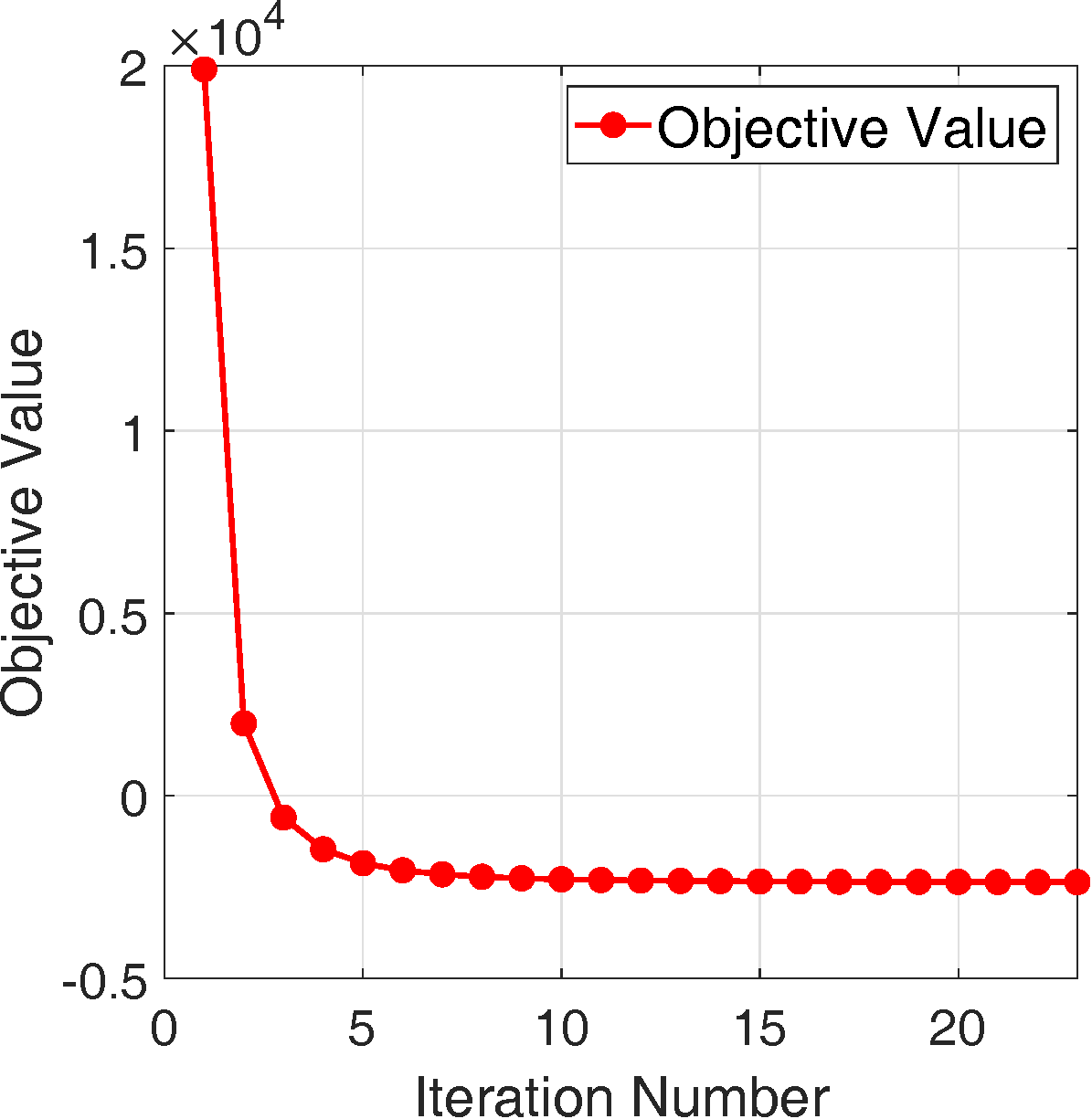}} \label{C-10}}
            \subfloat[Mfeat]{{\includegraphics[width=0.161\textwidth]{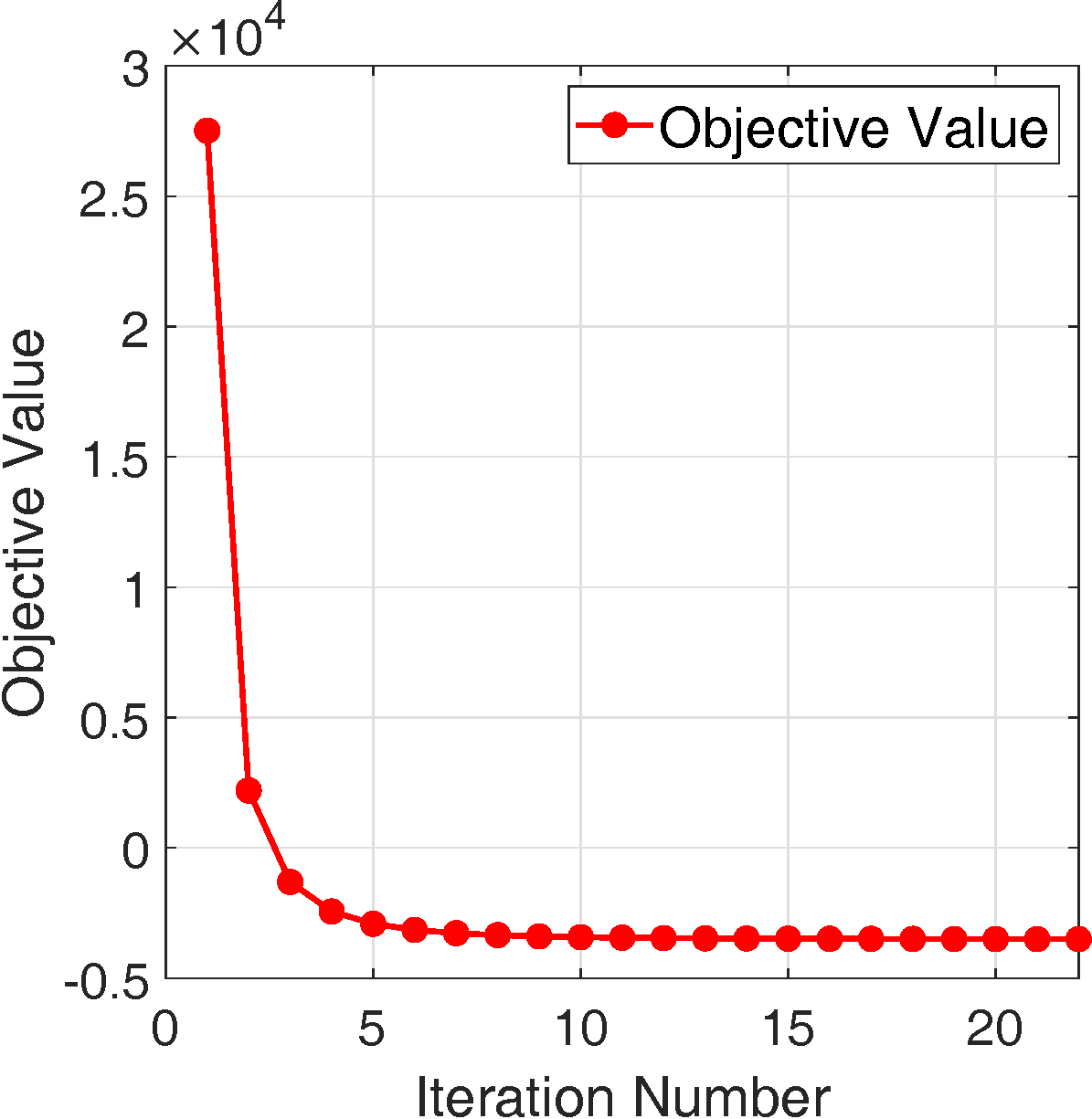}} \label{C-11}}
            \subfloat[Scene15]{{\includegraphics[width=0.156\textwidth]{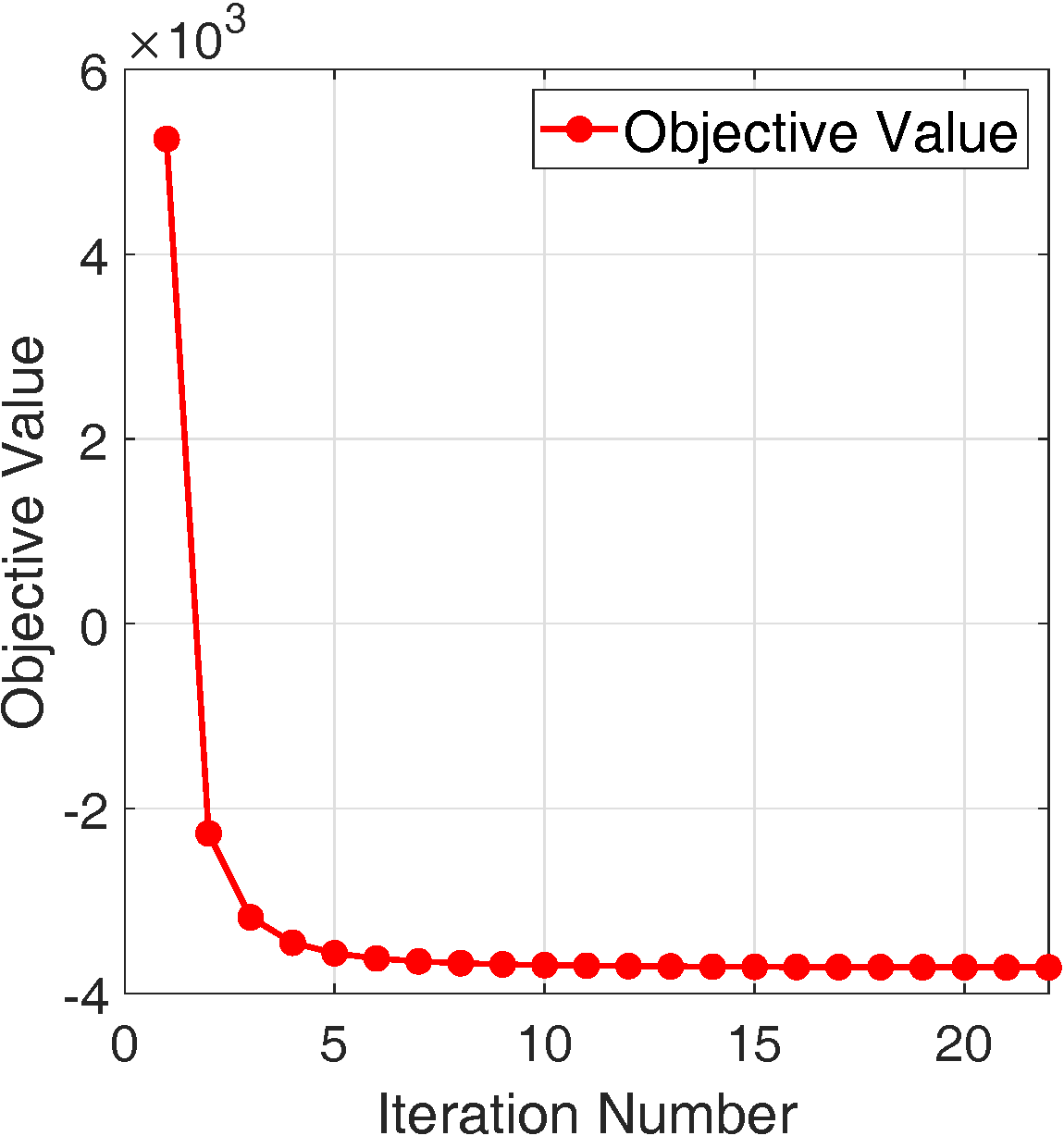}} \label{C-12}}
			\caption{{The convergence of the proposed LSWMKC on six datasets. Other datasets' results are provided in the appendix.}}
			\label{Convergence}
			}
\end{center}
\vspace{-13pt}
\end{figure*}
%----------------------------------------------------------
\subsection{Kernel Weight Analysis}
We further evaluate the distribution of the learned kernel weights on twelve datasets. As Figure \ref{Weights} shows, the kernel weight distributions of MKKM-MR, ONKC, and LKAM vary greatly and are highly sparse on most datasets. Such sparsity would incur clustering information across multiple views cannot be fully utilized. In contrast, the weight distributions of our proposed algorithm are non-sparse on all the datasets, thus the latent clustering information can be significantly exploited.
%----------------------------------------------------------
\subsection{Visualization}
To visually demonstrate the learning process of the proposed localized building strategy, Figure \ref{tsne} plot the t-SNE visual results on Handwritten dataset, which clearly shows the separation of different clusters during the iteration. Moreover, Figure \ref{Visualization-Z-K} plots the evolution of the learned affinity graph $\mathbf{Z}$ and neighborhood kernel $\mathbf{K}^{\ast}$ on Handwritten dataset. Clearly, the noises are gradually removed and the clustering structures become clearer. Besides, $\mathbf{K}^{\ast}$ can further denoise $\mathbf{Z}$, which exhibits more evident block diagonal structures. These results can well illustrate the effectiveness of our localized strategy.  
%----------------------------------------------------------------------------
\begin{figure*}[!t]
\vspace{-10pt}
\begin{center}{
		\centering
            \subfloat[BBC(ACC)]{{\includegraphics[width=0.165\textwidth]{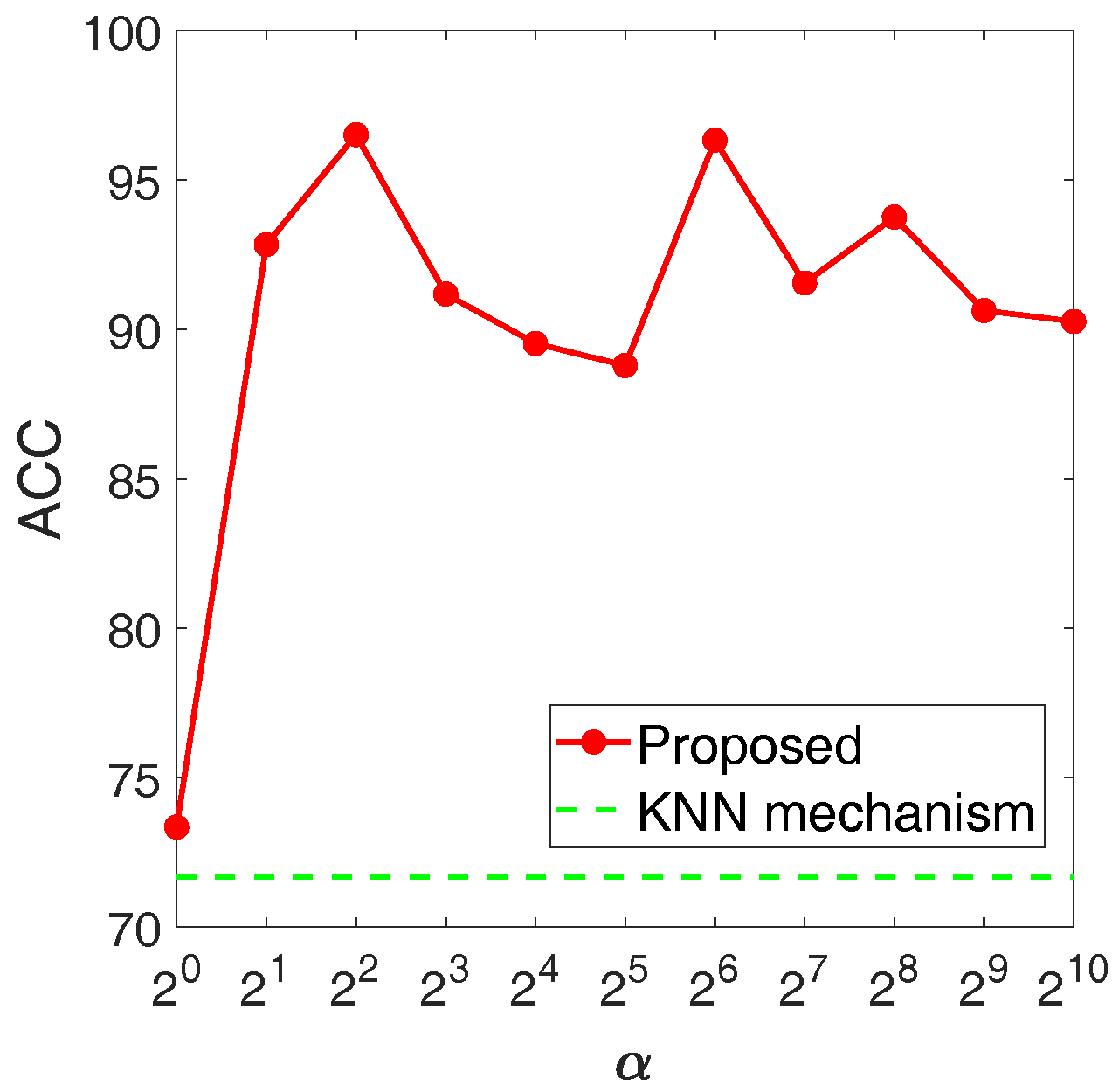}}}
            \subfloat[BBC(NMI)]{{\includegraphics[width=0.165\textwidth]{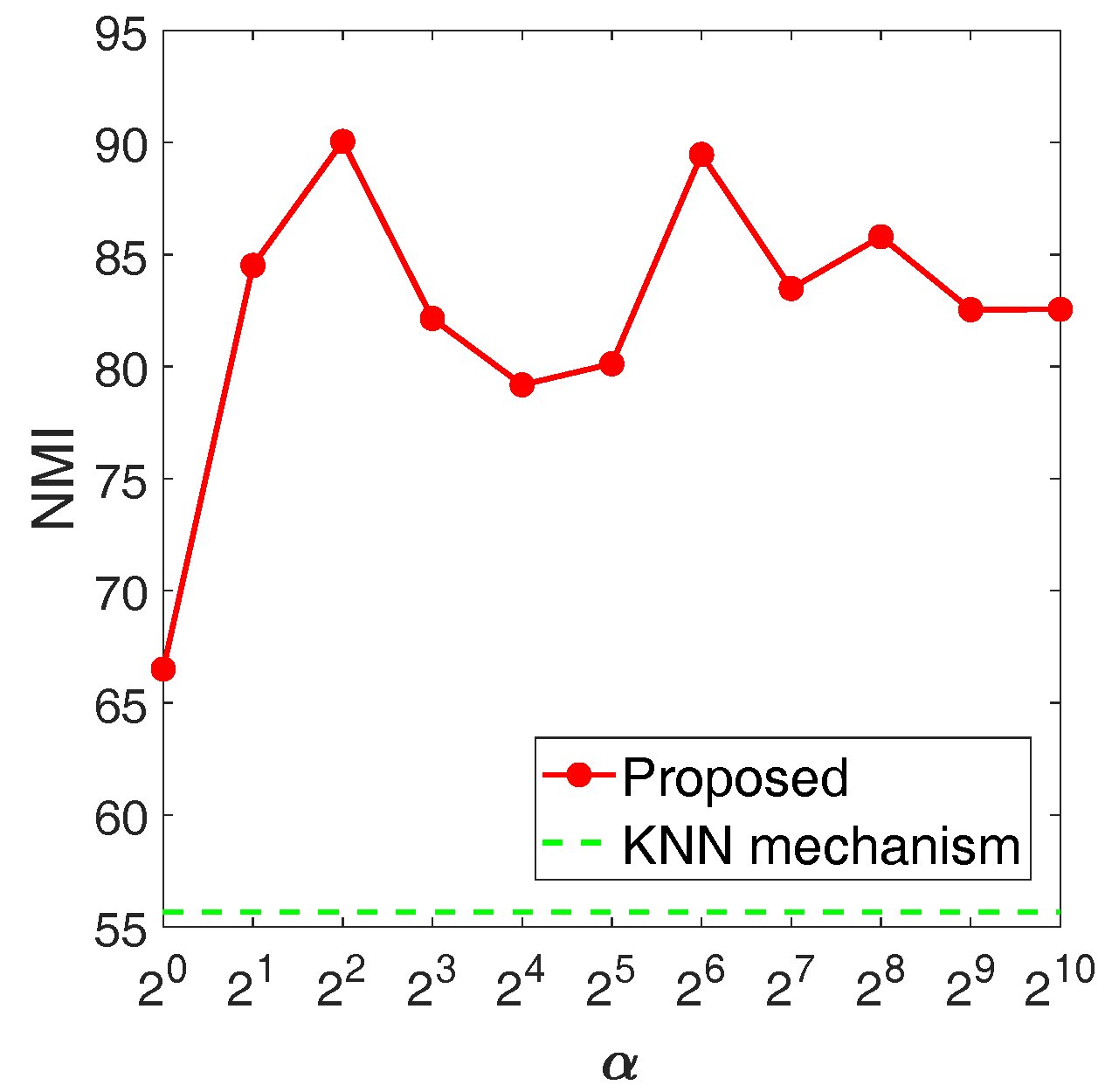}}}
            \subfloat[BBCSport(ACC)]{{\includegraphics[width=0.165\textwidth]{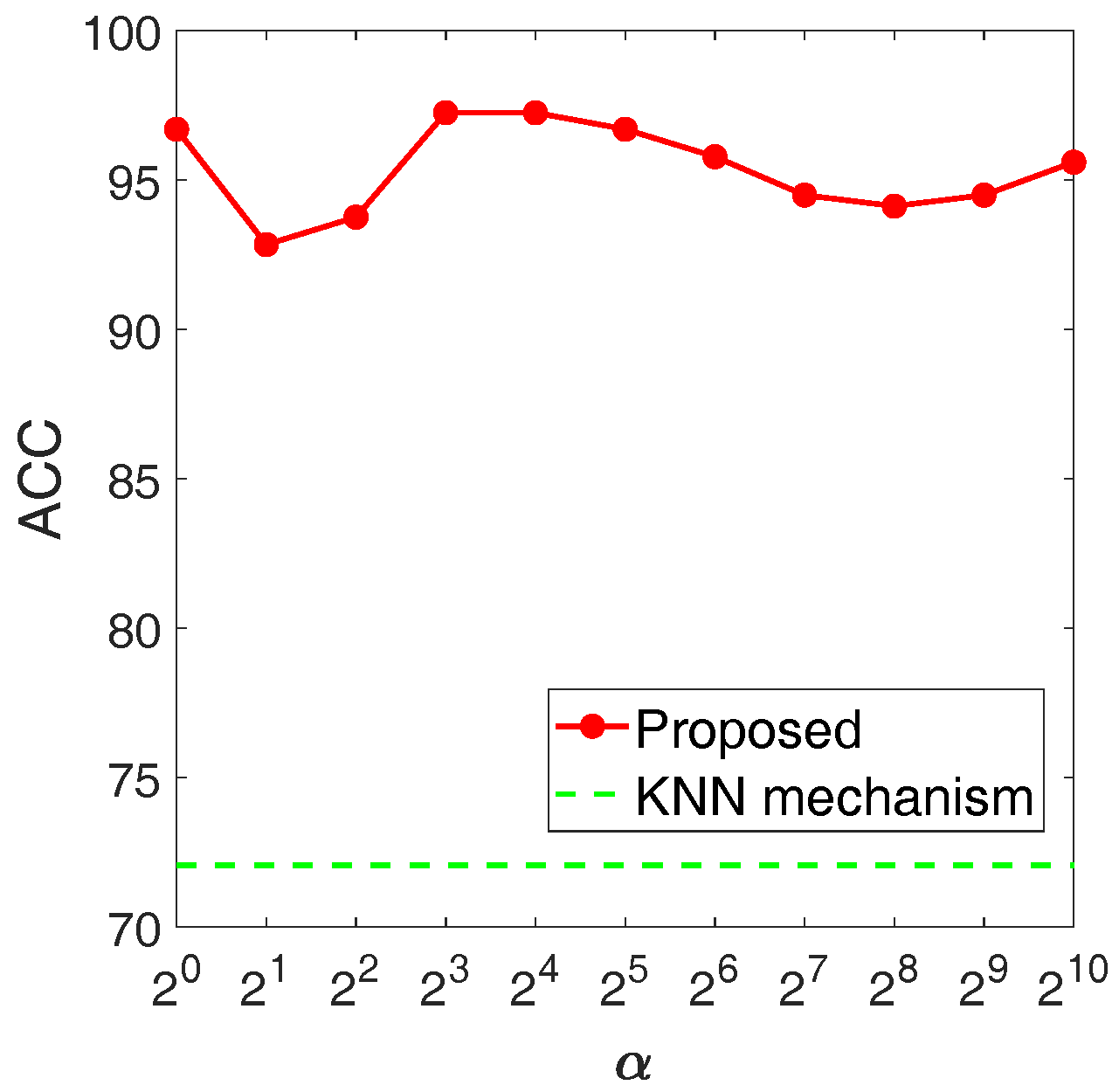}}}
            \subfloat[BBCSport(NMI)]{{\includegraphics[width=0.165\textwidth]{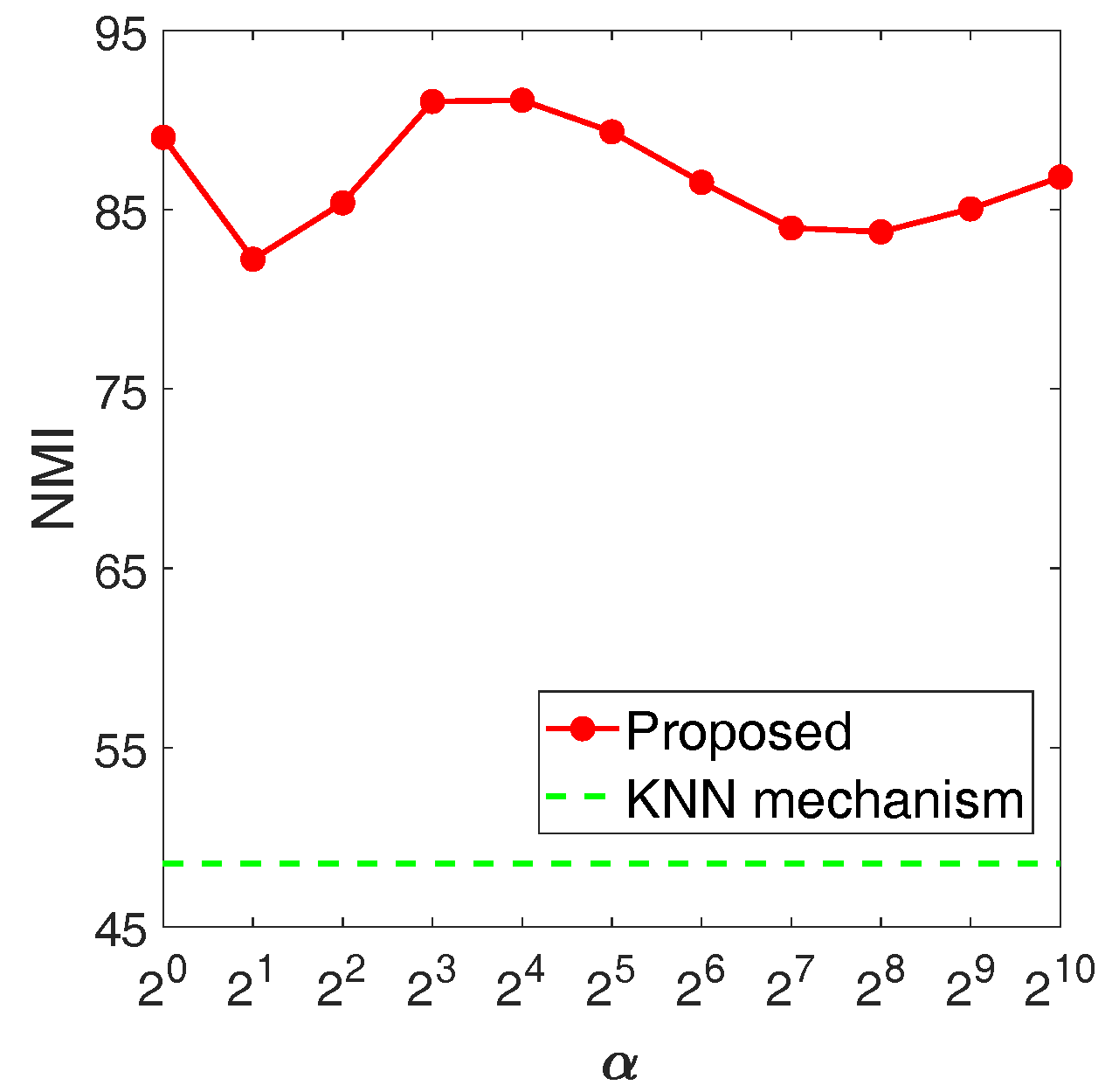}}}
            \subfloat[Caltech101-mit(ACC)]{{\includegraphics[width=0.165\textwidth]{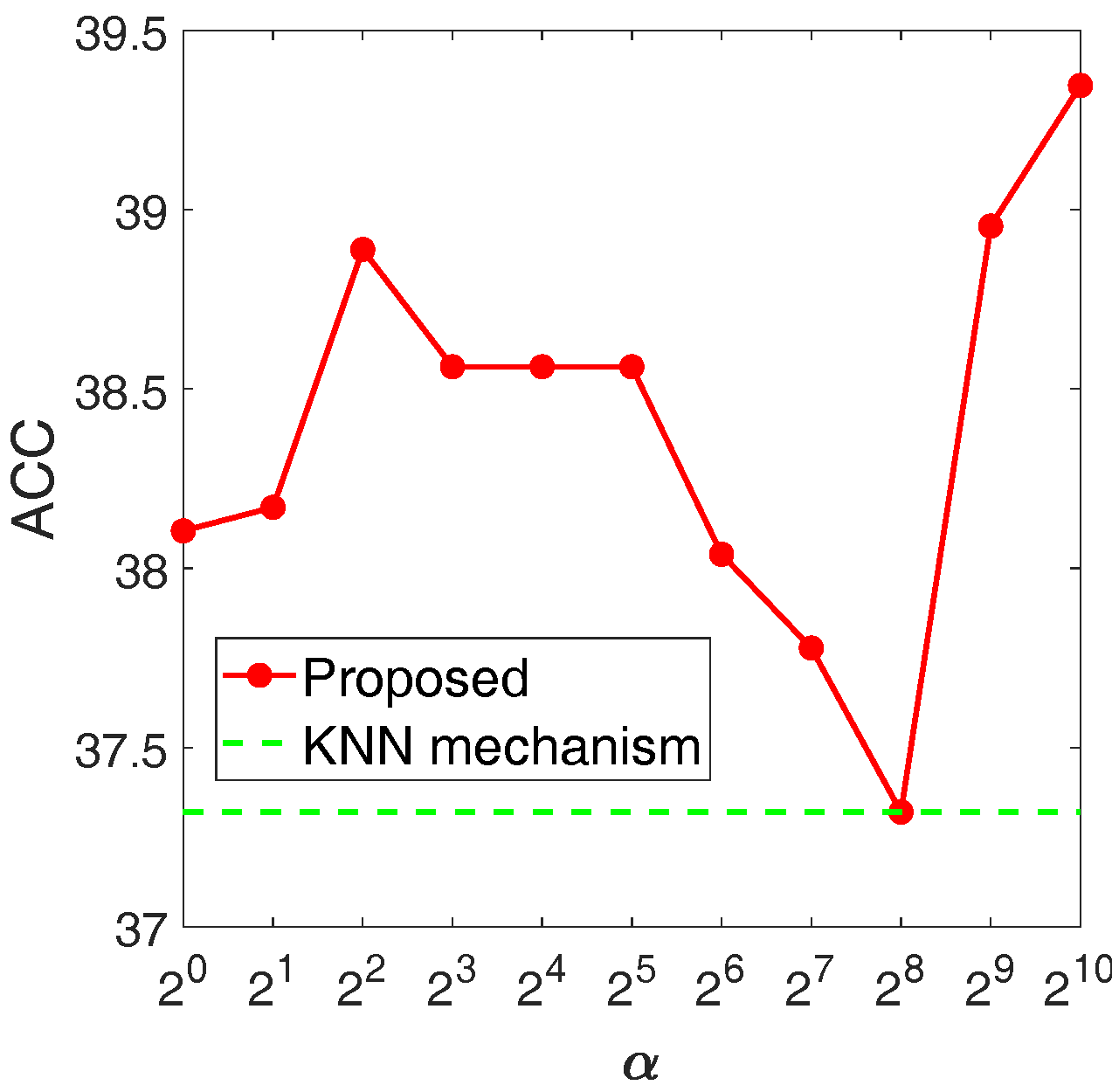}}}
            \subfloat[Caltech101-mit(NMI)]{{\includegraphics[width=0.165\textwidth]{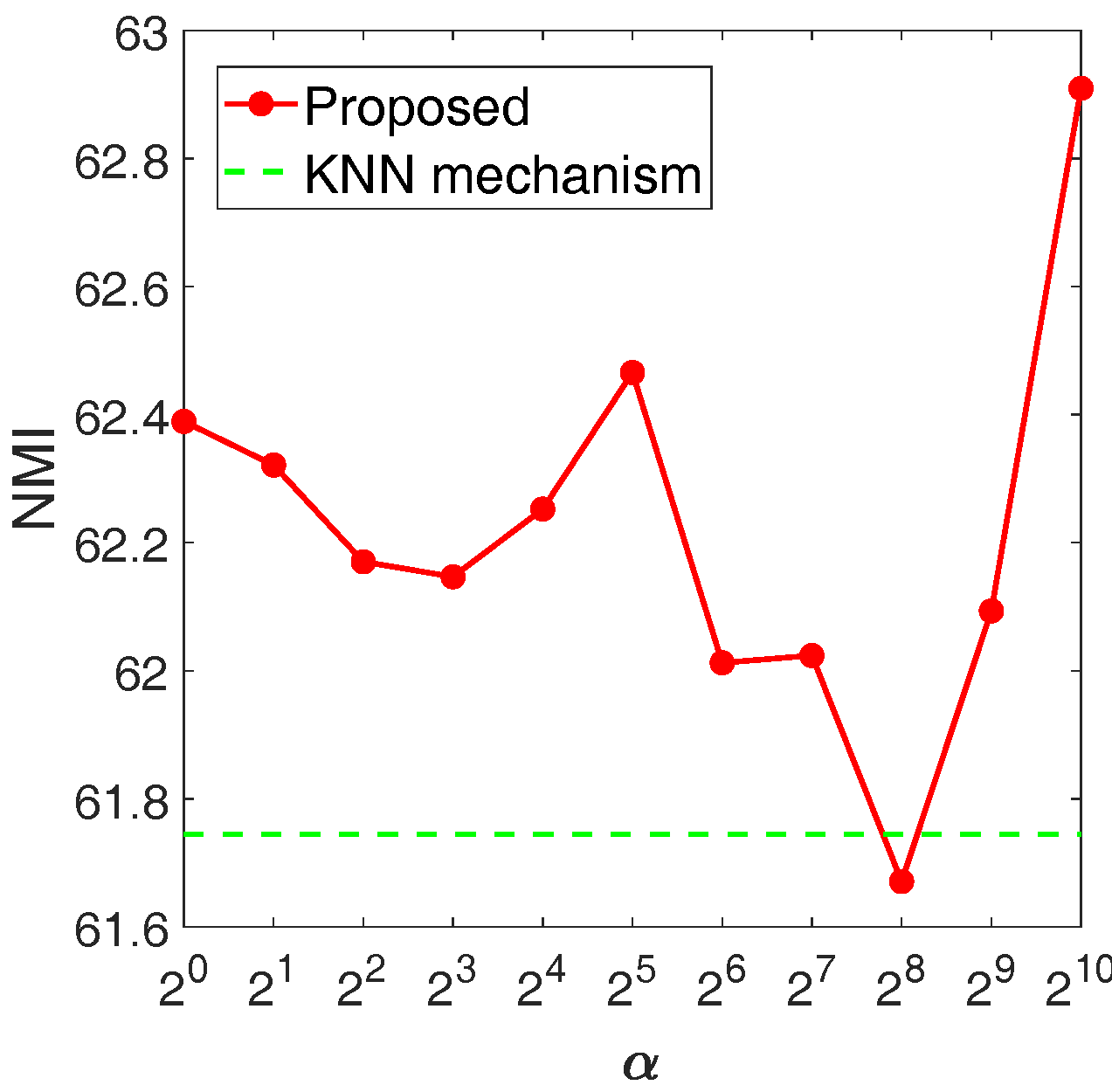}}}
			\caption{{Parameter sensitivity study of hyper-parameter $\alpha$ on BBC, BBCSport, and Caltech101-mit datasets.}}
			\label{Sensitivity}
			}
\end{center}
\vspace{-10pt}
\end{figure*}

\begin{figure*}[!t]
\vspace{-10pt}
\begin{center}{
		\centering
            \subfloat[Caltech101-7 (ACC)]{{\includegraphics[width=0.165\textwidth]{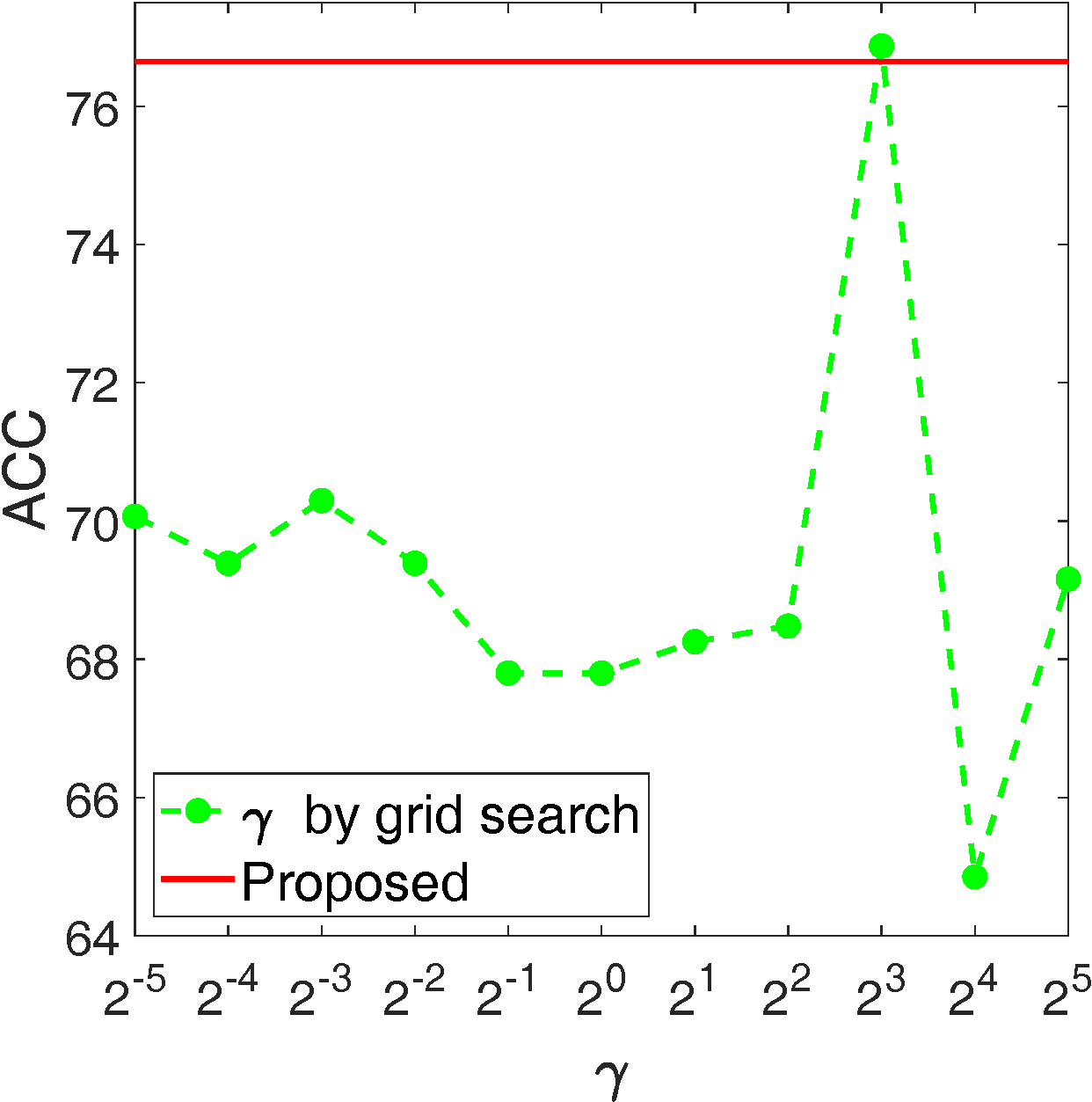}}}
            \subfloat[Caltech101-7 (NMI)]{{\includegraphics[width=0.165\textwidth]{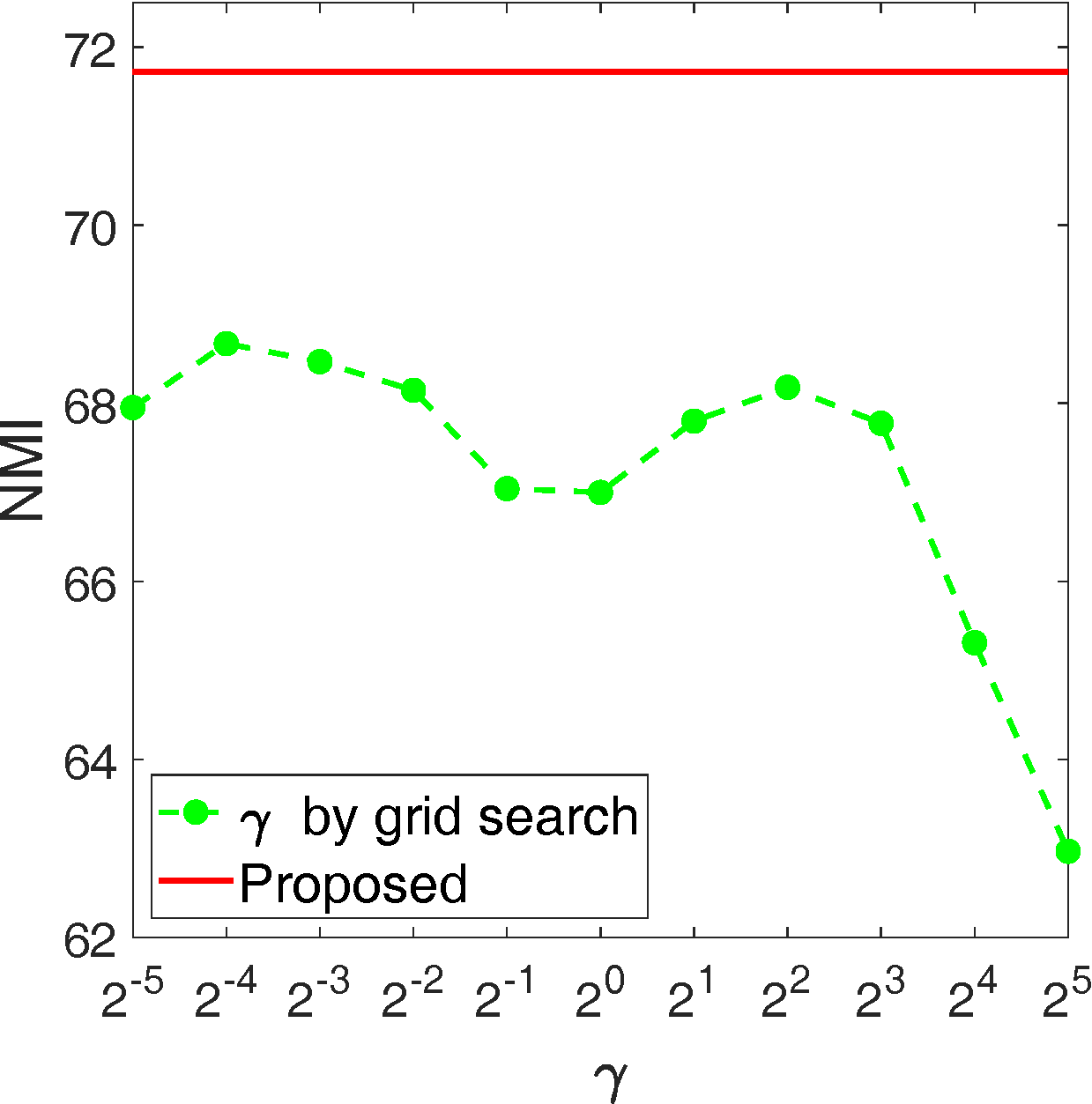}}}
            \subfloat[Caltech101-7 (Purity)]{{\includegraphics[width=0.165\textwidth]{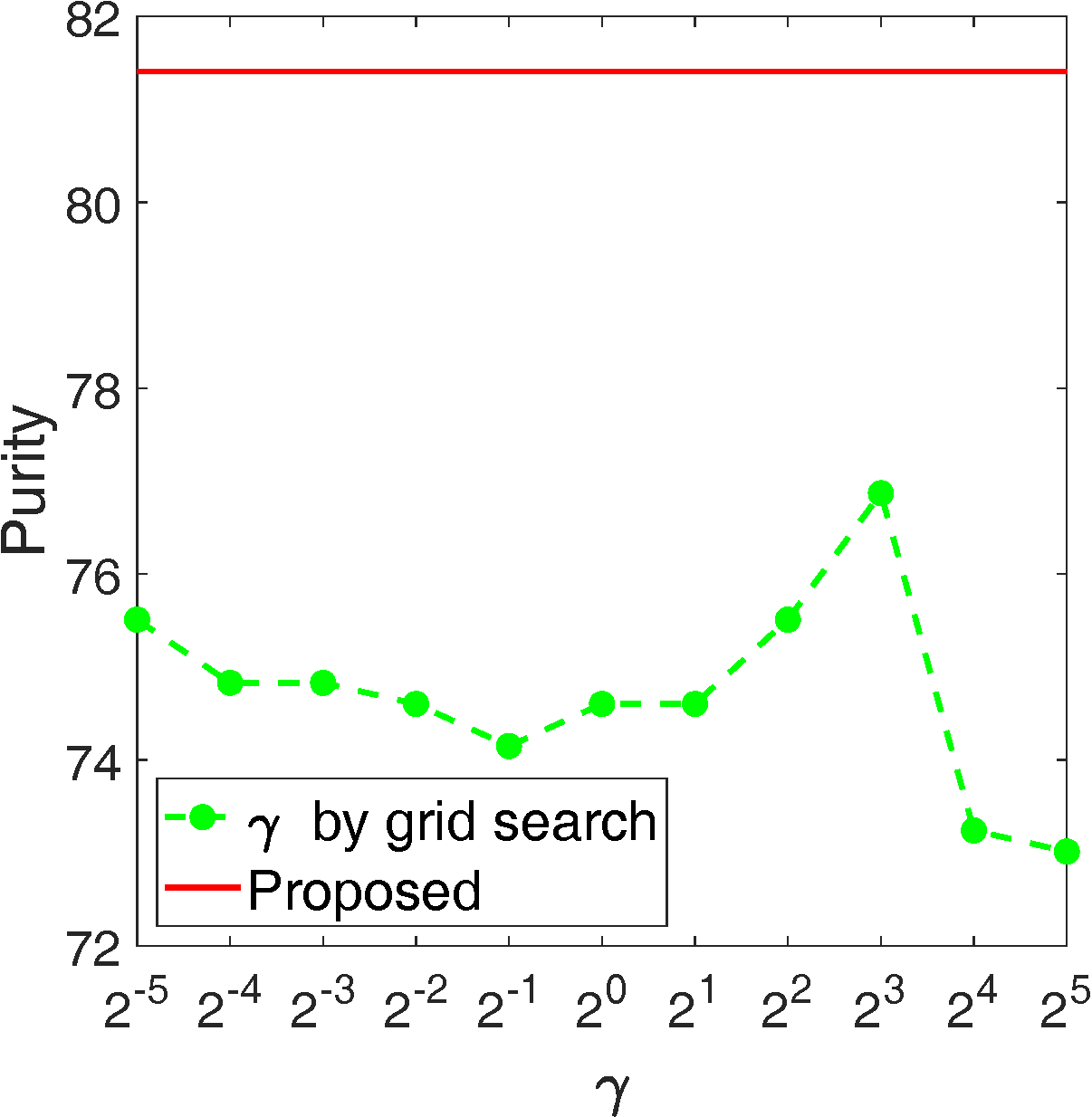}}}            
            \subfloat[BBCSport (ACC)]{{\includegraphics[width=0.165\textwidth]{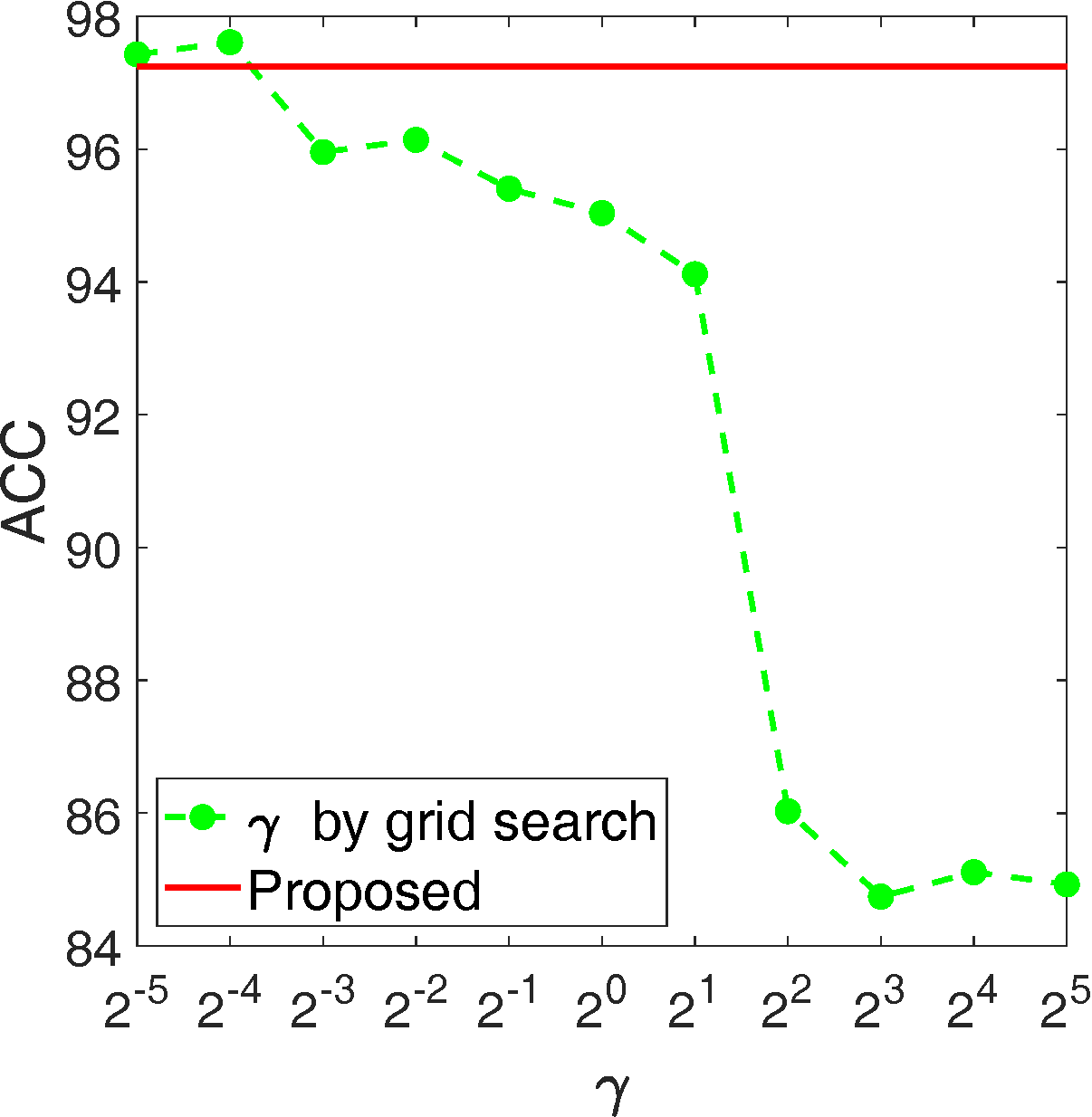}}}
            \subfloat[BBCSport (NMI)]{{\includegraphics[width=0.165\textwidth]{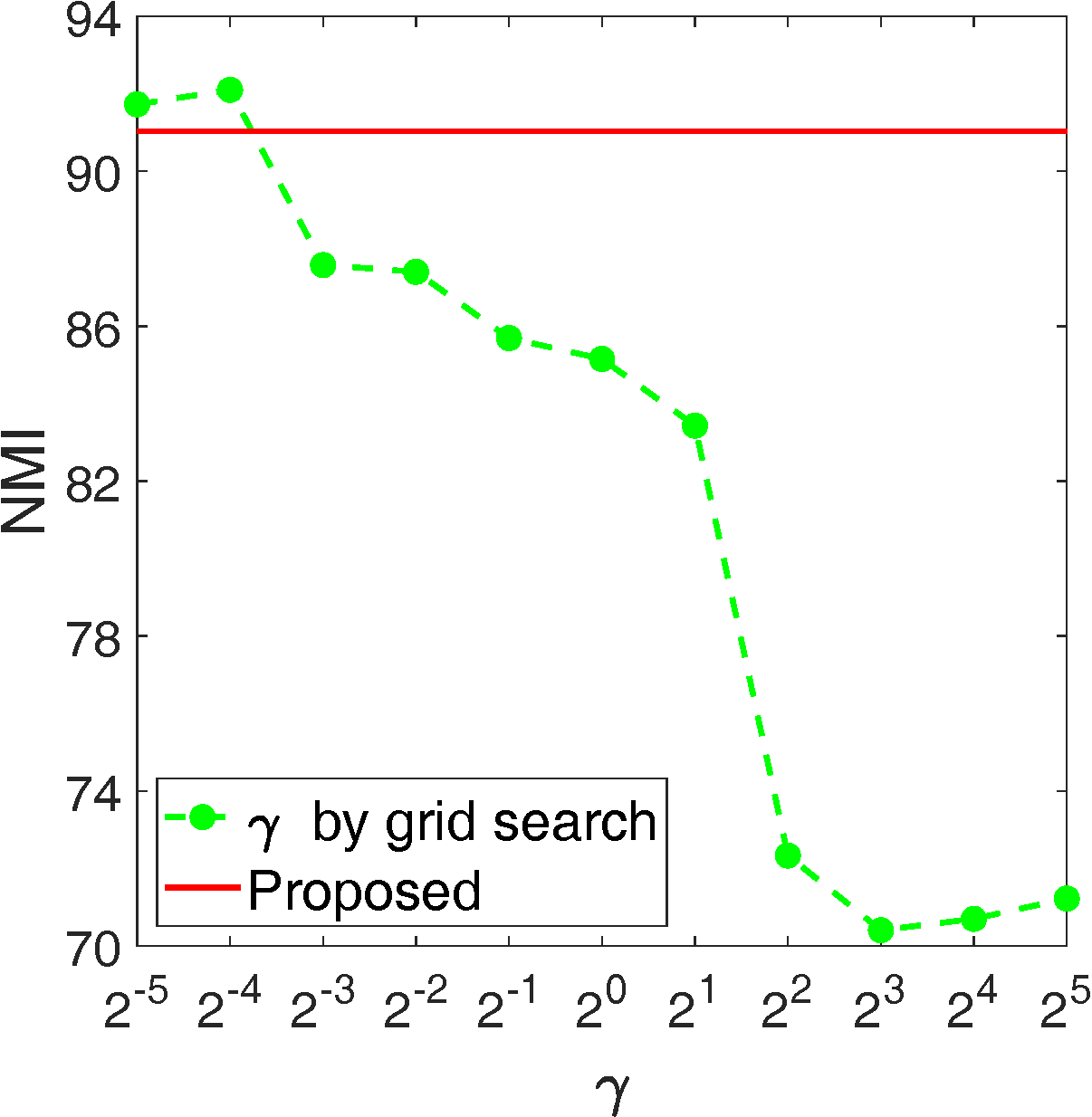}}}
            \subfloat[BBCSport (Purity)]{{\includegraphics[width=0.165\textwidth]{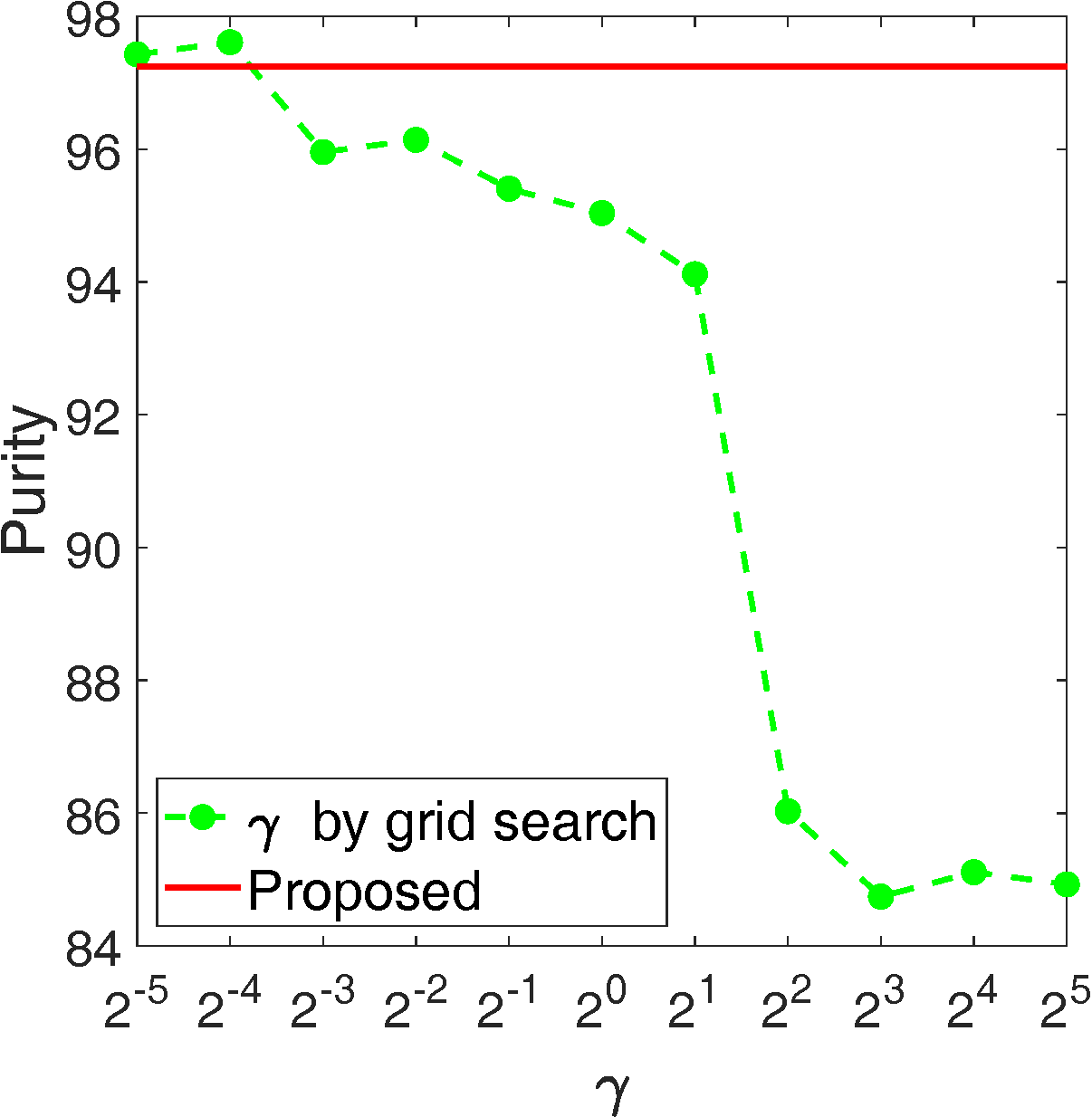}}}
			\caption{{Ablation study of $\mathbf{\gamma}$ by grid search on Caltech101-7 and BBCSport datasets. Other datasets' results are provided in the appendix.}}
			\label{Abalation_gamma}
			}
\end{center}
\vspace{-10pt}
\end{figure*}
%----------------------------------------------------------
\subsection{Convergence and Parameter Sensitivity} %
According to our previous theoretical analysis, the convergence of our LSWMKC model has been verified with a local optimal. Here, experimental verification is further conducted to illustrate this issue. Figure \ref{Convergence} reports the evolvement of optimization goal during iteration. Obviously, the objective function values monotonically decrease and quickly converge during the iteration. 

We further evaluate the parameter sensitivity of $\alpha$ by grid search varying in $[2^{0},2^{1},\cdots,2^{10}]$ on BBC, BBCSport, and Caltech101-mit datasets. From Figure \ref{Sensitivity}, we find the proposed method exhibits much better performance compared to KNN mechanism in a wide range of $\alpha$, making it practical in real-world applications. 
%----------------------------------------------------------
\subsection{Ablation Study on Tuning $\mathbf{\gamma}$ by Grid Search}\label{Ablation}
To evaluate the effectiveness of our learning $\gamma$ manner in Section \ref{Initialization}, we perform ablation study by tuning $\gamma$ in $[2^{-5},2^{-4},\cdots,2^{5}]$. The range of $\alpha$ still varies in $[2^{0},2^{1},\cdots,2^{10}]$. 
Figure. \ref{Abalation_gamma} plots the results on Caltech101-7 and BBCSport datasets. The red line denotes our reported results. The green dash line denotes the tuning results, for simplicity, $\alpha$ is fixed at the index of the optimal results. 

As can be seen, our learning manner exceeds the tuning manner with a large margin in a wide range of $\gamma$. Although tuning manner may achieve better performance at several values of $\gamma$, it is mainly due to tuning by grid search enlarges the search region of hyper-parameter $\gamma$, it dramatically increases the running time as well. In contrast, our learning manner can significantly reduce the search region and achieve comparable or much better performance.
%==========================================================
\section{Conclusion}
\label{conclusion}
This paper proposes a novel localized MKC algorithm LSWMKC. In contrast to traditional localized methods in KNN mechanism, which neglects the ranking relationship of neighbors, this paper adopts a heuristic manner to implicitly optimize adaptive weights on different neighbors according to the ranking relationship. We first learn a consensus discriminative graph across multiple views in kernel space, revealing the latent local manifold structures. We further to learn a neighborhood kernel with more discriminative capacity by denoising the consensus graph, which achieves naturally sparse property and clearer block diagonal property. Extensive experimental results on twelve datasets sufficiently demonstrate the superiority of our proposed algorithm over the existing thirteen methods. Our algorithm provides a heuristic insight into localized methods in kernel space. 

However, we should emphasize that the promising performance obtained at the expense of $\mathcal{O}(n^{3})$ computational complexity, which restricts applications in large-scale clustering. Introducing more advanced and efficient graph learning strategies deserves future investigation, especially for prototype or anchor learning, which may reduce the complexity from $\mathcal{O}(n^{3})$ to $\mathcal{O}(n^{2})$, even $\mathcal{O}(n)$. Moreover, the present work still requires post-processing to get the final clustering labels, i.e. $k$-means. Interestingly, several concise strategies, such as rank constraint or one-pass mechanism, provide promising solutions to this issue, which deserves our further research.

%==========================================================
\section*{Acknowledgments}
The authors would like to thank the anonymous reviewers who provided constructive comments for improving the quality of this work. This work was supported by the National Key R\&D Program of China 2020AAA0107100 and the National Natural Science Foundation of China (project no. 61922088, 61773392 and 61976196).

%==========================================================
\bibliographystyle{IEEEtran}
\bibliography{Reference_DBLP}{}

\vfill

\end{document}